\documentclass[journal]{IEEEtran}
\usepackage{adjustbox}

\usepackage{xfrac}
\usepackage[utf8]{inputenc} 
\usepackage[T1]{fontenc} 
\usepackage{nohyperref} 
\usepackage{url} 
\usepackage{booktabs} 
\usepackage{amsfonts} 
\usepackage{nicefrac} 
\usepackage{microtype} 
\usepackage{enumitem}

\usepackage{mathtools, cuted}

\usepackage{booktabs,array}
\usepackage{textcomp}
\usepackage{amsmath}

\usepackage{amsthm, amsfonts, amssymb, mathrsfs, amsmath}
\usepackage{amsmath}

{
	\theoremstyle{plain}
	
}
\usepackage{bm}
\usepackage{nicefrac}
\usepackage{mathtools}
\usepackage{multirow}
\usepackage{color}
\usepackage{subfigure}
\usepackage[ruled,linesnumbered]{algorithm2e}
\let\oldnl\nl
\newcommand{\nonl}{\renewcommand{\nl}{\let\nl\oldnl}}
\usepackage[mathcal]{euscript}
\usepackage[draft]{fixme}
\usepackage{enumerate}
\usepackage{dsfont}
\usepackage{tikz}
\allowdisplaybreaks[1]

\DeclareMathAlphabet\mathcal{OMS}{cmsy}{m}{n}

\newcommand{\bZ}{{\boldsymbol{\mathrm{Z}}}}

\newcommand{\bG}{{\boldsymbol{\mathrm{G}}}}
\newcommand{\bF}{{\boldsymbol{\mathrm{F}}}}

\newcommand{\bth}{{\boldsymbol{\mathrm{\varTheta}}}}
\newcommand{\bt}{{\boldsymbol{\mathrm{\bt}}}}
\newcommand{\bW}{{\boldsymbol{\mathrm{W}}}}

\newcommand{\bY}{{\boldsymbol{\mathrm{Y}}}}

\newcommand{\bD}{{\boldsymbol{\mathrm{D}}}}
\newcommand{\bA}{{\boldsymbol{\mathrm{A}}}}
\newcommand{\bcW}{{\mathbfcal{W}}}
\newcommand{\bcU}{{\mathbfcal{U}}}
\newcommand{\bU}{{\boldsymbol{\mathrm{U}}}}

\DeclareMathAlphabet\mathbfcal{OMS}{cmsy}{b}{n}

\newcommand{\bX}{{\boldsymbol{\mathrm{X}}}}
\newcommand{\bH}{{\boldsymbol{\mathrm{H}}}}

\newcommand{\bep}{{\boldsymbol{\mathrm{\epsilon}}}}
\newcommand{\bx}{{\boldsymbol{\mathrm{x}}}}
\newcommand{\bu}{{\boldsymbol{\mathrm{u}}}}
\newcommand{\bv}{{\boldsymbol{\mathrm{v}}}}

\newcommand{\by}{\boldsymbol{\mathrm{y}}}
\newcommand{\bw}{\boldsymbol{\mathrm{w}}}

\newcommand{\bh}{{\boldsymbol{\mathrm{h}}}}
\newcommand{\bhbar}{\bm{\hbar}}

\newcommand{\bg}{{\boldsymbol{\mathsf{g}}}}
\newcommand{\g}{{\mathsf{g}}}

\newcommand{\bI}{{\boldsymbol{\mathrm{I}}}}

\newcommand{\bigO}{\mathcal{O}}

\newcommand{\bs}{\boldsymbol{\mathrm{s}}}

\makeatletter
\renewcommand{\@algocf@capt@plain}{above}
\makeatother

\DeclareMathOperator*{\argmin}{arg\,min}

\usepackage{amsthm}
\newtheorem{theorem}{Theorem}[section]
\newtheorem{corollary}{Corollary}[theorem]
\newtheorem{lemma}[theorem]{Lemma}

\newtheorem{proposition}[theorem]{Proposition}

\ifCLASSINFOpdf
\else
\fi
\begin{document}
%
\title{Interpretable Deep Recurrent Neural Networks via Unfolding Reweighted $\ell_1$-$\ell_1$ Minimization: Architecture Design and Generalization Analysis}
%
%
%

\author{Huynh Van Luong,~Boris Joukovsky,~and~Nikos Deligiannis
	\thanks{Huynh Van Luong~and~Nikos Deligiannis are with the Department
		of Electronics and Informatics, Vrije Universiteit Brussel, 1050 Brussels, Belgium.}}

%
%

\markboth{}%
{}
%



\maketitle

\begin{abstract}
		Deep unfolding methods---for example, the learned iterative shrinkage thresholding algorithm (LISTA)---design deep neural networks as learned variations of optimization methods. These networks have been shown to achieve faster convergence and higher accuracy than the original optimization methods. In this line of research, this paper develops a novel deep recurrent neural network (coined reweighted-RNN) by the unfolding of a reweighted $\ell_1$-$\ell_1$ minimization algorithm and applies it to the task of sequential signal reconstruction. To the best of our knowledge, this is the first deep unfolding method that explores reweighted minimization. Due to the underlying reweighted minimization model, our RNN has a different soft-thresholding function (alias, different activation function) for each hidden unit in each layer. Furthermore, it has higher network expressivity than existing deep unfolding RNN models due to the over-parameterizing weights. Importantly, we establish theoretical generalization error bounds for the proposed reweighted-RNN model by means of Rademacher complexity. The bounds reveal that the parameterization of the proposed reweighted-RNN ensures good generalization. We apply the proposed reweighted-RNN to the problem of video frame reconstruction from low-dimensional measurements, that is, sequential frame reconstruction. The experimental results on the moving MNIST dataset demonstrate that the proposed deep reweighted-RNN significantly outperforms existing RNN models.
\end{abstract}


%
\IEEEpeerreviewmaketitle

\section{Introduction}
%
%
%
%
\IEEEPARstart{T}{he} problem of reconstructing sequential signals from low-dimensional measurements across time is of great importance for a number of applications such as time-series data analysis, future-frame prediction, and compressive video sensing. Specifically, we consider the problem of reconstructing a sequence of signals $\bs_t\in\mathbb{R}^{n_0}$, $t=1,2,\dots,T$, from low-dimensional measurements $\bx_t=\bA\bs_t$, where $\bA\in \mathbb{R}^{ n\times n_0}~(n\ll n_0)$ is a sensing matrix. We assume that $\bs_t$ has a sparse representation $\bh_t\in\mathbb{R}^{h}$ in a dictionary~$\bD\in \mathbb{R}^{ n_0\times h}$, that is, $\bs_t = \bD \bh_t$. At each time step $t$, the signal $\bs_t$ can be independently reconstructed using the measurements $\bx_t$ by solving~\cite{DonohoTIT06}:
\begin{equation}\label{l1-norm}
\min_{\bh_t} 
\Big\{\frac{1}{2}\|\bx_t -\bA\bD\bh_t\|_2^2 + \lambda\|\bh_t\|_{1}\Big\},
\end{equation}
where $\|\cdot\|_p$ is the $\ell_{p}$-norm and $\lambda$ is a regularization parameter. The iterative shrinkage-thresholding algorithm (ISTA) \cite{daubechies2004iterative} solves~\eqref{l1-norm}
by iterating over $\bh_t^{(l)} = \phi_{\frac{\lambda}{c}}(\bh_t^{(l-1)} - \frac{1}{c}\bD^{\mathrm{T}}\bA^{\mathrm{T}}(\bA\bD\bh_t^{(l-1)}-\bx_t))$,
where $l$ is the iteration counter, $\phi_\gamma(u) = \mathrm{sign}(u)[0,|u|-\gamma]_+ $ is the soft-thresholding operator, $\gamma=\frac{\lambda}{c}$, and $c$ is an upper bound on the Lipschitz constant of the gradient of $\frac{1}{2}\|\bx_t -\bA\bD\bh_t\|_2^2$. 

Under the assumption that sequential signal instances are correlated, we consider the following sequential signal reconstruction problem:
\begin{equation}\label{seqquentialProblem}
\min_{\bh_{t}} \Big\{\frac{1}{2}\|\bx_t-\bA\bD\bh_t\|_2^{2} + \lambda_1\|\bh_t\|_{1} + \lambda_2R(\bh_t,\bh_{t-1})\Big\},
\end{equation}
where $\lambda_1,\lambda_2>0$~are regularization parameters and $R(\bh_t,\bh_{t-1})$ is an added regularization term that expresses the similarity of the representations $\bh_t$ and $\bh_{t-1}$ of two consecutive signals. \cite{WisdomICASSP17} proposed an RNN design (coined Sista-RNN) by unfolding the sequential version of ISTA. That study assumed that two consecutive signals are close in the $\ell_2$-norm sense, formally, $R(\bh_t,\bh_{t-1})=\frac{1}{2}\|\bD\bh_t - \bF\bD\bh_{t-1}\|_{2}^{2}$, where $\bF\in \mathbb{R}^{n_0\times n_0}$ is a correlation matrix between $\bs_t$ and $\bs_{t-1}$. More recently, the study by \cite{LeArXiv19} designed the $\ell_{1}$-$\ell_{1}$-RNN, which stems from unfolding an algorithm that solves the $\ell_{1}$-$\ell_{1}$ minimization problem~\cite{MotaTIT17,MotaTSP17}. This is a version of Problem~\eqref{seqquentialProblem} with $R(\bh_t,\bh_{t-1})=\|\bh_t - \bG\bh_{t-1}\|_{1}$, where $\bG\in \mathbb{R}^{h\times h}$ is an affine transformation that promotes the correlation between $\bh_{t}$ and $\bh_{t-1}$. Both studies \cite{WisdomICASSP17,LeArXiv19} have shown that carefully-designed deep RNN models outperform the generic RNN model and ISTA \cite{daubechies2004iterative} in the task of sequential frame reconstruction. 

Deep neural networks (DNN) have achieved state-of-the-art performance in solving~\ref{l1-norm} for individual signals, both in terms of accuracy and inference speed~\cite{mousavi2015deep}. However, these models are often criticized for their lack of interpretability and theoretical guarantees \cite{lucas2018using}. Motivated by this, several studies focus on designing DNNs that incorporate domain knowledge, namely, signal priors. These include deep unfolding methods which design neural networks to learn approximations of iterative optimization algorithms. Examples of this approach are LISTA~\cite{GregorICML10} and its variants, including ADMM-Net~\cite{sun2016deep}, AMP~\cite{borgerding2017amp}, and an unfolded version of the iterative hard thresholding algorithm~\cite{xin2016maximal}.

LISTA~\cite{GregorICML10}
unrolls the iterations of ISTA into a feed-forward neural
network\ with weights, where each
layer implements an iteration: 
$
\bh_t^{(l)} = \phi_{\gamma^{(l)}}(\bW^{(l)}\bh_t^{(l-1)} + \bU^{(l)}\bx_t),
$
with $\bW^{(l)} = \mathbf{I}-\frac{1}{c}\bD^{\mathrm{T}}\bA^{\mathrm{T}}\bA\bD$, $\bU^{(l)} = \frac{1}{c}\bD^{\mathrm{T}}\bA^{\mathrm{T}}$, and $\gamma^{(l)}$ being learned
from data. It has been shown \cite{GregorICML10,SpechmannPAMI15} that a $d$-layer LISTA network with trainable parameters $\bth= \{\bW^{(l)},\bU^{(l)},\gamma^{(l)}\}_{l=1}^{d}$ achieves the same performance as the original ISTA but with much fewer iterations (i.e., number of layers). Recent studies \cite{ChenNIPS18,LiuICLR19} have found that exploiting dependencies between $\bW^{(l)}$ and $\bU^{(l)}$ leads to reducing the number of trainable parameters while retaining the performance of LISTA. These works provided theoretical insights to the convergence conditions of LISTA. However, the problem of \textit{designing deep unfolding methods for dealing with sequential signals} is significantly less explored.  In this work, we will consider a deep RNN for solving Problem \ref{seqquentialProblem} that outputs a sequence, $\hat{\bs}_1,\dots,\hat{\bs}_T$ from an input measurement sequence, $\bx_1,\dots,\bx_T$, as following:
\begin{align}\label{generic-RNN}
\bh_t&=\phi_{\gamma}(\bW\bh_{t-1} + \bU\bx_t),\nonumber\\
\hat{\bs}_t~& = \bD \bh_t.
\end{align}

It has been shown that reweighted algorithms---such as the reweighted $\ell_1$ minimization method by~\cite{Candes08} and the reweighted $\ell_{1}$-$\ell_{1}$ minimization by~\cite{LuongTIP18}---outperform their non-reweighted counterparts. Driven by this observation, this paper proposes a novel deep RNN architecture by unfolding a reweighted-$\ell_1$-$\ell_1$ minimization algorithm. Due to the reweighting, our network has higher expressivity than existing RNN models leading to better data representations, especially when depth increases. This is in line with recent studies~\cite{HeCVPR16,CortesPCML17,HuangCVPR17}, which have shown that better performance can be achieved by highly over-parameterized networks, i.e., networks with far more parameters than the number of training samples. While the most recent studies (related over-parameterized DNNs) consider fully-connected networks applied on classification problems \cite{NeyshaburICLR19}, our approach focuses on deep-unfolding architectures and opts to understand how the networks learn a low-complexity representation for sequential signal reconstruction, which is a regression problem across time. Furthermore, while there have been efforts to build deep RNNs~\cite{PascanuICLR14,LiCVPR18,LuoICCV18,WisdomICASSP17}, examining the generalization property of such deep RNN models on unseen sequential data still remains elusive. In this work, we derive the generalization error bound of the proposed design and further compare it with existing RNN bounds \cite{ZhangICML18,KusupatiNIPS18}.

\textbf{Contributions}.	The contributions of this work are as follows: 
\begin{itemize}
	\item 
	We propose a principled deep RNN model for sequential signal reconstruction by unfolding a reweighted $\ell_1\text{-}\ell_1$ minimization method. Our reweighted-RNN model employs different soft-thresholding functions that are adaptively learned per hidden unit. Furthermore, the proposed model is over-parameterized, has high expressivity and can be efficiently stacked. 
	
	\item We derive the generalization error bound of the proposed model (and deep RNNs) by measuring Rademacher complexity and show that the over-parameterization of our RNN ensures good generalization. To best of our knowledge, this is the first generalization error bound for deep RNNs; moreover, our bound is tighter than existing bounds derived for shallow RNNs \cite{ZhangICML18,KusupatiNIPS18}.
	
	\item We provide experiments in the task of reconstructing video sequences from low-dimensional measurements. We show significant gains when using our model compared to several state-of-the-art RNNs (including unfolding architectures), especially when the depth of RNNs increases.
\end{itemize}

\section{Deep Recurrent Neural Networks via unfolding reweighted-$\ell_1$-$\ell_1$ minimization}
\label{proposedRNN-l1-l1}

In this section, we describe a reweighted $\ell_1\text{-}\ell_1$ minimization problem for sequential signal reconstruction and propose an iterative algorithm based on the proximal method. We then design a deep RNN architecture by unfolding this algorithm. 


\textbf{The proposed reweighted $\ell_1\text{-}\ell_1$ minimization}. We introduce the following problem: 
\begin{align}
\label{reweighted-l1-l1minimization}
\min_{\bh_t} \Big\{&\frac{1}{2}\|\bx_t-\bA\bD\bZ\bh_t\|_2^2 + \lambda_1\|\bg\circ\bZ\bh_t\|_{1} \nonumber\\
&+ \lambda_2\|\bg\circ(\bZ\bh_t - \bG\bh_{t-1})\|_{1}\Big\} ,
\end{align}
where~\textquotedblleft$\circ$\textquotedblright~denotes element-wise multiplication, $\bg\in\mathbb{R}^{h}$ is a vector of positive weights, $\bZ\in \mathbb{R}^{h\times h}$ is a reweighting matrix, and $\bG\in \mathbb{R}^{h\times h}$ is an affine transformation that promotes the correlation between $\bh_{t-1}$ and $\bh_{t}$. Intuitively, $\bZ$ is adopted to transform $\bh_t$ to $\bZ\bh_t\in\mathbb{R}^{h}$, producing a reweighted version of it. Thereafter, $\bg$ aims to reweight each transformed component of $\bZ\bh_t$ and $\bZ\bh_t-\bG\bh_{t-1}$ in the $\ell_1$-norm regularization terms. Because of applying reweighting~\cite{Candes08}, the solution of Problem \eqref{reweighted-l1-l1minimization} is a more accurate sparse representation compared to the solution of the $\ell_1$-$\ell_1$ minimization problem in \cite{LeArXiv19} (where $\bZ=\bI$ and $\bg=\bI$). Furthermore, the use of the reweighting matrix $\bZ$ to transform $\bh_t$ to $\bZ\bh_t$ differentiates Problem~\eqref{reweighted-l1-l1minimization} from the reweighted $\ell_1\text{-}\ell_1$ minimization problem in~\cite{LuongTIP18} where $\bZ=\bI$.

The objective function in~\eqref{reweighted-l1-l1minimization} consists of the differentiable fidelity term $f(\bZ\bh_t) = \frac{1}{2}\|\bx_t-\bA\bD\bZ\bh_t\|_2^{2}$ and the non-smooth term $g(\bZ\bh_t) = \lambda_1\|\bg\circ\bZ\bh_t\|_{1} + \lambda_2\|\bg\circ(\bZ\bh_t - \bG\bh_{t-1})\|_{1}$. We use a proximal gradient method \cite{Beck09} to solve~\eqref{reweighted-l1-l1minimization}: At iteration~$l$, we first update $\bh_t^{(l-1)}$---after being multiplied by $\bZ_{l}$---with a gradient descent step on the fidelity term as $\bu=\bZ_{l}\bh_t^{(l-1)}-\frac{1}{c}\bZ_{l}\nabla f(\bh_t^{(l-1)})$, where $\nabla f(\bh_t^{(l-1)})=\bD^{\mathrm{T}}\bA^{\mathrm{T}}(\bA\bD\bh_t^{(l-1)}-\bx_t)$. Then, $\bh_t^{(l)}$ is updated as 
\begin{equation}\label{reweighted-l1-proximal}
\bh_t^{(l)}= \varPhi_{\frac{\lambda_1}{c}\bg_{l},\frac{\lambda_2}{c}\bg_{l},\bG\bh_{t-1}}\Big(\bZ_{l}\bh_t^{(l-1)}-\frac{1}{c}\bZ_{l}\nabla f(\bh_t^{(l-1)})\Big),
\end{equation}
where the proximal operator $\varPhi_{\frac{\lambda_1}{c}\bg_{l},\frac{\lambda_2}{c}\bg_{l},\bG\bh_{t-1}}(\bu)$ is defined as
\begin{equation}\label{reweighted-l1-proximalOperator}
\varPhi_{\frac{\lambda_1}{c}\bg_{l},\frac{\lambda_2}{c}\bg_{l},\bhbar}(\bu)=\argmin\limits_{\bv\in \mathbb{R}^h}\Big\{ \frac{1}{c}g(\bv) + \frac{1}{2}||\bv-\bu||^{2}_{2}\Big\},
\end{equation}
with $\bhbar=\bG\bh_{t-1}$. Since the minimization problem is separable, we can minimize \eqref{reweighted-l1-proximalOperator} independently for each of the elements $\g_{l}$, $\hbar$, $u$ of the corresponding $\bg_{l}$, $\bhbar$, $\bu$ vectors. After solving \eqref{reweighted-l1-proximalOperator}, we obtain $\varPhi_{\frac{\lambda_1}{c}\g_{l},\frac{\lambda_2}{c}\g_{l},\hbar}(u)$ [for solving \eqref{reweighted-l1-proximalOperator}, we refer to Proposition \ref{propReweighted-l1-l1} in Appendix \ref{solvingReweightedl1-l1Minimization}]. For $\hbar \geq 0$:
\begin{align}
&\varPhi_{\frac{\lambda_1}{c}\g_l,\frac{\lambda_2}{c}\g_l,\hbar\geq 0}(u)=\nonumber\\
&\begin{cases}
u - \frac{\lambda_1}{c}\g_l - \frac{\lambda_2}{c}\g_l, & \hbar + \frac{\lambda_1}{c}\g_l+ \frac{\lambda_2}{c}\g_l < u < \infty \\
\hbar, & \hbar + \frac{\lambda_1}{c}\g_l - \frac{\lambda_2}{c}\g_l \leq u \leq \hbar + \frac{\lambda_1}{c}\g_l + \frac{\lambda_2}{c}\g_l \\
u - \frac{\lambda_1}{c}\g_l + \frac{\lambda_2}{c}\g_l, & \frac{\lambda_1}{c}\g_l - \frac{\lambda_2}{c}\g_l <u < \hbar + \frac{\lambda_1}{c}\g_l
- \frac{\lambda_2}{c}\g_l\\
0, & -\frac{\lambda_1}{c}\g_l- \frac{\lambda_2}{c}\g_l\leq u \leq \frac{\lambda_1}{c}\g_l- \frac{\lambda_2}{c}\g_l\\
u + \frac{\lambda_1}{c}\g_l + \frac{\lambda_2}{c}\g_l, & -\infty < u < -\frac{\lambda_1}{c}\g_l - \frac{\lambda_2}{c}\g_l,\\
\end{cases}\label{reweighted-l1-proximalOperatorElementCompute1_positive}
\end{align}
and for $\hbar <0$:
\begin{align}
&\varPhi_{\frac{\lambda_1}{c}\g_l,\frac{\lambda_2}{c}\g_l,\hbar<0}(u)=\nonumber\\
&\begin{cases}
u - \frac{\lambda_1}{c}\g_l - \frac{\lambda_2}{c}\g_l, & \frac{\lambda_1}{c}\g_l+ \frac{\lambda_2}{c}\g_l < u < \infty \\
0, & - \frac{\lambda_1}{c}\g_l +\frac{\lambda_2}{c}\g_l \leq u \leq \frac{\lambda_1}{c}\g_l + \frac{\lambda_2}{c}\g_l \\
u + \frac{\lambda_1}{c}\g_l - \frac{\lambda_2}{c}\g_l , & \hbar - \frac{\lambda_1}{c}\g_l +\frac{\lambda_2}{c}\g_l < u < - \frac{\lambda_1}{c}\g_l
+ \frac{\lambda_2}{c}\g_l\\
\hbar, & \hbar-\frac{\lambda_1}{c}\g_l- \frac{\lambda_2}{c}\g_l\leq u \leq \hbar - \frac{\lambda_1}{c}\g_l + \frac{\lambda_2}{c}\g_l\\
u- \frac{\lambda_1}{c}\g_l + \frac{\lambda_2}{c}\g_l, & -\infty < u <\hbar -\frac{\lambda_1}{c}\g_l - \frac{\lambda_2}{c}\g_l\label{reweighted-l1-proximalOperatorElementCompute1_negative}\\
\end{cases}
\end{align}
Fig. \ref{proximalActivation} depicts the proximal operators for $\hbar \geq 0$ and $\hbar <0$. Observe that different values of $g_l$ lead to different shapes of the proximal functions $\varPhi_{\frac{\lambda_1}{c}\g_l,\frac{\lambda_2}{c}\g_l,\hbar}(u)$ for each element $u$ of $\bu$.
\begin{figure}[t]
	\vspace{-20pt}
	\centering 
	\subfigure[$\hbar\geq 0$.]{\label{Positive}\includegraphics[width=0.33\textwidth]{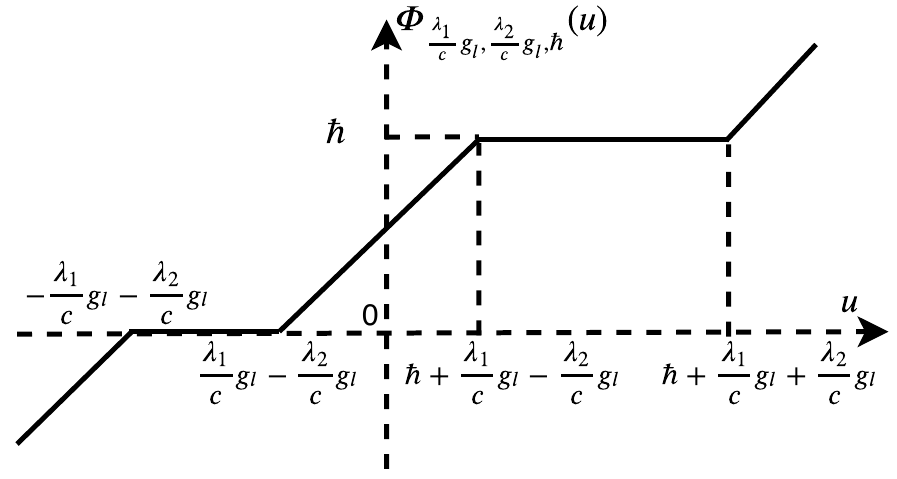}}	
	\subfigure[$\hbar< 0$.]{\label{Negative}\includegraphics[width=0.33\textwidth]{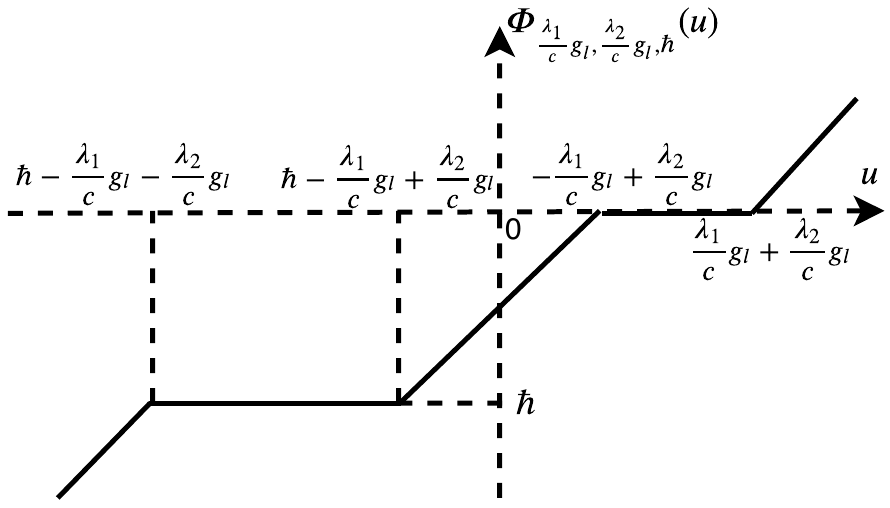}}			
	\caption{The generic form of the proximal operators for Algorithm \ref{reweighted-l1-l1-algorithm} - but also the activation function in the proposed reweighted-RNN. Note that per unit per layer $\g_l$ leads to a different activation function.}
	\label{proximalActivation}
\end{figure}

\begin{algorithm}[t!]
	\textbf{Input:} Measurements~$\bx_1,\dots,\bx_T$, measurement matrix~$\textbf{A}$, dictionary~$\textbf{D}$, affine transform~$\bG$, initial~${\bh}_0^{(d)}\equiv\bh_0$, reweighting matrices $\bZ_1,\dots,\bZ_d$ and vectors $\bg_1,\dots,\bg_d$, $c$, $\lambda_1$, $\lambda_2$.\\
	\textbf{Output:} Sequence of sparse codes~$\bh_1,\dots,\bh_T$.\\
	\For{t = 1,\dots,T}{
		$\bh_t^{(0)} = \bG{\bh}^{(d)}_{t-1}$\\
		\For{l = 1 to d}{
			$\bu = [\bZ_l - \frac{1}{c} \bZ_l\textbf{D}^{\mathrm{T}}\textbf{A}^{\mathrm{T}}
			\textbf{AD}]\bh_t^{(l-1)} + \frac{1}{c}\bZ_l\textbf{D}^{\mathrm{T}} \textbf{A}^{\mathrm{T}} \bx_t$\label{gradientDescent}\\
			$\bh_t^{(l)} = \varPhi_{\frac{\lambda_1}{c}\bg_l, \frac{\lambda_2}{c}\bg_l, \bG\bh^{(d)}_{t-1}}(\bu)$\label{proximalOperator}\\
		}
	}
	\textbf{return} $\bh_1^{(d)},\dots,\bh_T^{(d)}$\\
	\caption{The proposed algorithm for sequential signal reconstruction.}
	\label{reweighted-l1-l1-algorithm}
\end{algorithm}
\vspace{-10pt}
Our iterative algorithm is given in Algorithm~\ref{reweighted-l1-l1-algorithm}. We reconstruct a sequence $\bh_1,\dots,\bh_T$ from a sequence of measurements $\bx_1,\dots,\bx_T$. For each time step $t$, Step \ref{gradientDescent} applies a gradient descent update for~$f(\bZ\bh_{t-1})$ and Step \ref{proximalOperator} applies the proximal operator $\varPhi_{\frac{\lambda_1}{c}\bg_l,\frac{\lambda_2}{c}\bg_l,\bG\bh_{t-1}^{(d)}}$ element-wise to the result. Let us compare the proposed method against the algorithm in \cite{LeArXiv19}---which resulted in the $\ell_1$-$\ell_1$-RNN---that solves the $\ell_1$-$\ell_1$ minimization in \cite{MotaTSP17} (where $\bZ_l=\bI$ and $\bg_l=\bI$). In that algorithm, the update terms in Step~\ref{gradientDescent}, namely $\bI- \frac{1}{c} \textbf{D}^{\mathrm{T}}\textbf{A}^{\mathrm{T}}
\textbf{AD}$ and $\frac{1}{c}\textbf{D}^{\mathrm{T}} \textbf{A}^{\mathrm{T}} $, and the proximal operator in Step \ref{proximalOperator} are the same for all iterations of $l$. In contrast, Algorithm \ref{reweighted-l1-l1-algorithm} uses a different $\bZ_l$ matrix per iteration to reparameterize the update terms (Step~\ref{gradientDescent}) and, through updating $\bg_l$, it applies a different proximal operator to each element $\bu$ (in Step \ref{proximalOperator}) per iteration $l$.

\textbf{The proposed reweighted-RNN}. We now describe the proposed architecture for sequential signal recovery, designed by unrolling the steps of Algorithm \ref{reweighted-l1-l1-algorithm} across the iterations $l = 1,\dots,d$ (yielding the hidden layers) and time steps $t = 1,\dots,T$.
Specifically, the $l$-th hidden layer is given by 
\begin{equation}\label{reweighted-l1-l1-RNN}
\bh_t^{(l)}\hspace{-2pt}=\hspace{-2pt}\left\{
\begin{array}{l}
\varPhi_{\frac{\lambda_1 }{c}\bg_1, \frac{\lambda_2 }{c}\bg_1, \bG\bh^{(d)}_{t-1}}\Big(\mathbf{W}_{1}\bh_{t-1}^{(d)}+\mathbf{U}_1\bx_t\Big),~\text{if}~l=1,\\
\varPhi_{\frac{\lambda_1 }{c}\bg_l, \frac{\lambda_2 }{c}\bg_l, \bG\bh^{(d)}_{t-1}}\Big(\mathbf{W}_{l}\bh_{t}^{(l-1)} + \textbf{U}_l\bx_t\Big),~~\text{if}~l>1,\\
\end{array}
\right.
\end{equation}
and the reconstructed signal at time step~$t$ is given by $\hat{\bs}_t=\bD\bh_t^{(d)}$; where $\mathbf{U}_l$, $\mathbf{W}_{l}$, $\mathbf{V}$ are defined as
\begin{align}
\mathbf{U}_l&=\frac{1}{c}\bZ_l\bD^{\mathrm{T}}\bA^{\mathrm{T}}, \forall l,\label{weightU} \\ 
\mathbf{W}_{1}&=\bZ_1\bG - \frac{1}{c} \bZ_1\textbf{D}^{\mathrm{T}}\textbf{A}^{\mathrm{T}} \textbf{AD}\bG,\label{weightW1}\\
\mathbf{W}_{l}&=\bZ_l-\frac{1}{c}\bZ_l\bD^{\mathrm{T}}\bA^{\mathrm{T}}\bA\bD, ~l>1.\label{weightWl}
\end{align}
The activation function is the proximal operator $\varPhi_{\frac{\lambda_1}{c}\bg_l,\frac{\lambda_2}{c}\bg_l,\bhbar}(\bu)$ with learnable parameters $\lambda_1$, $\lambda_2$, $c$, $\bg_l$ (see Fig. \ref{proximalActivation} for the shapes of the activation functions). 

\begin{figure}[t]
	\centering
	\label{proposed-model}
	\subfigure[The proposed reweighted-RNN.]{\label{reweighted-RNN}\includegraphics[width=0.42\textwidth]{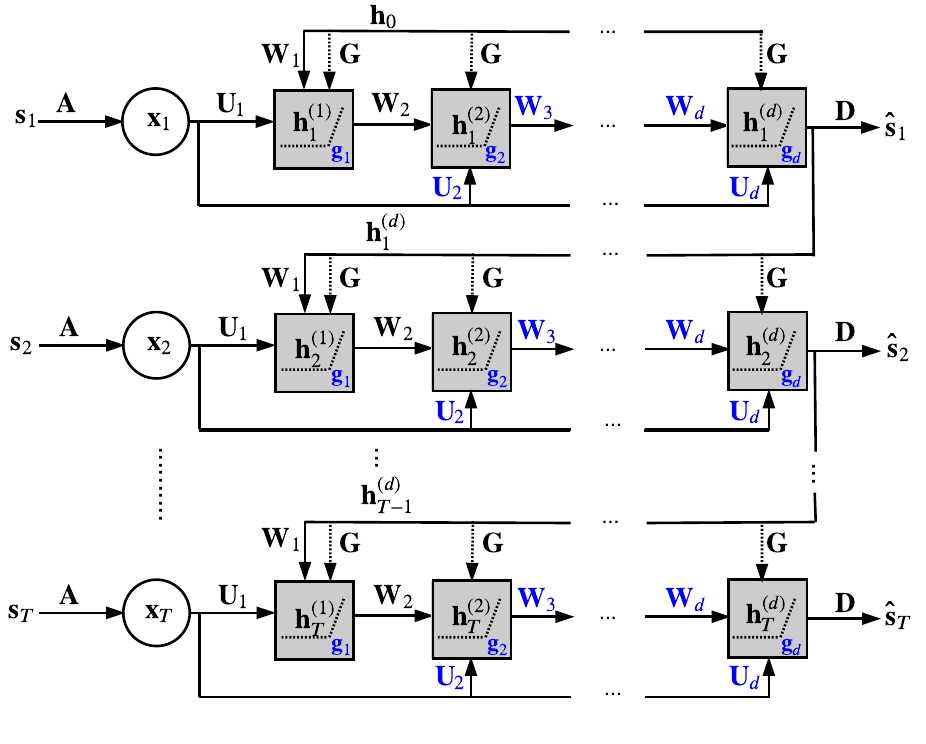}}
	\subfigure[$\ell_1$-$\ell_1$-RNN.]{\label{fig-l1-l1-RNN}\includegraphics[width=0.42\textwidth]{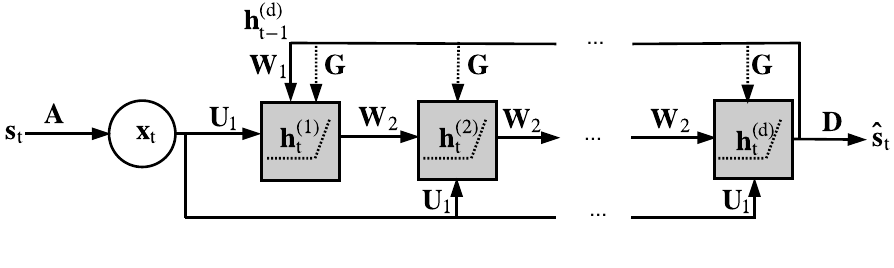}}		
	\subfigure[Stacked RNN.]{\label{fig-stackedRNN}\includegraphics[width=0.42\textwidth]{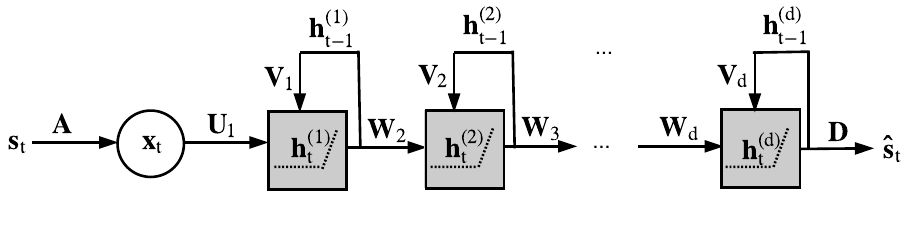}}
	\caption{The proposed (a) reweighted-RNN vs. (b) $\ell_1$-$\ell_1$-RNN and (c) Stacked RNN.}
\end{figure} 	
Fig. \ref{reweighted-RNN} depicts the architecture of the proposed reweighted-RNN. Input vectors $\bs_t$, $t=1,\dots,T$ are compressed by a linear measurement layer $\bA$, resulting in compressive measurements $\bx_{t}$. The reconstructed vectors $\hat{\bs}_t$, $t=1,\dots,T$, are obtained by multiplying linearly the hidden representation $\bh_{t}^{(d)}$ with the dictionary $\bD$. We train our network in an end-to-end fashion. During training, we minimize the loss function~$\mathcal{L}(\bth) =\underset{\bs_1,\cdots,\bs_T}{\mathbb{E}} \Big[ \sum_{t=1}^{T}\|\bs_t - \hat{\bs}_t\|_{2}^{2}\Big]$ using stochastic gradient descent on mini-batches, where the trainable parameters are $\bth= \{\bA, \bD, \bG,\bh_0, \bZ_1,\dots,\bZ_d,\bg_1,\dots,\bg_d, c, \lambda_1, \lambda_2\}$. 

We now compare the proposed reweighted-RNN [Fig. \ref{reweighted-RNN}] against the recent $\ell_1$-$\ell_1$-RNN \cite{LeArXiv19} [Fig. \ref{fig-l1-l1-RNN}]. The $l$-th hidden layer in $\ell_1$-$\ell_1$-RNN is given by
\begin{equation}\label{l1-l1-RNN}
\bh_t^{(l)}\hspace{-2pt}=\hspace{-2pt}\left\{
\begin{array}{l}
\varPhi_{\frac{\lambda_1}{c}, \frac{\lambda_2}{c}, \bG\bh^{(d)}_{t-1}}\Big(\mathbf{W}_{1}\bh_{t-1}^{(d)}+\mathbf{U}_1\bx_t\Big),~\text{if}~l=1,\\
\varPhi_{\frac{\lambda_1}{c}, \frac{\lambda_2}{c}, \bG\bh^{(d)}_{t-1}}\Big(\mathbf{W}_{2}\bh_{t}^{(l-1)} + \textbf{U}_1\bx_t\Big),~\text{if}~l>1.\\
\end{array}
\right.
\end{equation}
The proposed model has the following advantages over $\ell_1$-$\ell_1$-RNN. Firstly, $\ell_1$-$\ell_1$-RNN uses the proximal operator $\varPhi_{\frac{\lambda_1}{c}, \frac{\lambda_2}{c}, \bhbar}(\bu)$ as activation function, whose learnable parameters $\lambda_1$, $\lambda_2$ are fixed across the network. Conversely, the corresponding parameters $\frac{\lambda_1}{c}\bg_l$ and $\frac{\lambda_2}{c}\bg_l$ [see \eqref{reweighted-l1-proximalOperatorElementCompute1_positive}, \eqref{reweighted-l1-proximalOperatorElementCompute1_negative}, and Fig. \ref{proximalActivation}] in our proximal operator, $\varPhi_{\frac{\lambda_1}{c}\bg_l,\frac{\lambda_2}{c}\bg_l,\bhbar}(\bu)$, are learned for each hidden layer due to the reweighting vector $\bg_l$; hence, the proposed model has a different activation function for each unit per layer. The second difference comes from the set of parameters $\{\bW_l, \bU_l\}$ in \eqref{l1-l1-RNN} and \eqref{reweighted-l1-l1-RNN}. The $\ell_1$-$\ell_1$-RNN model uses the same $\{\bW_2, \bU_1\}$ for the second and higher layers. In contrast, our reweighted-RNN has different sets of $\{\bW_l, \bU_l\}$ per hidden layer due to the reweighting matrix $\bZ_l$. These two aspects [which are schematically highlighted in blue fonts in Fig. \ref{reweighted-RNN}] can lead to an increase in the learning capability of the proposed reweighted-RNN, especially when the depth of the model increases. 

In comparison to a generic stacked RNN~\cite{PascanuICLR14} [Fig. \ref{fig-stackedRNN}], reweighted-RNN promotes the inherent data structure, that is, each vector $\bs_t$ has a sparse representation $\bh_t$ and consecutive $\bh_t$'s are correlated. This design characteristic of the reweighted-RNN leads to residual connections which reduce the risk of vanishing gradients during training [the same idea has been shown in several works~\cite{HeCVPR16,HuangCVPR17} in deep neural network literature]. Furthermore, in \eqref{weightU} and \eqref{weightWl}, we see a weight coupling of $\bW_l$ and $\bU_l$ (due to the shared components of $\bA$, $\bD$ and $\bZ$). This coupling satisfies the necessary condition of the convergence in \cite{ChenNIPS18} (Theorem 1). Using Theorem 2 in \cite{ChenNIPS18}, it can be shown that reweighted-RNN, in theory, needs a smaller number of iterations (i.e.,
$d$ in Algorithm \ref{reweighted-l1-l1-algorithm}) to reach convergence, compared to ISTA \cite{daubechies2004iterative} and FISTA \cite{Beck09}.


\section{Generalization Analysis}
While increasing the network expressivity, the over-parameterization of reweighted-RNN raises the question of whether our network ensures good generalization. In this section, we derive and analyze the generalization properties of the proposed reweighted-RNN model in comparison to state-of-the-art RNN architectures. We provide bounds on the Rademacher complexity \cite{ShwartzBook14} for functional classes of the considered deep RNNs, which are used to derive generalization error bounds for evaluating their generalization properties

\subsection{Rademacher complexity and generalization error bound}\label{generalizationBound}

\textbf{Notations}. Let $f_{\bcW}^{(d)}:\mathbb{R}^n\mapsto \mathbb{R}^{h}$ be the function computed by a $d$-layer network with weight parameters $\bcW$. The network $f_{\bcW}^{(d)}$ maps an input sample $\bx_i \in \mathbb{R}^n$ (from an input space $\mathcal{X}$) to an output $\by_i \in \mathbb{R}^h$ (from an output space $\mathcal{Y}$), i.e., $\by_i = f_{\bcW}^{(d)}(\bx_i)$. Let $S$ denote a training set of size $m$, i.e., $S=\{(\bx_i,\by_i)\}_{i=1}^m$ and $\mathbb{E}_{(\bx_i,\by_i)\sim S}[\cdot]$ denote an expectation over $(\bx_i,\by_i)$ from $S$. The set $S$ is drawn i.i.d. from a distribution $\mathcal{D}$, denoted as $S \sim \mathcal{D}^m$, over a space $\mathcal{Z} = \mathcal{X}\times \mathcal{Y}$. Let $\mathcal{F}$ be a (class) set of functions. Let $\ell:\mathcal{F} \times \mathcal{Z} \mapsto \mathbb{R}$ denote the loss function and $	\ell\boldsymbol{\circ}\mathcal{F} = \{z\mapsto \ell(f,z):f\in \mathcal{F}\}$. We define the true loss and the empirical (training) loss by $L_{\mathcal{D}}(f)$ and $L_S(f)$, respectively, as follows:
\begin{equation}\label{testLoss}
L_{\mathcal{D}}(f)=\mathbb{E}_{(\bx_i,\by_i)\sim \mathcal{D}}\big[\ell\big(f(\bx_i),\by_i\big)\big],
\end{equation}
and
\begin{equation}\label{trainLoss}
L_S(f)=\mathbb{E}_{(\bx_i,\by_i)\sim S}\big[\ell\big(f(\bx_i),\by_i\big)\big].
\end{equation}	
\textit{Generalization error} that is defined as a measure of how accurately a learned algorithm is able to predict outcome values for unseen data is calculated by $L_{\mathcal{D}}(f)-L_S(f)$.

\textbf{Rademacher complexity}. Let $\mathcal{F}$ be a hypothesis set of functions (neural networks). The \textit{empirical Rademacher complexity} of $\mathcal{F}$ \cite{ShwartzBook14} for a training sample set $S$ is defined as follows:
\begin{equation}\label{empiricalRademacher}
\mathfrak{R}_S(\mathcal{F})= \frac{1}{m} \underset{\bep\in \{\pm 1\}^m}{\mathbb{E}} \Bigg[\sup_{f \in \mathcal{F}}\sum_{i=1}^{m}\epsilon_if(\bx_i)\Bigg],
\end{equation}
where $\bep = (\epsilon_1,...,\epsilon_m)$, here $\epsilon_i$ is independent uniformly distributed random (Rademacher) variables from $\{\pm1\}$, according to $\mathbb{P}[\epsilon_i=1]=\mathbb{P}[\epsilon_i=-1]=1/2$. 

\textbf{The generalization error bound} \cite{ShwartzBook14} is derived based on the Rademacher complexity in the following theorem:
\begin{theorem}\label{GETheorem}\cite[Theorem 26.5]{ShwartzBook14}\\
	Assume that $|\ell(f,z)|\leq \eta$ for all $f \in \mathcal{F}$ and $z$. Then, for any $\delta>0$, with probability at least $1-\delta$,
	\begin{align}
	&L_{\mathcal{D}}(f) - L_S(f) \leq 2 \mathfrak{R}_S(\ell\boldsymbol{\circ} \mathcal{F}) + 4\eta\sqrt{\frac{2\log(4/\delta)}{m}}. \label{GETheoremEq}
	\end{align}
\end{theorem}
It can be remarked that the bound in \eqref{GETheoremEq} depends on the training set $S$, which is able to be applied to a number of learning problems, e.g., regression and classification, given a loss function $\ell$.

\subsection{Generalization error bounds for reweighted-RNN}

\begin{theorem}[Generalization error bound for reweighted-RNN]\label{GE-reweighted-RNN-Theorem} Let $\mathcal{F}_{d,T}:\mathbb{R}^h\times \mathbb{R}^n\mapsto \mathbb{R}^h$ denote the functional class of reweighted-RNN with~$T$ time steps, where $\|\bW_l\|_{1,\infty}\leq \alpha_l$, $\|\bU_l\|_{1,\infty}\leq \beta_l$, and $1\leq l \leq d$. Assume that the input data $\|\bX_t\|_{2,\infty}\leq \sqrt{m}B_{\bx}$, initial hidden state $\bh_0$, and the loss function is 1-Lipschitz and bounded by $\eta$. Then, for $f \in \mathcal{F}_{d,T}$ and any $\delta>0$, with probability at least $1-\delta$ over a training set $S$ of size $m$,
	\begin{align}\label{GE-reweighted-RNN-TheoremEq}
	&L_{\mathcal{D}}(f) - L_S(f) \leq 2 \mathfrak{R}_S( \mathcal{F}_{d,T}) + 4\eta\sqrt{\frac{2\log(4/\delta)}{m}},
	\end{align}
	where
	\begin{align}\label{deepreweighted-RNNBoundEqTime}
	\hspace{-8pt}\mathfrak{R}_S(\mathcal{F}_{d,T})\leq&
	\sqrt{\frac{2(4dT\log 2+\log n + \log h)}{m}}
	\nonumber
	\\	
	~~~~~&\cdot\sqrt{\Big(\sum\limits_{l=1}^{d}\beta_l\varLambda_l\Big)^2 \Big(\frac{\Lambda_0^T-1}{\Lambda_0-1}\Big)^2 B_{\bx}^2+\varLambda_{0}^{2T} \|\bh_0\|_{\infty}^2},
	\end{align}
	\vspace{-0pt}
	with $\varLambda_l$ defined as follows: $\varLambda_l=\prod\limits_{k=l+1}^{d}\alpha_k$ with $~0\leq l\leq d-1$ and $\varLambda_d=1$.
\end{theorem}
\begin{proof}
	The proof is given in Appendix~\ref{deepreweighted-RNNTheoremProofTime}.
\end{proof}
The generalization error in~\eqref{GE-reweighted-RNN-TheoremEq} is bounded by the Rademacher complexity, which depends on the training set $S$. If the Rademacher complexity is small, the network can be learned with a small generalization error. The bound in~\eqref{deepreweighted-RNNBoundEqTime} is in the order of the square root of the network depth $d$ multiplied by the number of time steps $T$. The bound depends on the logarithm of the number of measurements $n$ and the number of hidden units $h$. It is worth mentioning that the second square root in \eqref{deepreweighted-RNNBoundEqTime} only depends on the norm constraints and the input training data, and it is independent of the network depth $d$ and the number of time steps $T$ under the appropriate norm constraints.

To compare our model with $\ell_1$-$\ell_1$-RNN~\cite{LeArXiv19} and Sista-RNN~\cite{WisdomICASSP17}, we derive bounds on their Rademacher complexities for a time step $t$. The definitions of a functional class $\mathcal{F}_{d,t}$ for the $t^{th}$ time step of reweighted-RNN, $\ell_1$-$\ell_1$-RNN, and Sista-RNN are given in Appendix \ref{deepUnfoldRNNs}. Let $\bH_{t-1}\in \mathbb{R}^{h\times m}$ denote a matrix with columns the vectors of the previous hidden state $\{\bh_{t-1,i}\}_{i=1}^m$, and $\|\bH_{t-1}\|_{2,\infty}=\sqrt{\max\limits_{k\in \{1,\dots,h\}}\sum_{i=1}^m\mathrm{h}_{t-1,i,k}^2}\leq \sqrt{m}B_{\bh_{t-1}}$. 
\begin{corollary}\label{deepreweighted-RNNBound}
	The empirical Rademacher complexity of $\mathcal{F}_{d,t}$ for reweighted-RNN is bounded as:
	\begin{align}\label{deepreweighted-RNNBoundEq}
	\mathfrak{R}_S(\mathcal{F}_{d,t})\leq&
	\sqrt{\frac{2(4d\log 2+\log n + \log h)}{m}}
	\nonumber
	\\
	~~~~~&\cdot\sqrt{\Big(\sum\limits_{l=1}^{d}\beta_l\varLambda_l\Big)^2 B_{\bx}^2+\varLambda_{0}^2 B_{\bh_{t-1}}^2},
	\end{align}
	\vspace{-0pt}
	with $m$ the number of training samples and $\varLambda_l$ given by $\varLambda_d=1$, $\varLambda_l=\prod\limits_{k=l+1}^{d}\alpha_k$ with $~0\leq l\leq d-1$.
\end{corollary}
\begin{proof}
	The proof is a special case of Theorem~\ref{GE-reweighted-RNN-Theorem} for time step $t$.
\end{proof}	

Following the proof of Theorem~\ref{GE-reweighted-RNN-Theorem}, we can easily obtain the following Rademacher complexities for the $\ell_1$-$\ell_1$-RNN and Sista-RNN models. 
\begin{corollary}\label{deepl1-l1RNNBound}
	The empirical Rademacher complexity of $\mathcal{F}_{d,t}$ for $\ell_1$-$\ell_1$-RNN is bounded as:
	\begin{align}\label{deepl1-l1RNNBoundEq}
	\mathfrak{R}_S(\mathcal{F}_{d,t})\leq&
	\sqrt{\frac{2(4d\log 2+\log n + \log h)}{m}}
	\nonumber
	\\
	~~~~~&\cdot\sqrt{\beta_1^2\Big(\frac{\alpha_2^d-1}{\alpha_2-1}\Big)^2 B_{\bx}^2+\alpha_1^2\alpha_2^{2(d-1)} B_{\bh_{t-1}}^2}.
	\end{align}
\end{corollary}	
\begin{corollary}\label{deepl1-l2RNNBound}
	The empirical Rademacher complexity of $\mathcal{F}_{d,t}$ for Sista-RNN is bounded as:
	\begin{align}\label{deepl1-l2RNNBoundEq}
	&\mathfrak{R}_S(\mathcal{F}_{d,t})\leq
	\sqrt{\frac{2(4d\log 2+\log n + \log h)}{m}}\nonumber\\
	&\cdot\sqrt{\beta_1^2\Big(\frac{\alpha_2^d-1}{\alpha_2-1}\Big)^2 B_{\bx}^2+\Bigg(\alpha_1\alpha_2^{(d-1)} +\beta_2\Big(\frac{\alpha_2^{d-1}-1}{\alpha_2-1}\Big) \Bigg)^2 B_{\bh_{t-1}}^2}.
	\end{align}
\end{corollary}
By contrasting~\eqref{deepreweighted-RNNBoundEq} with \eqref{deepl1-l1RNNBoundEq} and \eqref{deepl1-l2RNNBoundEq}, we see that the complexities of $\ell_1$-$\ell_1$-RNN and Sista-RNN have a polynomial dependence on $\alpha_1$, $\beta_1$ and $\alpha_2$, $\beta_2$ (the norms of first two layers), whereas the complexity of reweighted-RNN has a polynomial dependence on $\alpha_1,\dots,\alpha_d$ and $\beta_1,\dots,\beta_d$ (the norms of all layers). This over-parameterization offers a flexible way to control the generalization error of reweighted-RNN. We derive empirical generalization errors in Fig. \ref{fig:mseDepth} demonstrating that increasing the depth of reweighted-RNN still ensures the low generalization error.

\subsection{Comparison with existing generalization bounds}
Recent works have established generalization bounds for RNN models with a single recurrent layer ($d=1$) using Rademacher complexity (see FastRNN in~\cite{KusupatiNIPS18}) or PAC-Bayes theory (see SpectralRNN in~\cite{ZhangICML18}). We re-state these generalization bounds below and apply Theorem~\ref{GE-reweighted-RNN-Theorem} with $d=1$ to compare with our bound for reweighted-RNN.

\textbf{FastRNN} \cite{KusupatiNIPS18}. The hidden state $\bh_t$ of FastRNN is updated as follows:
\begin{align}\label{FastRNN}
\tilde{\bh}_t&=\phi(\bW\bh_{t-1}+\bU\bx_t)\nonumber\\
\bh_t&=a\tilde{\bh}_t + b\bh_{t-1},
\end{align}
where $0\leq a,b \leq 1$ are trainable parameters parameterized by the sigmoid function. Under the assumption that $a+b=1$, the Rademacher complexity $\mathfrak{R}_S(\mathcal{F}_{T})$ of the class $\mathcal{F}_{T}$ of FastRNN \cite{KusupatiNIPS18}, with $\|\bW\|_F\leq \alpha_F$, $\|\bU\|_F\leq \beta_F$, and $\|\bx_t\|_2\leq B$, is given by
\begin{align}\label{FastRNNRademacherAssumption1}
\mathfrak{R}_S(\mathcal{F}_{T})&\leq
\frac{2a}{\sqrt{m}}B\beta_F\Big(\frac{(1+a(2\alpha_F-1))^{T+1}-1}{a(2\alpha_F-1)}\Big).
\end{align}	
Alternatively, under the additional assumption that $a\leq \frac{1}{2(2\alpha_F-1)T}$, the bound in \cite{KusupatiNIPS18} becomes:
\begin{align}\label{FastRNNRademacherAssumption2}
\mathfrak{R}_S(\mathcal{F}_{T})&\leq
\frac{2a}{\sqrt{m}}B\beta_F\Big(\frac{2a(2\alpha_F-1)(T+1)-1}{(2\alpha_F-1)}\Big).
\end{align}	

\textbf{SpectralRNN} \cite{ZhangICML18}. The hidden state $\bh_t$ and output $\by_t \in \mathbb{R}^{n_{\by}}$ of SpectralRNN are computed as: 
\begin{align}\label{SpectralRNN}
\bh_t=&\phi(\bW\bh_{t-1}+\bU\bx_t)\nonumber\\
\by_t=&\bY \bh_{t}, 
\end{align}
where $\bY\in \mathbb{R}^{n_{\by}\times h}$. The generalization error in~\cite{ZhangICML18} is derived for a classification problem. For any $\delta>0, \gamma>0$, with probability $\geq 1-\delta$ over a training set $S$ of size $m$, the generalization error \cite{ZhangICML18} of SpectralRNN is bounded by
\begin{align}\label{SpectralBound}
&\bigO\Bigg(\sqrt{\frac{\frac{B^2 T^4 \xi \ln(\xi)}{\gamma^2}(\|\bW\|_F^2+\|\bU\|_F^2+\|\bY\|_F^2)\cdot\zeta +\ln \frac{m}{\delta}}{m}}\Bigg)
,
\end{align}	
where $\zeta = \max\{\|\bW\|_2^{2T-2},1\}\max\{\|\bU\|_2^2,1\}\max\{\|\bY\|_2^2,1\}$ and $\xi=\max\{n,n_{\by},h\}$.

\textbf{Reweighted-RNN}. Based on Theorem~\ref{GE-reweighted-RNN-Theorem}, under the assumption that the initial hidden state $\bh_0=\mathbf{0}$, the Rademacher complexity of reweighted-RNN with $d=1$ is bounded as
\begin{align}\label{deepreweighted-RNNBoundEq1LayerTime}
\mathfrak{R}_S(\mathcal{F}_{1,T})&\leq
\sqrt{\frac{4T\log 2+\log n + \log h}{m}}\Big(\sqrt{2}\beta_1\frac{\alpha_1^T-1}{\alpha_1-1} B_{\bx} \Big).
\end{align}
\vspace{-0pt}

We observe that the bound of SpectralRNN in \eqref{SpectralBound} depends on $T^2$, whereas the bound of FastRNN either grows exponentially with $T$ \eqref{FastRNNRademacherAssumption1} or is proportional to $T$ \eqref{FastRNNRademacherAssumption2}. Our bound \eqref{deepreweighted-RNNBoundEq1LayerTime} depends on $\sqrt{T}$, given that the second factor in \eqref{deepreweighted-RNNBoundEq1LayerTime} is only dependent on the norm constraints $\alpha_1,~\beta_1$ and the input training data; meaning that it is tighter than those of SpectralRNN and FastRNN in terms of the number of time steps.

\vspace{-0pt}
\section{Experimental results}	
\vspace{-0pt}

\subsection{Video frame reconstruction from compressive measurements}
\vspace{-0pt}
{\renewcommand{\arraystretch}{1.1}		
	\begin{table*}[th!] 
		\caption{Average PSNR results [dB] on the test set
			with different CS rates.}
		\label{tab:result-rate}
		\centering
		\begin{scriptsize}
			\begin{tabular}{ c| c c c c c c c c c }
				\hline
				\hspace{-4pt}\textbf{CS Rate}\hspace{-4pt}& \textbf{sRNN} & \textbf{LSTM} & \textbf{GRU} & \textbf{IndRNN \cite{LiCVPR18}} & \textbf{FastRNN \cite{KusupatiNIPS18}} & \hspace{-4pt}\textbf{SpectralRNN \cite{ZhangICML18}}\hspace{-4pt} & \textbf{Sista-RNN} & \textbf{$\ell_1$-$\ell_1$-RNN} & \textbf{Ours} \\
				\hline 
				\hspace{-4pt}\textbf{0.1}\hspace{-4pt} & 25.11 & 24.58 & 25.18 & 25.68 & 25.21 &\hspace{-4pt} 25.15 \hspace{-4pt}& 25.16 & 24.68 & \textbf{26.25} \\
				\hline
				\hspace{-4pt}\textbf{0.2}\hspace{-4pt} & 31.14 & 29.46 & 31.19 & 32.90 & 32.05 &\hspace{-4pt} 31.65 \hspace{-4pt}& 31.53 & 30.79 & \textbf{34.19} \\
				\hline
				\hspace{-4pt}\textbf{0.3}\hspace{-4pt} & 35.38 & 32.91 & 36.49 & 37.12 & 36.40 &\hspace{-4pt}36.89\hspace{-4pt}& 36.96 & 37.77 & \textbf{42.39} \\
				\hline
				\hspace{-4pt}\textbf{0.4}\hspace{-4pt} & 38.05 & 34.95 & 39.47 & 40.84 & 39.21 &\hspace{-4pt} 40.22 \hspace{-4pt}& 39.57 & 40.35 & \textbf{46.03} \\
				\hline
				\hspace{-4pt}\textbf{0.5}\hspace{-4pt} & 39.34 & 36.28 & 41.12 & 45.49 & 41.87 &\hspace{-4pt} 41.36 \hspace{-4pt}& 41.56 & 43.35 & \textbf{48.70} \\
				\hline 
			\end{tabular}
		\end{scriptsize}
	\end{table*}	
	\vspace{-0pt}
	{\renewcommand{\arraystretch}{1.2}		
		\begin{table*}[h] 
			\caption{Average PSNR results [dB] on the test set
				with different network widths $h$ (CS rate is 0.2,~$d=3$).}
			\label{tab:result-width}
			\centering
			\begin{scriptsize}
				\begin{tabular}{ c | c c c c c c c c c }
					\hline
					\hspace{-1pt}\textbf{$h$}\hspace{-4pt} & \textbf{sRNN} & \textbf{LSTM} & \textbf{GRU} & \textbf{IndRNN \cite{LiCVPR18}} & \textbf{FastRNN \cite{KusupatiNIPS18}} & \textbf{SpectralRNN \cite{ZhangICML18}} & \textbf{Sista-RNN} & \textbf{$\ell_1$-$\ell_1$-RNN} & \textbf{Ours}\\
					\hline 
					\hspace{-1pt}$\textbf{2}^\textbf{7}$\hspace{-4pt} & 23.35 & 22.87 & 23.55 & 23.82 & 23.83 & 22.92 & 23.86 & 23.90 & \textbf{28.09} \\
					\hline
					\hspace{-1pt}$\textbf{2}^\textbf{8}$\hspace{-4pt} & 25.81 & 23.88 & 26.67 & 27.10 & 26.71 & 24.46 & 29.64 & 29.55 & \textbf{31.46} \\
					\hline
					\hspace{-1pt}$\textbf{2}^\textbf{9}$\hspace{-4pt} & 28.72 & 26.83 & 30.29 & 32.03 & 29.92 & 30.23 & 31.30 & 30.61 & \textbf{33.61} \\
					\hline 
					\hspace{-1pt}$\textbf{2}^\textbf{10}$\hspace{-4pt} & 31.14 & 29.46 & 31.19 & 32.90 & 32.05 & 31.65 & 31.53 & 30.79 & \textbf{34.19} \\
					\hline
					\hspace{-1pt}$\textbf{2}^\textbf{11}$\hspace{-4pt} & 29.91 & 29.30 & 31.15 & 33.10 & 30.80 & 31.68& 31.82 & 30.45 & \textbf{34.80}\\
					\hline 
					\hspace{-1pt}$\textbf{2}^\textbf{12}$\hspace{-4pt} & 29.71 & 29.08 & 30.93 & 32.47 & 24.26 &29.26 & 31.63 &30.09 & \textbf{34.98}\\
					\hline
				\end{tabular}
			\end{scriptsize}
		\end{table*}	
		{\renewcommand{\arraystretch}{1.1}		
			\begin{table*}[h] 
				\label{results}
				\caption{Average PSNR results [dB] on the test set
					with different network depths $d$ (a CS rate is 0.2,~$h=2^{10}$).}
				\label{tab:result-depth}
				\centering
				\begin{scriptsize}
					\begin{tabular}{ c | c c c c c c c c c }
						\hline
						\hspace{-1pt}\textbf{$d$} \hspace{-4pt}& \textbf{sRNN} & \textbf{LSTM} & \textbf{GRU} & \textbf{IndRNN \cite{LiCVPR18}} & \textbf{FastRNN \cite{KusupatiNIPS18}} & \textbf{SpectralRNN \cite{ZhangICML18}} & \textbf{Sista-RNN} & \textbf{$\ell_1$-$\ell_1$-RNN} & \textbf{Ours}\\
						\hline 
						\hspace{-1pt}\textbf{1} \hspace{-4pt}& 27.52 & 27.76 & 27.61 & \textbf{30.12} & 29.32 & 29.62 & 28.41 & 28.49 & 29.19 \\ 
						\hline 
						\hspace{-1pt}\textbf{2} \hspace{-4pt}& 29.21 & 29.46 & 29.68 & \textbf{32.73} & 30.84 & 31.37 & 30.67 & 30.19 & 32.12 \\
						\hline
						\hspace{-1pt}\textbf{3} \hspace{-4pt}& 31.14 & 22.29 & 31.19 & 32.90 & 32.05 & 31.65 & 31.53 & 30.79 & \textbf{34.19} \\
						\hline
						\hspace{-1pt}\textbf{4} \hspace{-4pt}& 31.64 & 16.50 & 29.26 & 20.65 & 31.07 & 31.10 & 32.56 & 31.80 & \textbf{35.99} \\
						\hline
						\hspace{-1pt}\textbf{5} \hspace{-4pt}& 16.50 & 26.66 & 16.50 & 25.17 & 20.10 & 30.52 & 33.07 & 32.50 & \textbf{36.91} \\
						\hline
						\hspace{-1pt}\textbf{6} \hspace{-4pt}& 22.28 & 16.50 & 16.50 & 20.90 & 19.37 & 29.56 & 31.99& 32.00& \textbf{38.90} \\
						\hline 
					\end{tabular}
				\end{scriptsize}
			\end{table*}		
			\vspace{-0pt}		
We assess the proposed RNN model in the task of video-frame reconstruction from compressive measurements. The performance is measured using the peak signal-to-noise ratio (PSNR) between the reconstructed $\hat{\bs}_t$ and the original frame $\bs_t$. We use the moving MNIST dataset~\cite{SrivastavaICML15}, which contains 10K video sequences of equal length (20 frames per sequence). Similar to the setup in~\cite{LeArXiv19}, the dataset is split into training, validation, and test sets of 8K, 1K, and 1K sequences, respectively. In order to reduce the training time and memory requirements, we downscale the frames from $64\times 64$ to $16\times 16$ pixels using bilinear decimation. After vectorizing, we obtain sequences of $\bs_1,\dots,\bs_T\in \mathbb{R}^{256}$. Per sequence, we obtain measurements $\bx_1,\dots,\bx_T \in \mathbb{R}^{n}$ using a trainable linear sensing matrix $\bA\in \mathbb{R}^{n\times n_0}$, with $T=20$, $n_0 = 256$ and $n < n_0$.

We compare the reconstruction performance of the proposed reweighted-RNN model against deep-unfolding RNN models, namely, $\ell_1$-$\ell_1$-RNN~\cite{LeArXiv19}, Sista-RNN~\cite{WisdomICASSP17}, and stacked-RNN models, that is, sRNN~\cite{ELMAN90}, LSTM~\cite{hochreiter1997long}, GRU~\cite{cho2014learning}, FastRNN~\cite{KusupatiNIPS18}\footnote{\cite{KusupatiNIPS18} also proposed FastGRNN; we found that, in our application scenario, the non-gated variant (the FastRNN) consistently outperformed FastGRNN. As such, we do not include results with the latter.}, IndRNN~\cite{LiCVPR18} and SpectralRNN~\cite{ZhangICML18}. For the vanilla RNN, LSTM and GRU, the native Pytorch cell implementations were used. The unfolding-based methods were implemented in Pytorch, with Sista-RNN and $\ell_1$-$\ell_1$-RNN tested by reproducing the experiments in~\cite{WisdomICASSP17,LeArXiv19}. For FastRNN, IndRNN, and SpectralRNN cells, we use the publically available Tensorflow implementations. While Sista-RNN, $\ell_1$-$\ell_1$-RNN and reweighted-RNN have their own layer-stacking schemes derived from unfolding minimization algorithms, we use the stacking rule in~\cite{pascanu2013construct} [see Fig~\ref{fig-stackedRNN}] to build deep networks for other RNN architectures.

Our default settings are: a compressed sensing (CS) rate of $n/n_0=0.2$, $d=3$ hidden layers\footnote{In our experiments, the 2-layer LSTM network outperforms the 3-layer one (see Table~\ref{tab:result-depth}), the default setting for LSTM is thus using 2 layers.} with $h=2^{10}$ hidden units per layer. In each set of experiments, we vary each of these hyper-parameters while keeping the other two fixed. For the unfolding methods, the overcomplete dictionary $\bD\in \mathbb{R}^{n_0\times h}$ is initialized with the discrete cosine transform (DCT) with varying dictionary sizes of $h=\{2^7,~2^8,~2^9,~2^{10},~2^{11},~2^{12}\}$ (corresponding to a number of hidden neurons in the other methods). For initializing $\lambda_1,\lambda_2$ [see \eqref{seqquentialProblem}, \eqref{reweighted-l1-l1minimization}], we perform a random search in the range of [$10^{-5}$,~3.0] in the validation set. To avoid the problem of exploding gradients, we clip the gradients during backpropagation such that the $\ell_2$-norms are less than or equal to 0.25. We do not apply weight decay regularization as we found it often leads to worse performance, especially since gradient clipping is already used for training stability. We train the networks for 200 epochs using the Adam optimizer with an initial learning rate of 0.0003, and a batch size of 32. During training, if the validation loss does not decrease for 5 epochs, we reduce the learning rate to 0.3 of its current value.

Table \ref{tab:result-rate} summarizes the reconstruction results for different CS rates $n/n_0$. The reweighted-RNN model systematically outperforms the other models, often by a large margin. Table~\ref{tab:result-width} shows similar improvements for various dimensions of hidden units. Table~\ref{tab:result-depth} shows that IndRNN delivers higher reconstruction performance than our model when a small number of hidden layers ($d = 1, 2$) is used. Moreover, when the depth increases, reweighted-RNN surpasses all other models. Our network also has fewer trainable parameters compared to the popular variants of RNN. At the default settings, reweighted-RNN, the stacked vanilla RNN, the stacked LSTM, and the stacked GRU have 4.47M, 5.58M, 21.48M, and 16.18M parameters, respectively.

In our experiments, we use the publically available Tensorflow implementations for FastRNN\footnote{Code available at https://github.com/microsoft/EdgeML}, IndRNN\footnote{Code available at https://github.com/batzner/indrnn}, and SpectralRNN\footnote{Code available at https://github.com/zhangjiong724/spectral-RNN} cells. While Sista-RNN, $\ell_1$-$\ell_1$-RNN and reweighted-RNN have their own layer-stacking schemes derived from unfolding minimization algorithms, we use the stacking rule in~\cite{pascanu2013construct} [see Fig~\ref{fig-stackedRNN}] to build deep networks for other RNN models. 

			\begin{figure}[h]
	\centering
	\subfigure[sRNN \cite{ELMAN90}.]{\label{subfig:rnn}\includegraphics[width=0.24\textwidth]{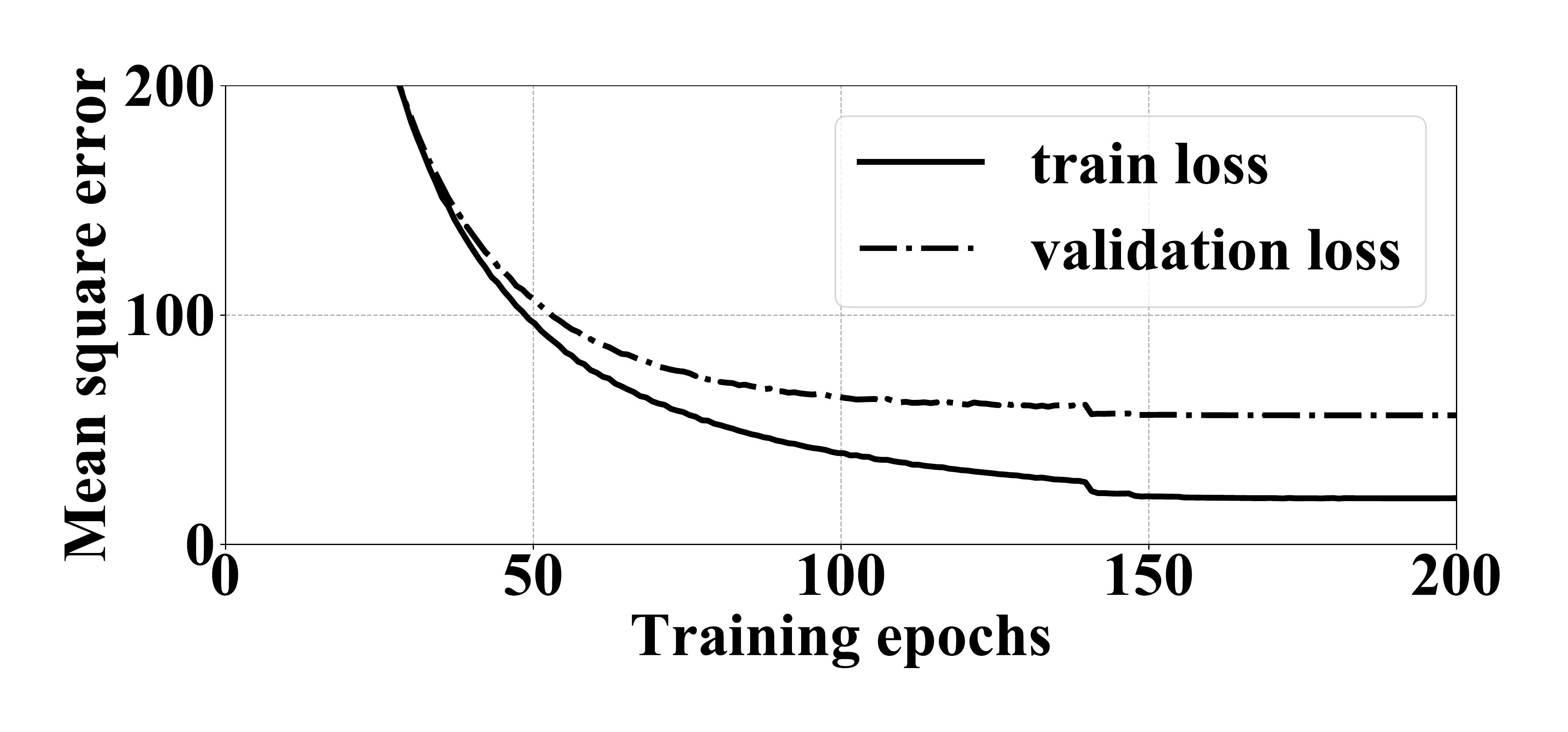}}
	\subfigure[LSTM \cite{hochreiter1997long}.]{\label{subfig:lstm}\includegraphics[width=0.24\textwidth]{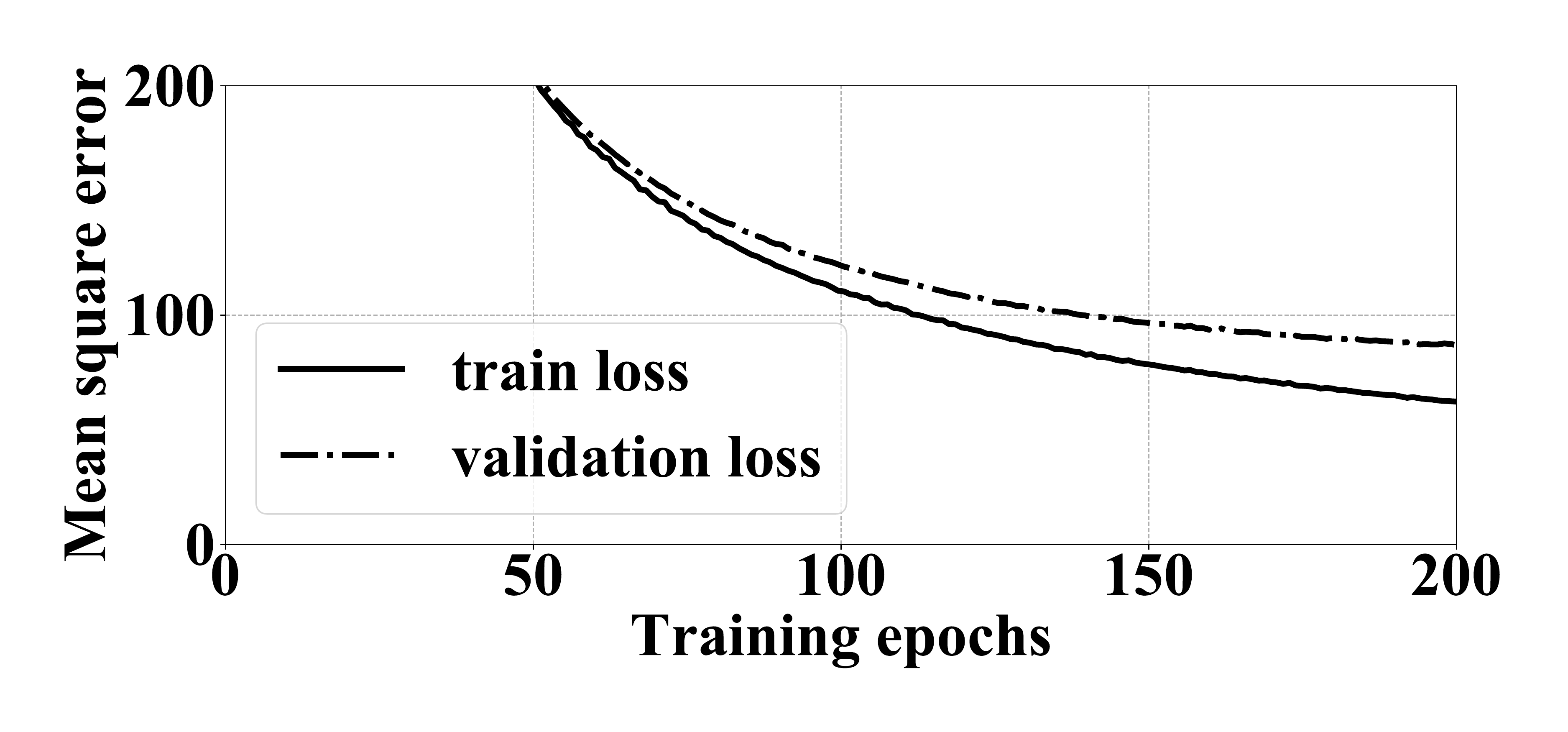}}
	\subfigure[GRU \cite{cho2014learning}.]{\label{subfig:gru}\includegraphics[width=0.24\textwidth]{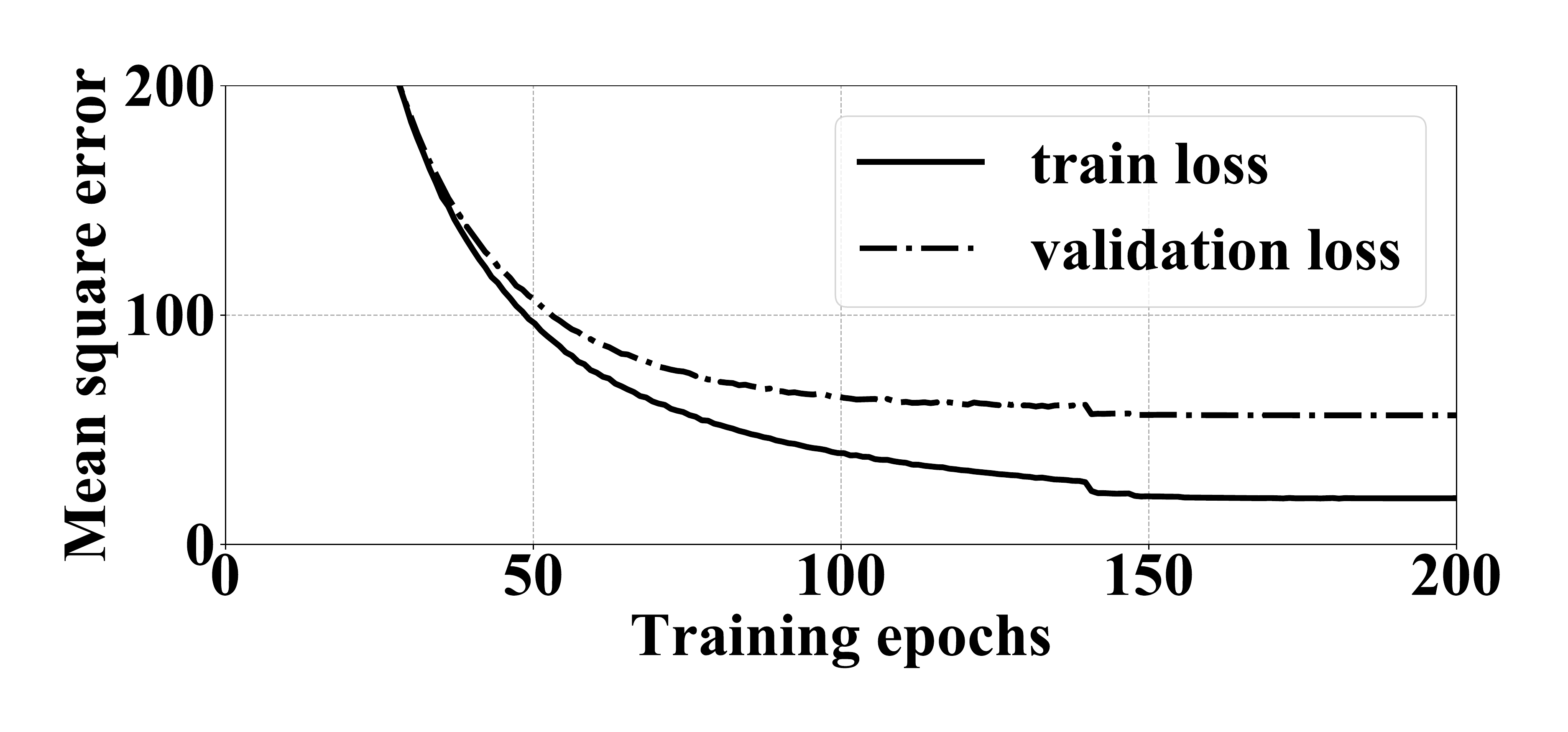}}
	\subfigure[IndRNN \cite{LiCVPR18}.]{\label{subfig:indrnn}\includegraphics[width=0.24\textwidth]{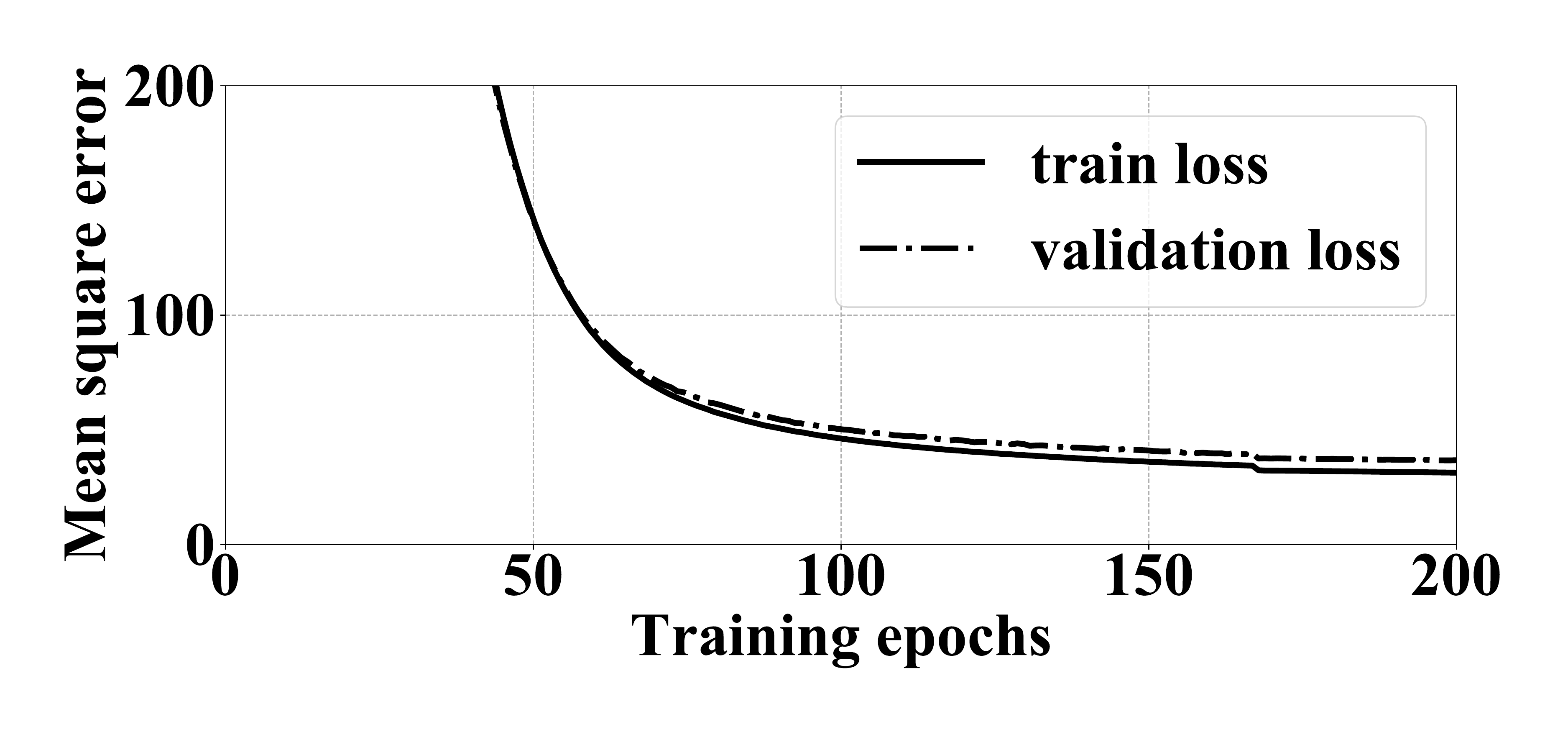}}			
	\subfigure[FastRNN \cite{KusupatiNIPS18}.]{\label{subfig:fastrnn}\includegraphics[width=0.24\textwidth]{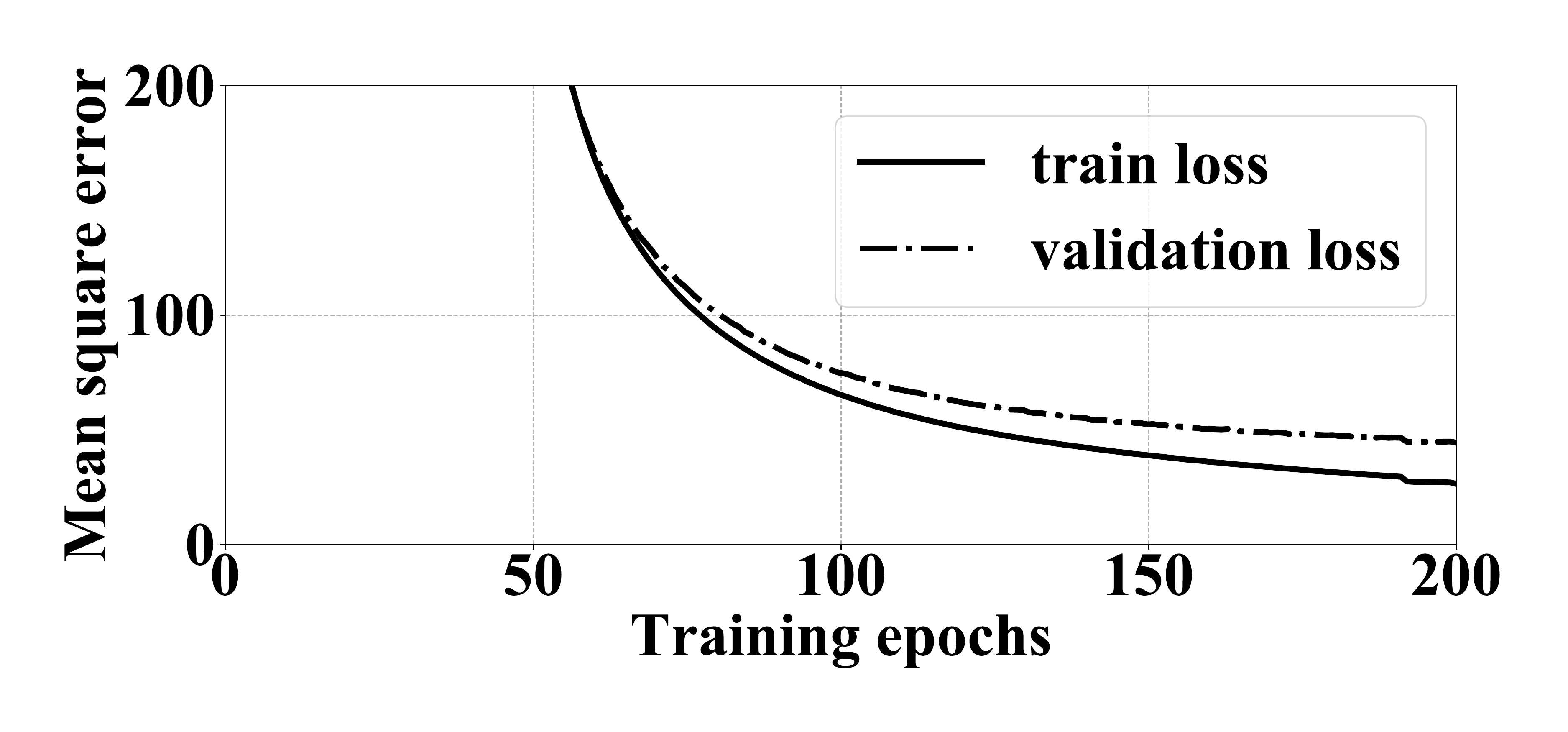}}
	\subfigure[SpectralRNN \cite{ZhangICML18}.]{\label{subfig:spectral}\includegraphics[width=0.24\textwidth]{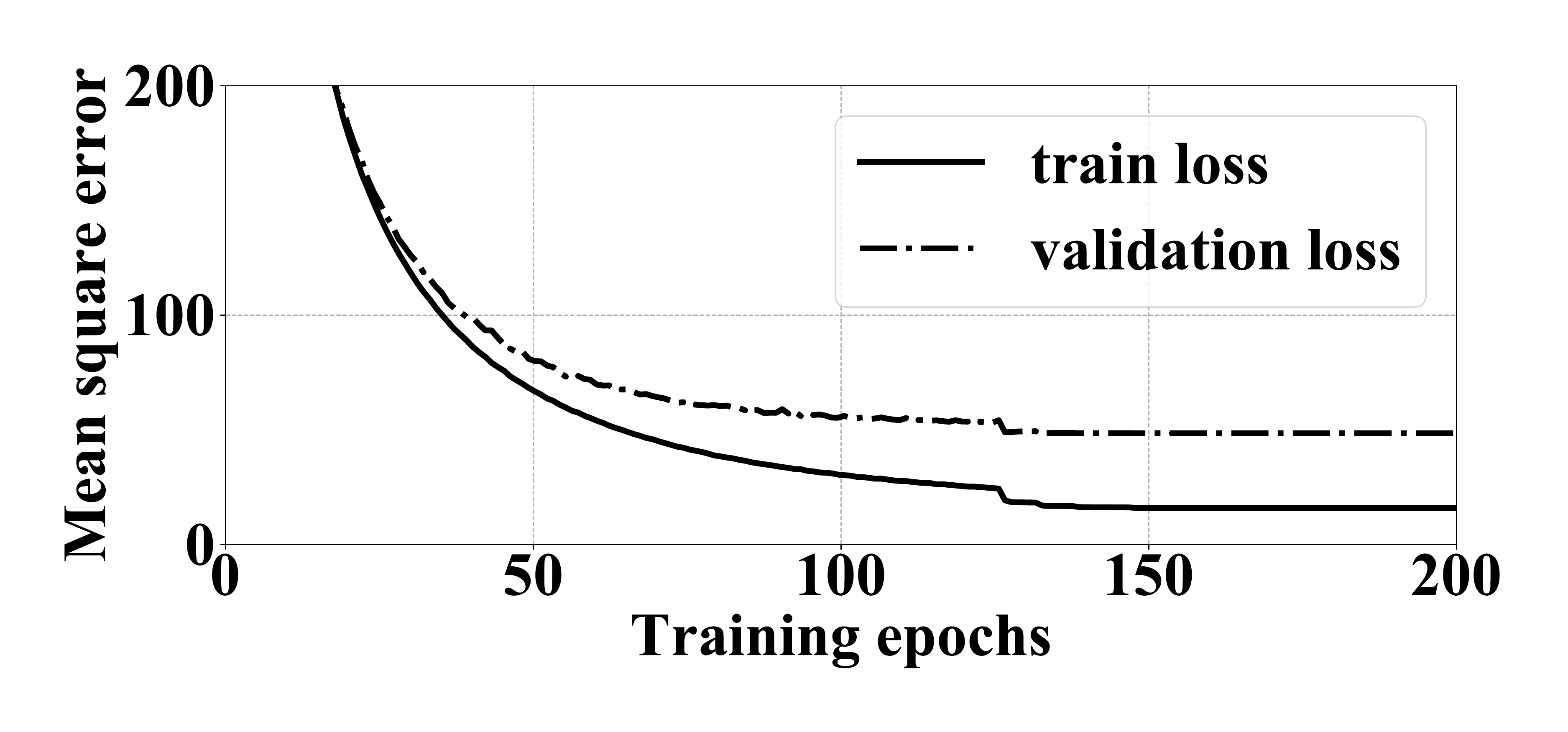}}
	\subfigure[Sista-RNN \cite{WisdomICASSP17}.]{\label{subfig:sista}\includegraphics[width=0.24\textwidth]{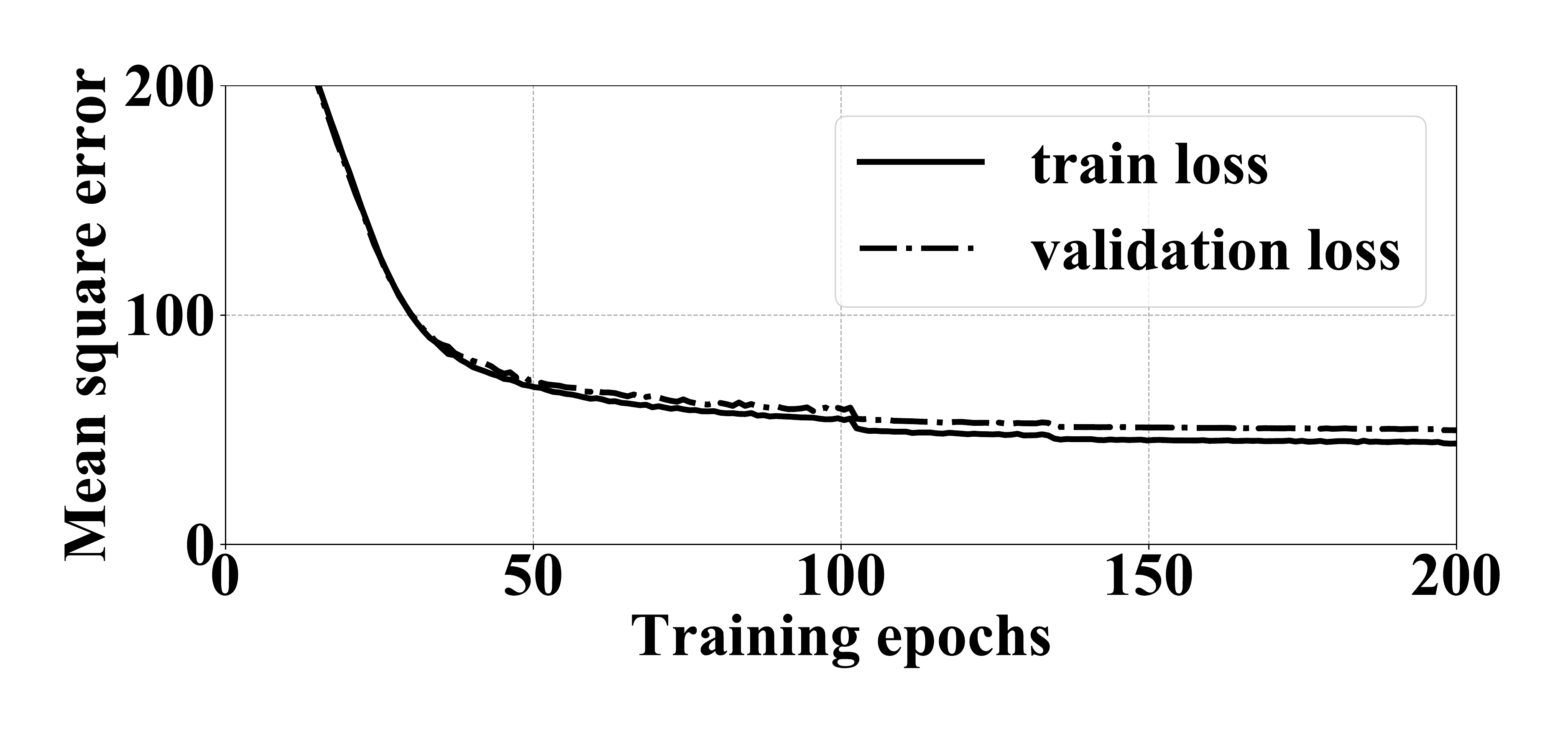}}
	\subfigure[$\ell_1$-$\ell_1$-RNN \cite{LeArXiv19}.]{\label{subfig:l1l1}\includegraphics[width=0.24\textwidth]{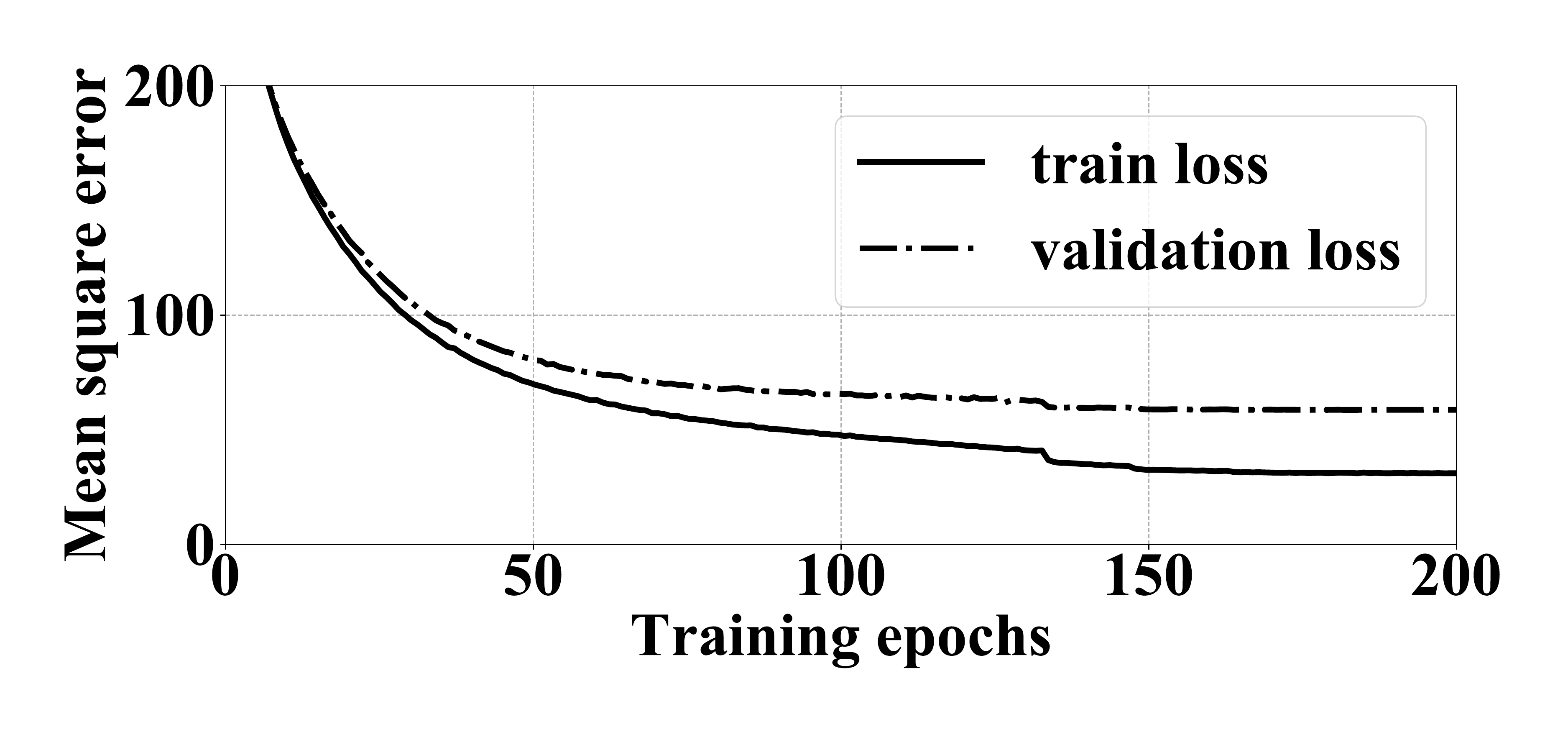}}
	\subfigure[reweighted-RNN.]{\label{subfig:reweighted}\includegraphics[width=0.24\textwidth]{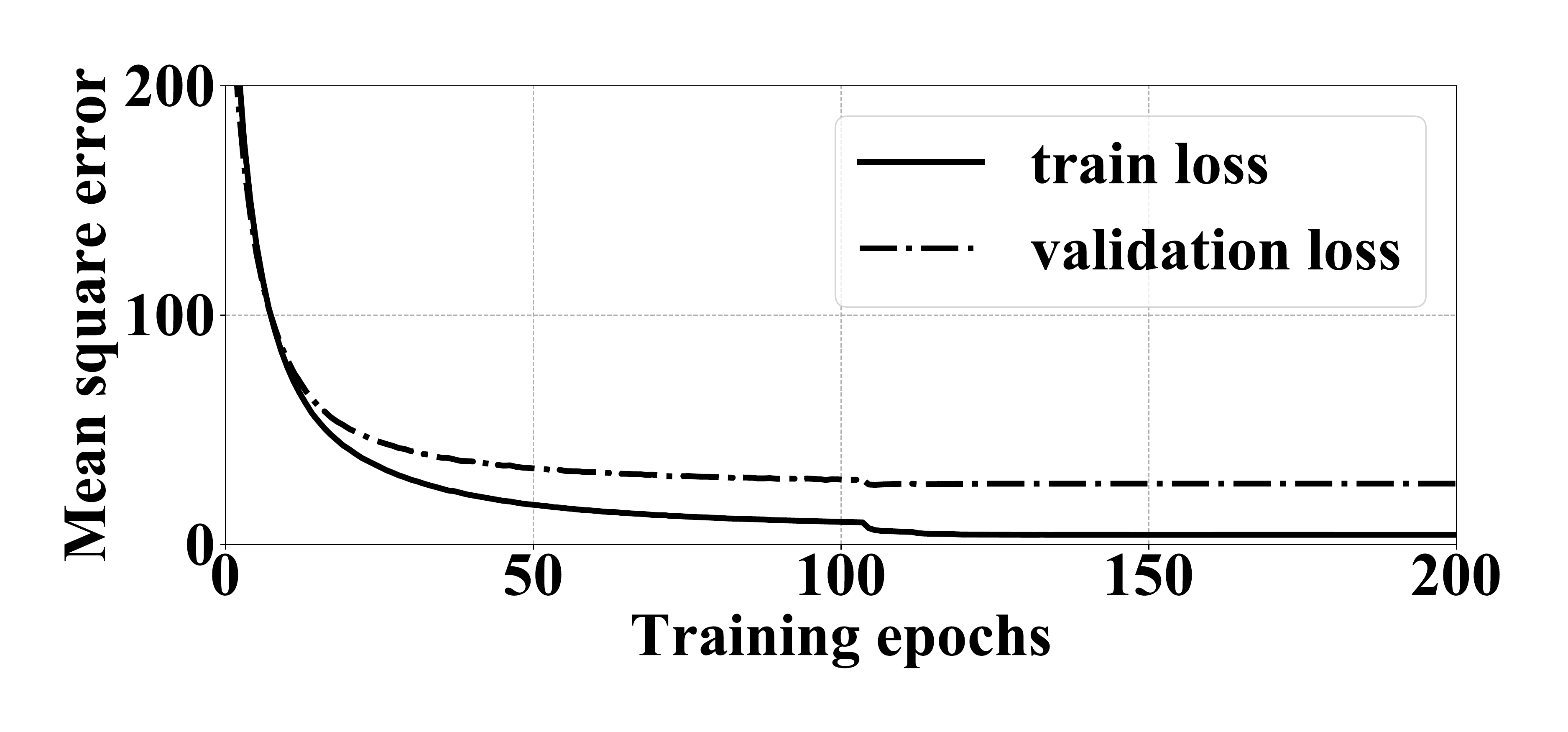}}
	\caption{Average mean square error vs. training epoches on the training and the validation sets for the default setting (a CS rate is $0.2$,~$d=3$,~$h=2^{10}$).}
	\label{fig:mse}
\end{figure} 	
\begin{figure}[h]
	\centering
	\subfigure[CS Rate is 0.1.]{\label{subfig:rnn}\includegraphics[width=0.24\textwidth]{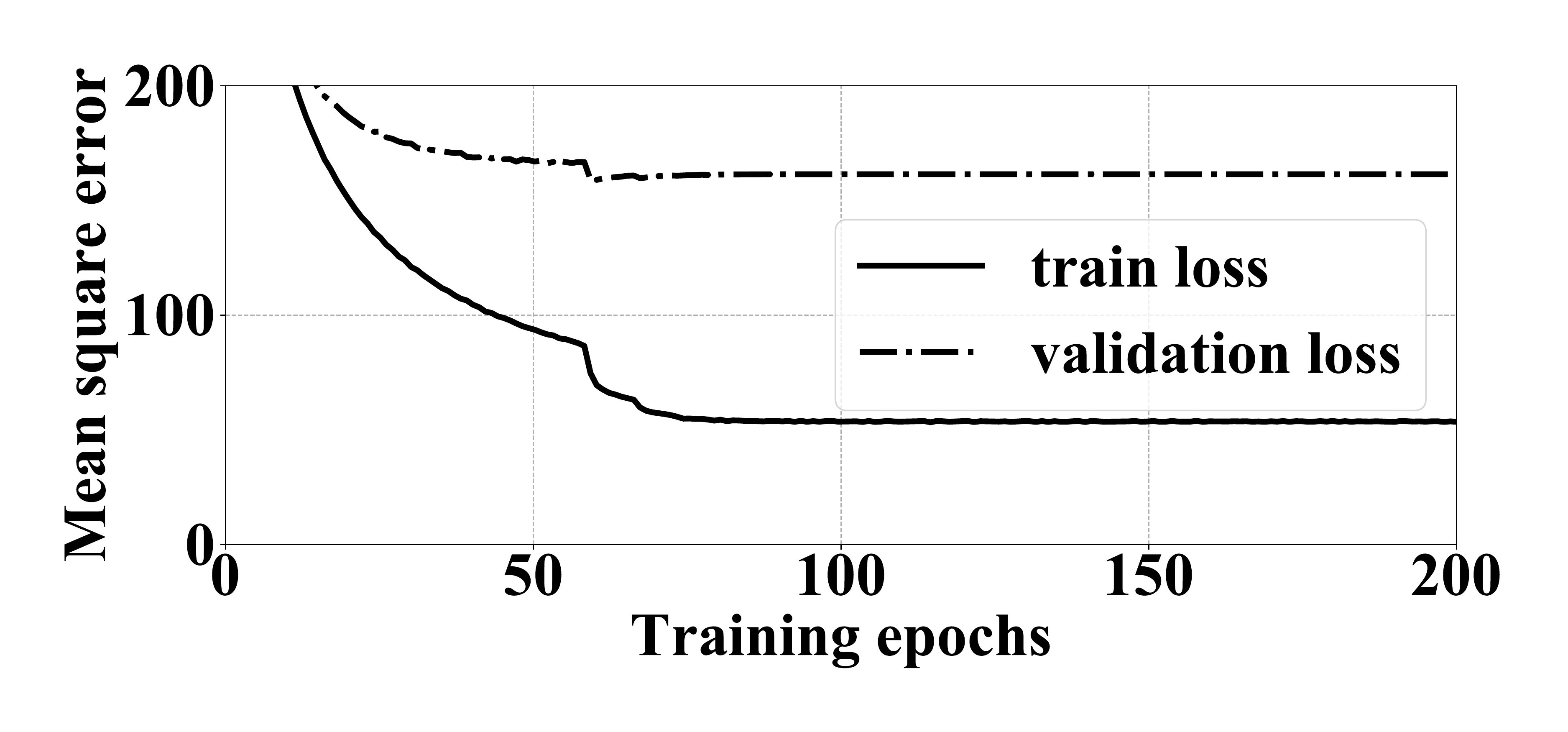}}
	\subfigure[CS Rate is 0.2.]{\label{subfig:reweighted}\includegraphics[width=0.24\textwidth]{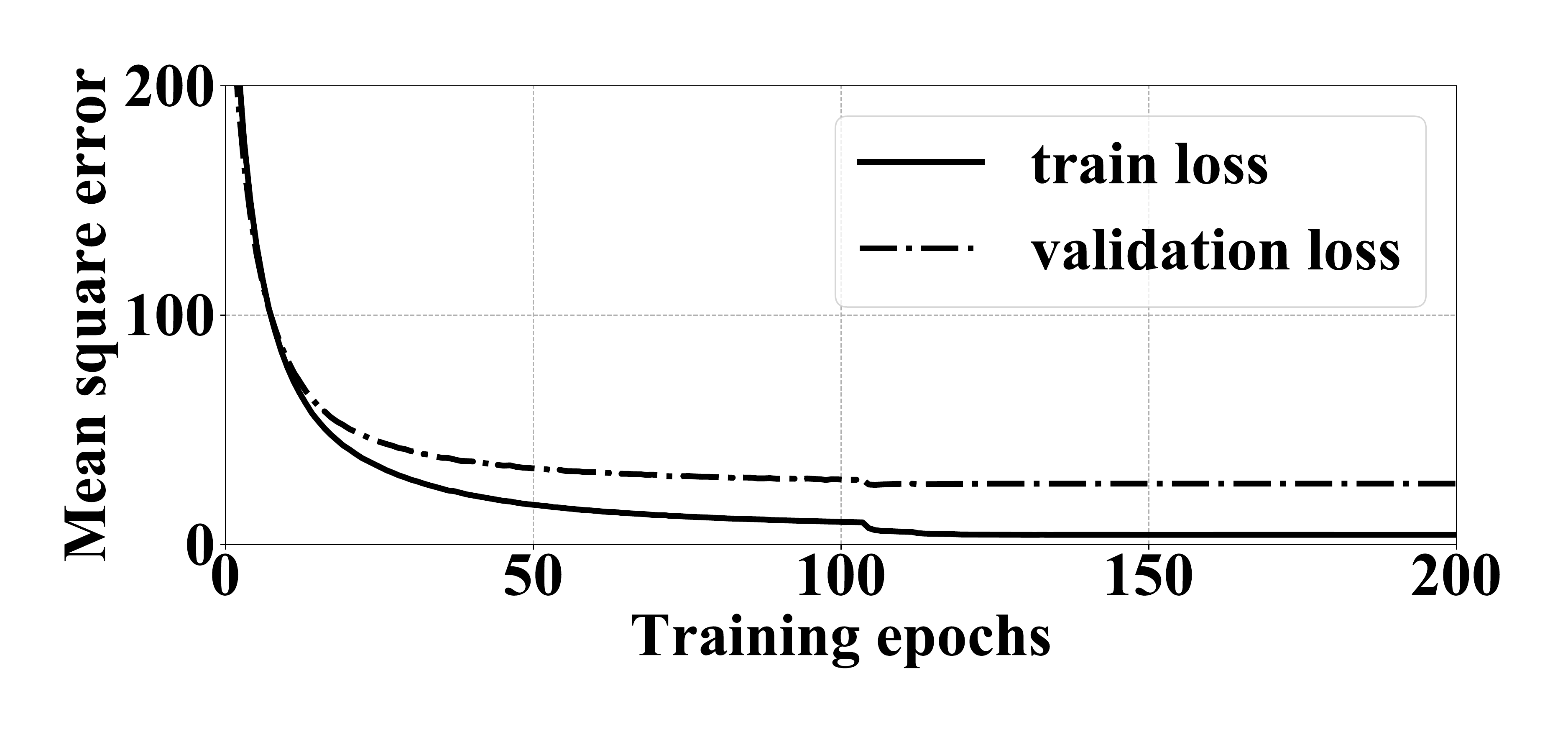}}
	\subfigure[CS Rate is 0.3.]{\label{subfig:rnn}\includegraphics[width=0.24\textwidth]{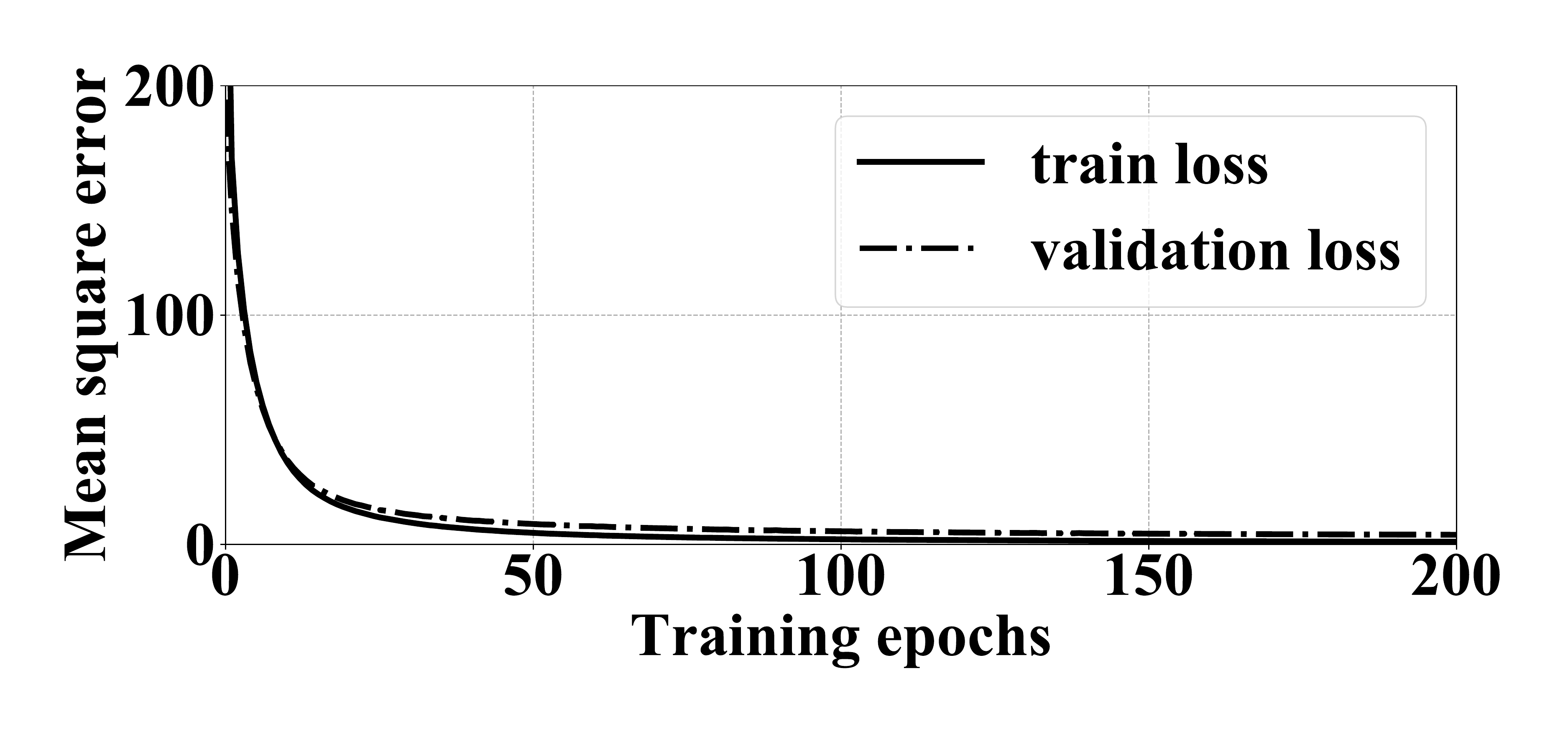}}
	\subfigure[CS Rate is 0.4.]{\label{subfig:rnn}\includegraphics[width=0.24\textwidth]{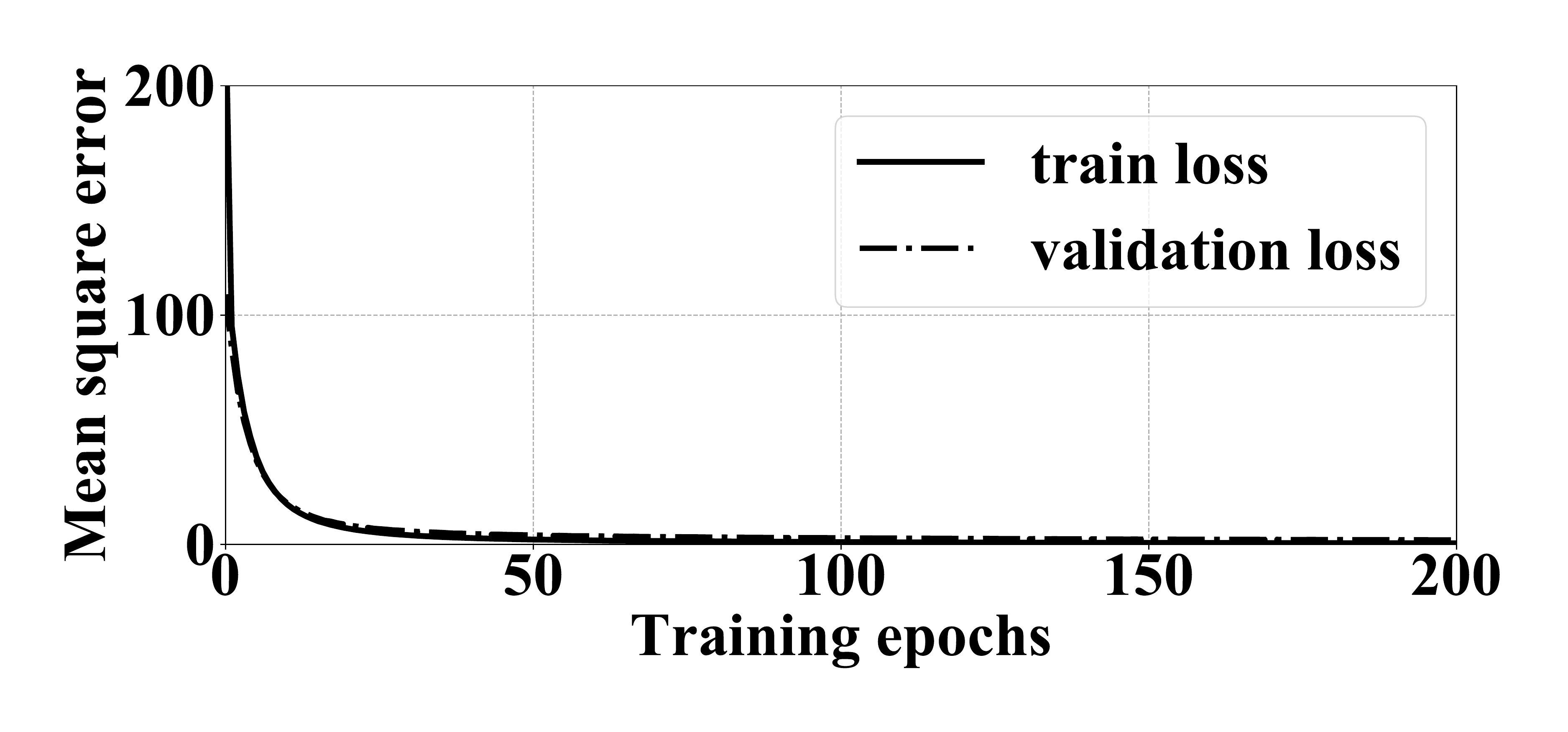}}
	\subfigure[CS Rate is 0.5.]{\label{subfig:rnn}\includegraphics[width=0.24\textwidth]{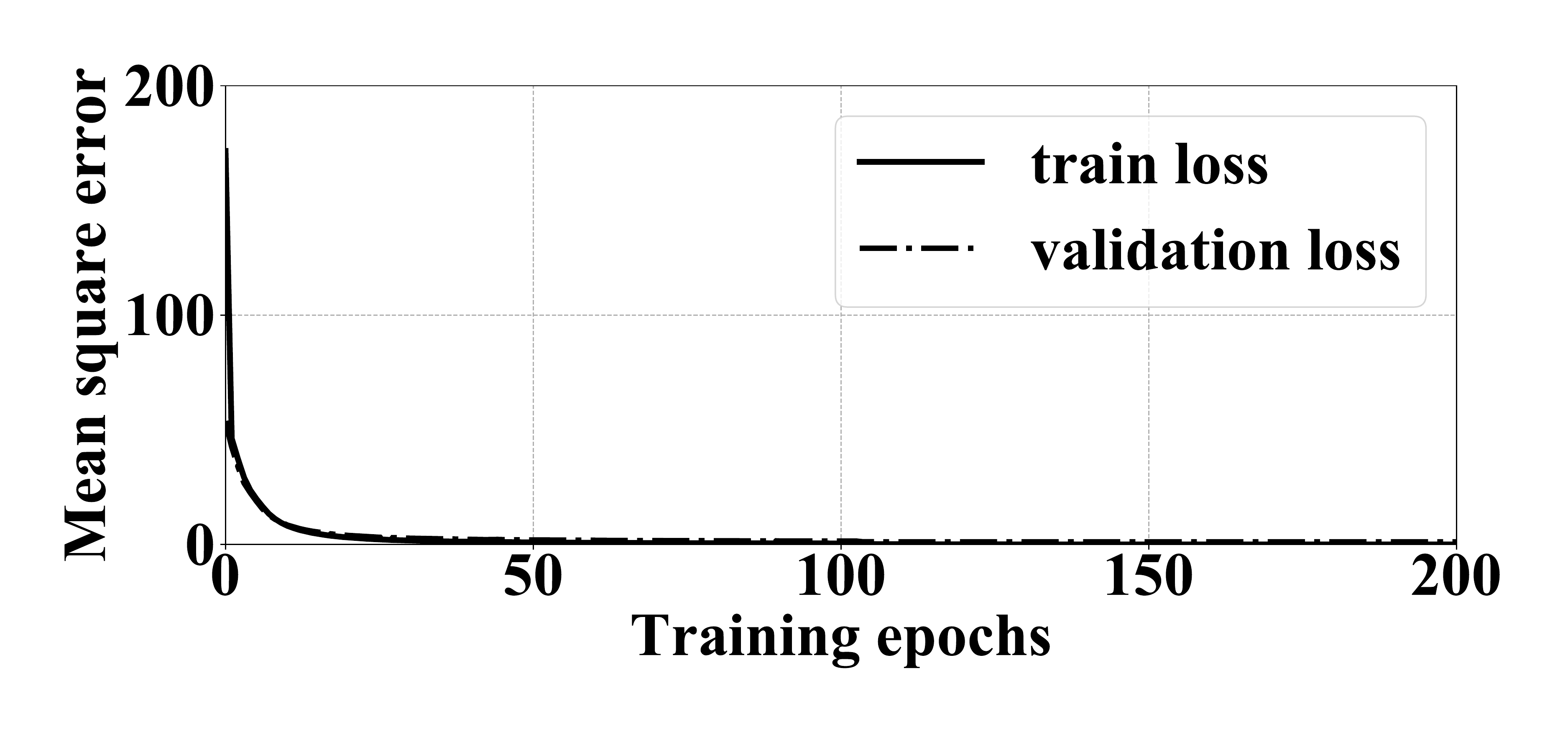}}
	\caption{Reweighted-RNN: Average mean square error vs. training epoches on the training and the validation sets with different CS rates ($d=3$,~$h=2^{10}$).}
	\label{fig:mseRate}
\end{figure} 
\begin{figure}[h]
	\centering
	\subfigure[$h=2^7$.]{\label{subfig:rnn}\includegraphics[width=0.24\textwidth]{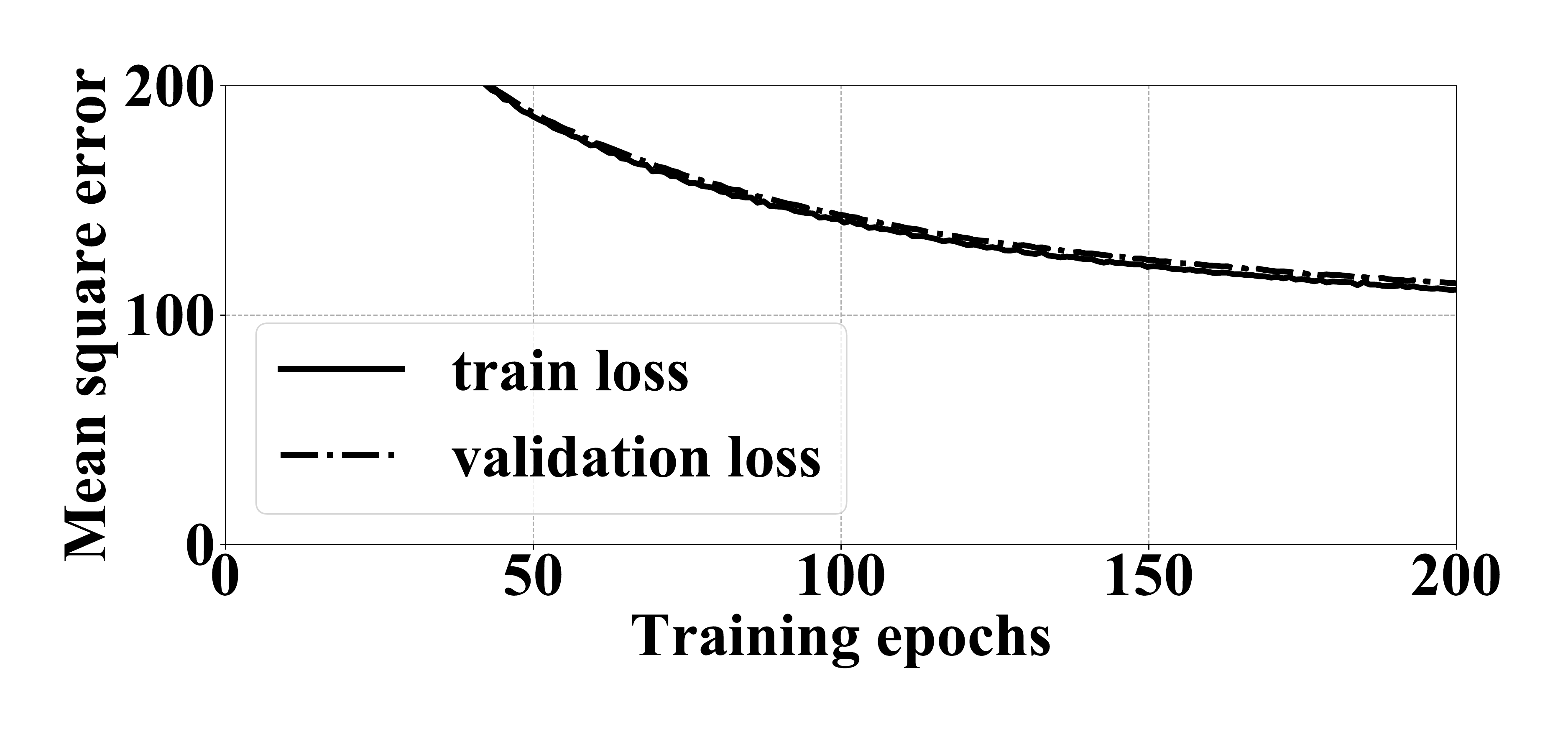}}
	\subfigure[$h=2^8$.]{\label{subfig:reweighted}\includegraphics[width=0.24\textwidth]{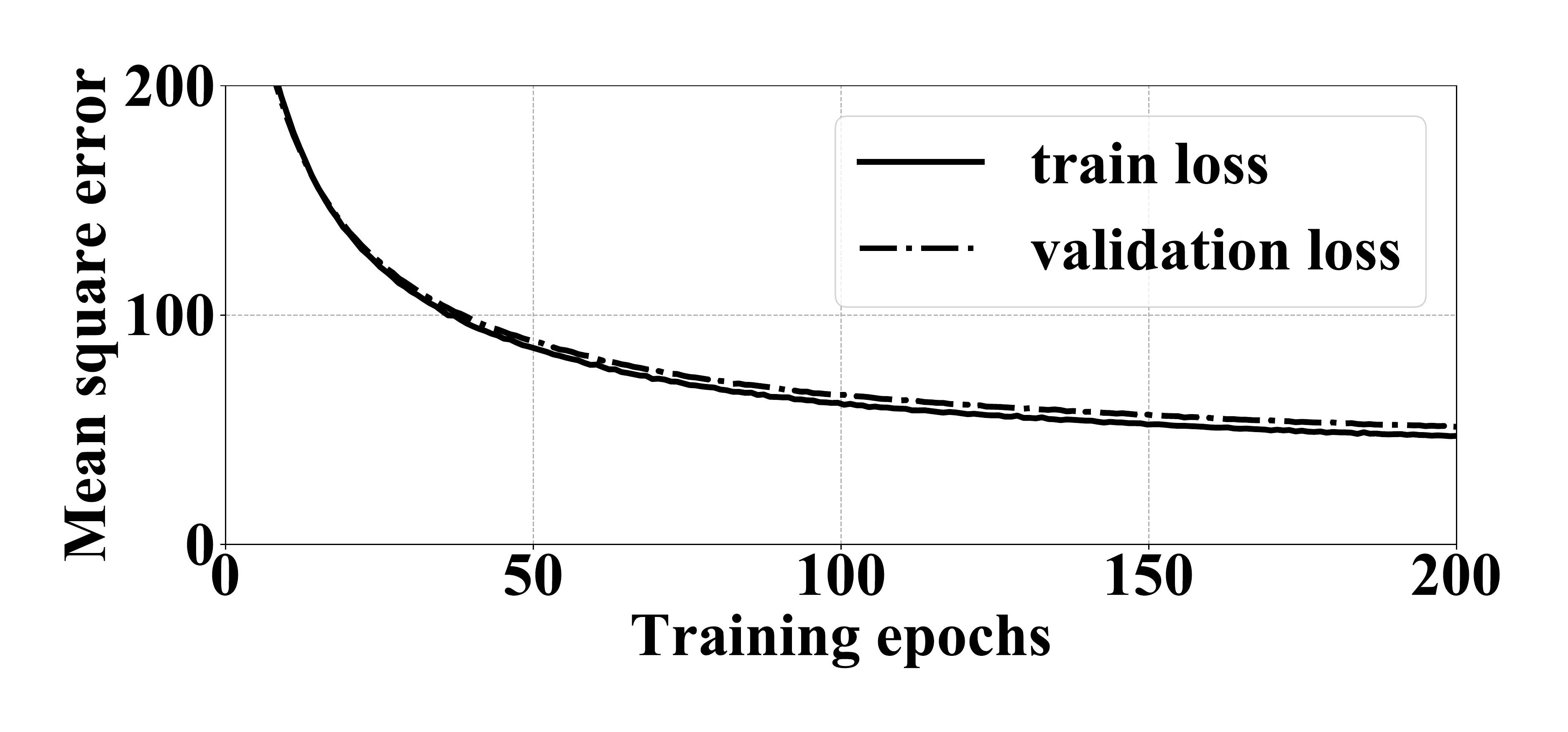}}
	\subfigure[$h=2^9$.]{\label{subfig:rnn}\includegraphics[width=0.24\textwidth]{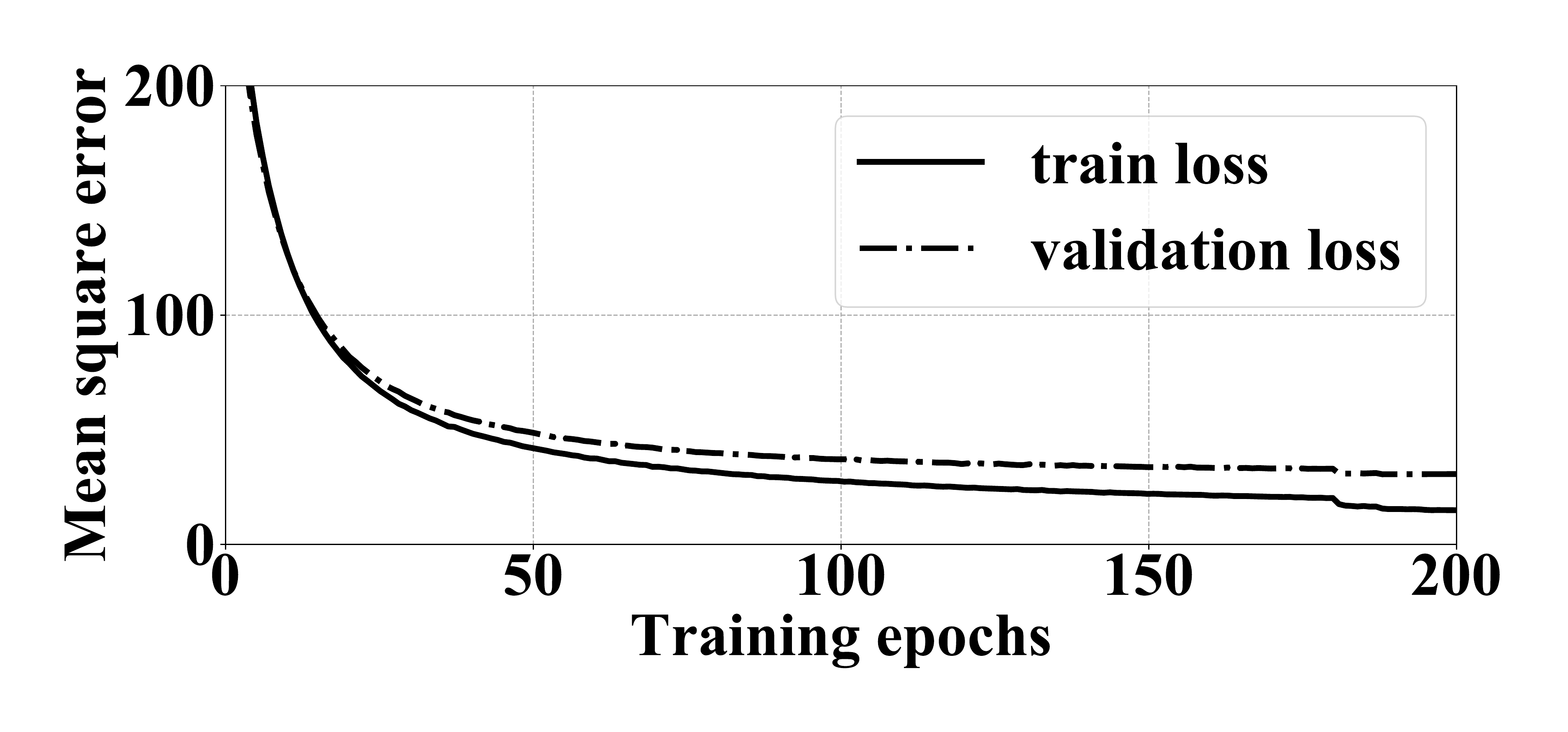}}
	\subfigure[$h=2^{10}$.]{\label{subfig:rnn}\includegraphics[width=0.24\textwidth]{20.pdf}}
	\subfigure[$h=2^{11}$.]{\label{subfig:rnn}\includegraphics[width=0.24\textwidth]{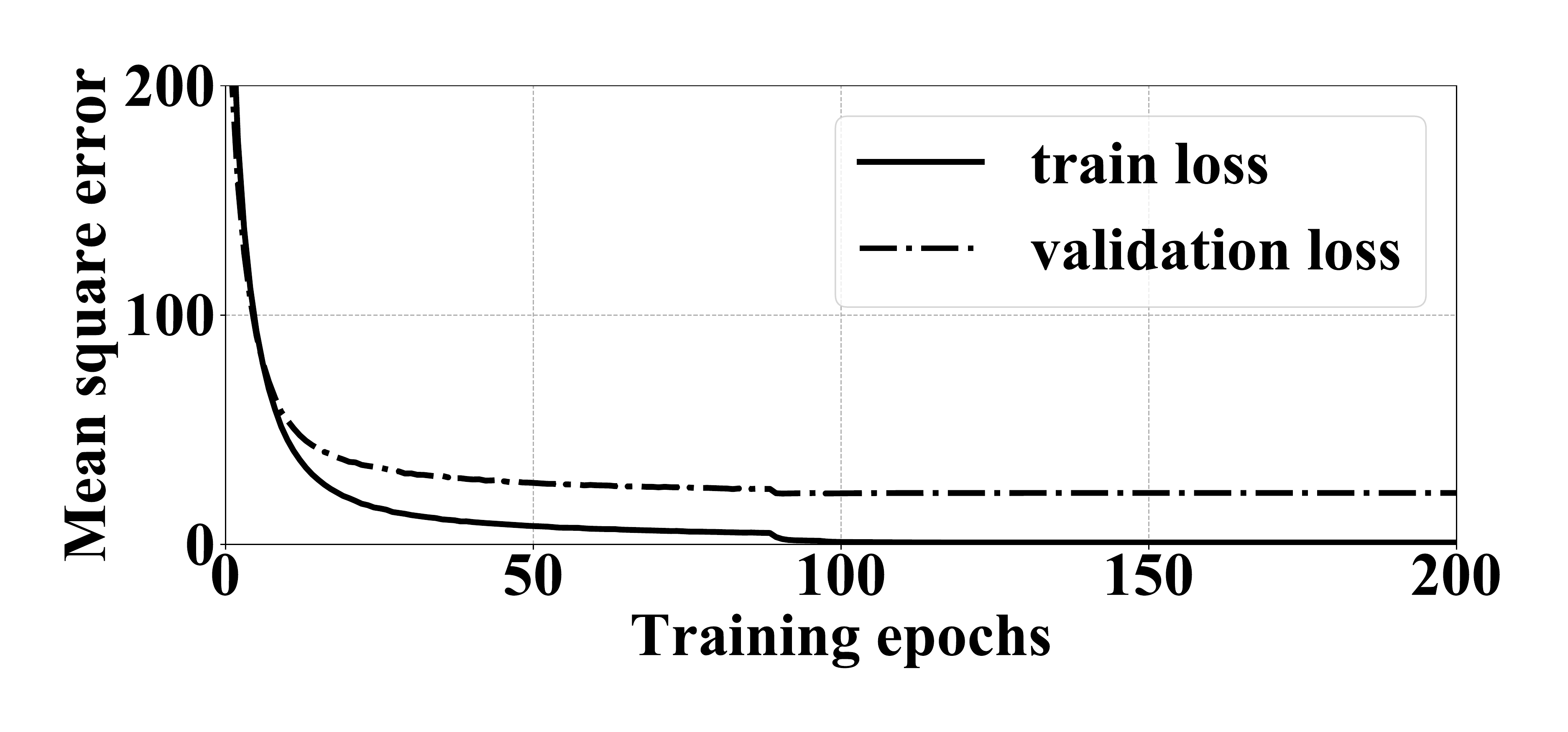}}
	\subfigure[$h=2^{12}$.]{\label{subfig:rnn}\includegraphics[width=0.24\textwidth]{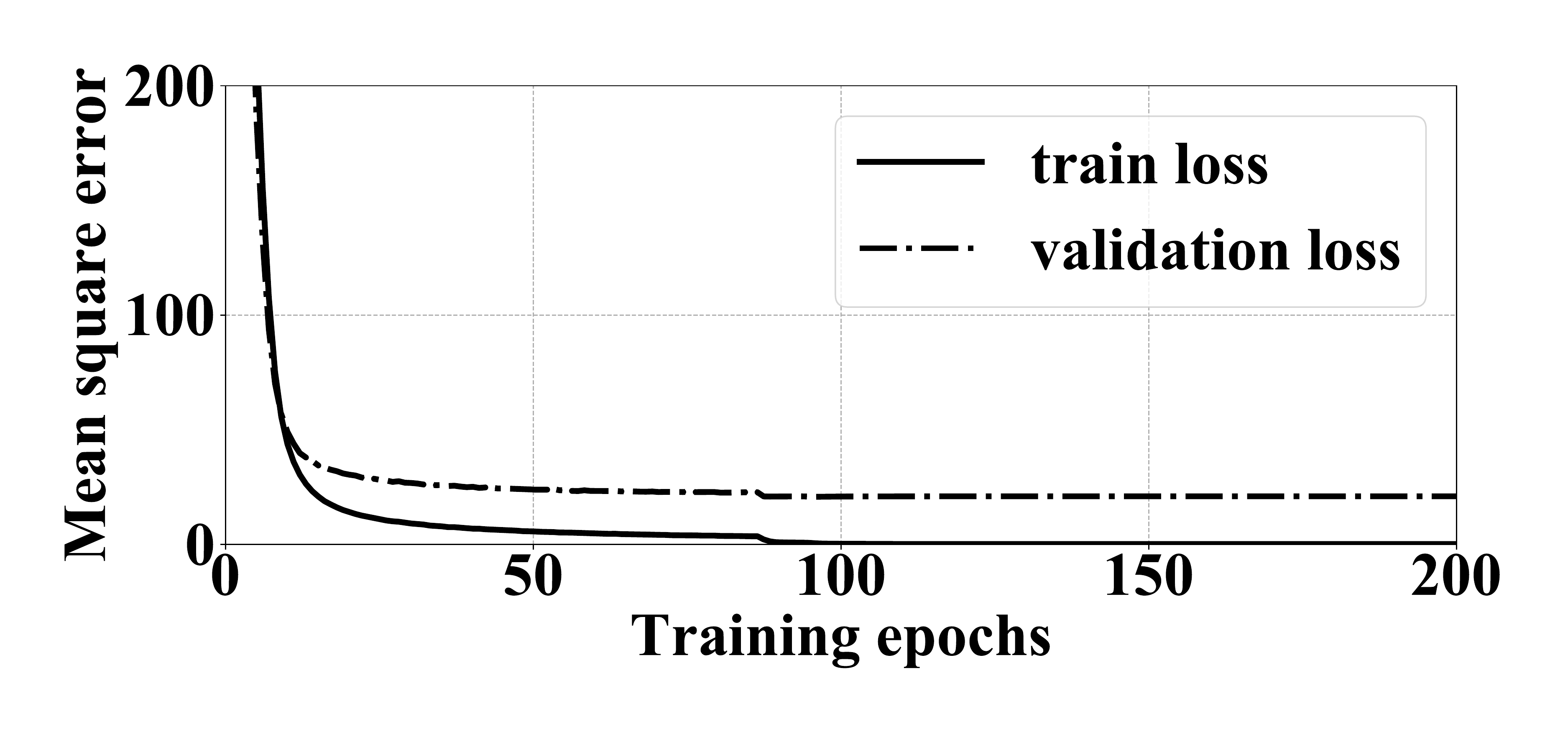}}
	\caption{Reweighted-RNN: Average mean square error vs. training epoches on the training and the validation sets with different network widths $h$ (a CS rates is $0.2$,~$d=3$).}
	\label{fig:mseWidth}
\end{figure} 
\begin{figure}[h]
	\centering
	\subfigure[$d=1$.]{\label{subfig:rnn}\includegraphics[width=0.24\textwidth]{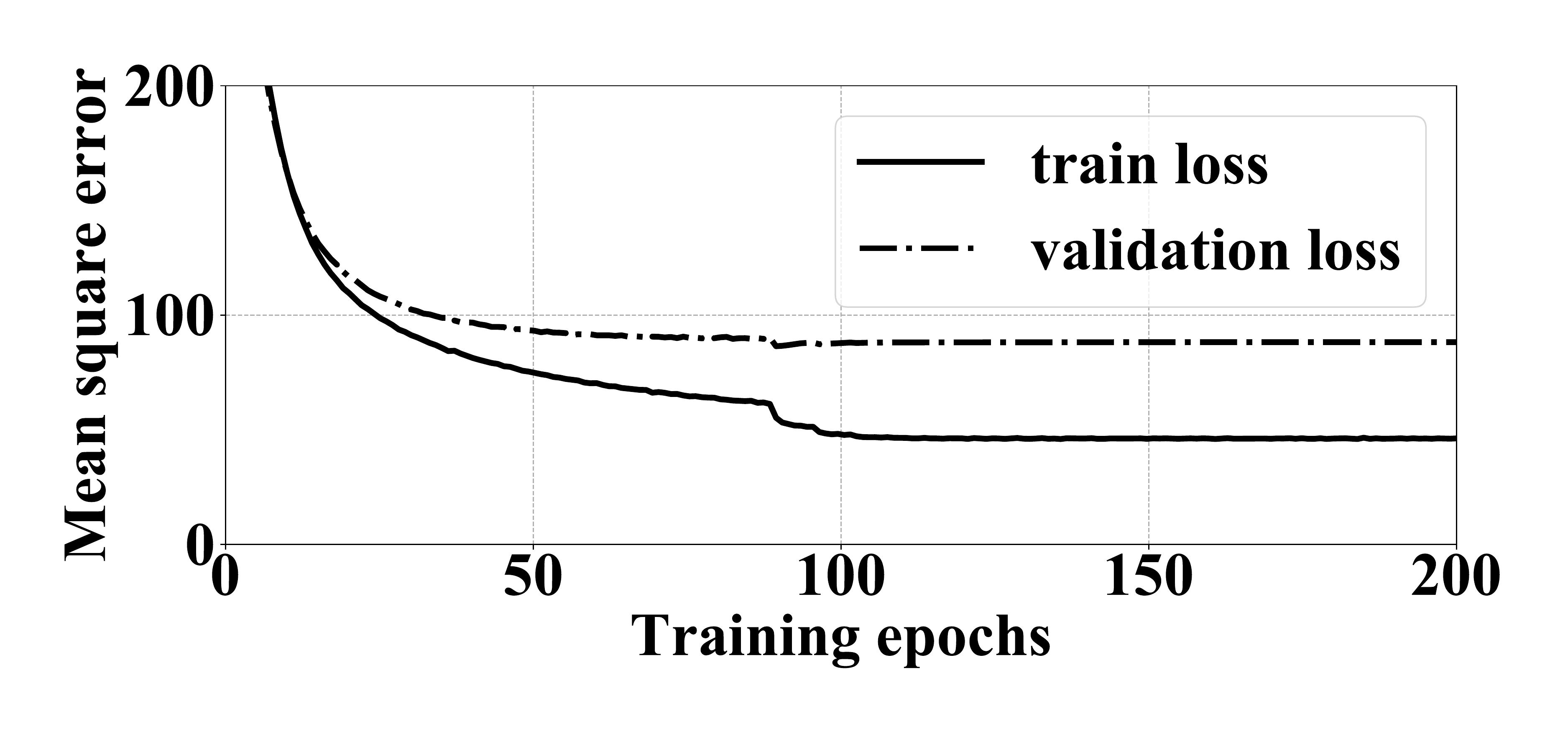}}
	\subfigure[$d=2$.]{\label{subfig:reweighted}\includegraphics[width=0.24\textwidth]{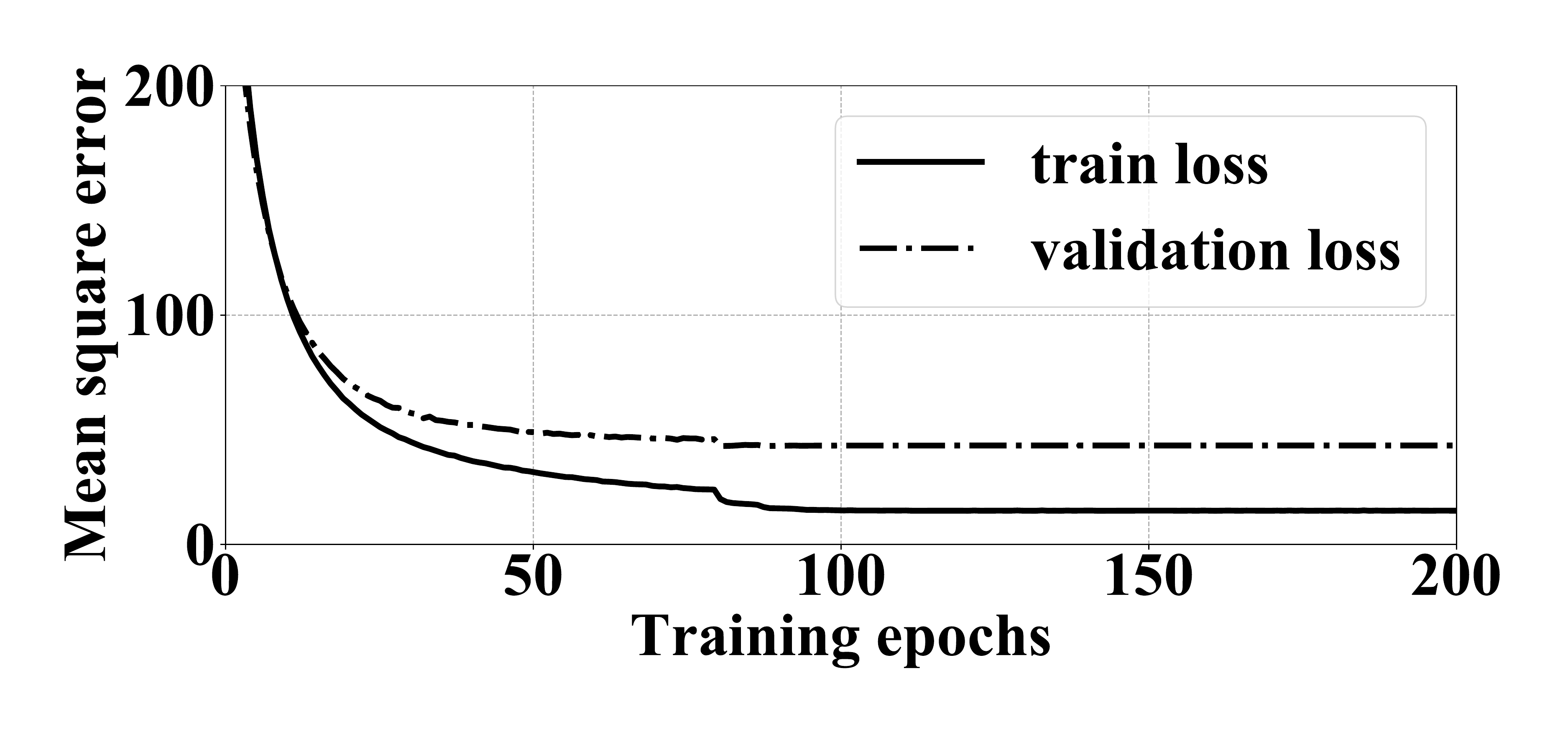}}
	\subfigure[$d=3$.]{\label{subfig:rnn}\includegraphics[width=0.24\textwidth]{20.pdf}}
	\subfigure[$d=4$.]{\label{subfig:rnn}\includegraphics[width=0.24\textwidth]{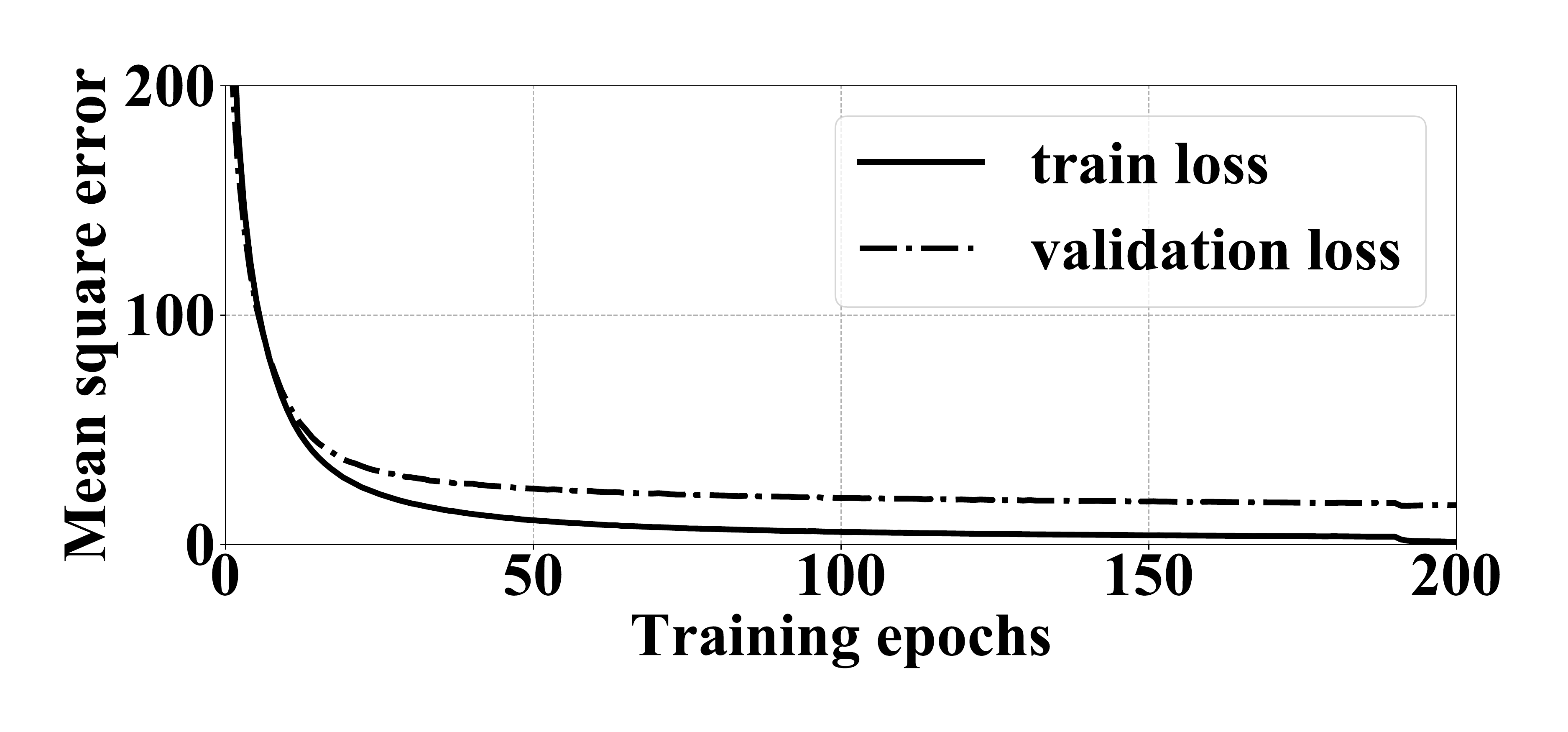}}
	\subfigure[$d=5$.]{\label{subfig:rnn}\includegraphics[width=0.24\textwidth]{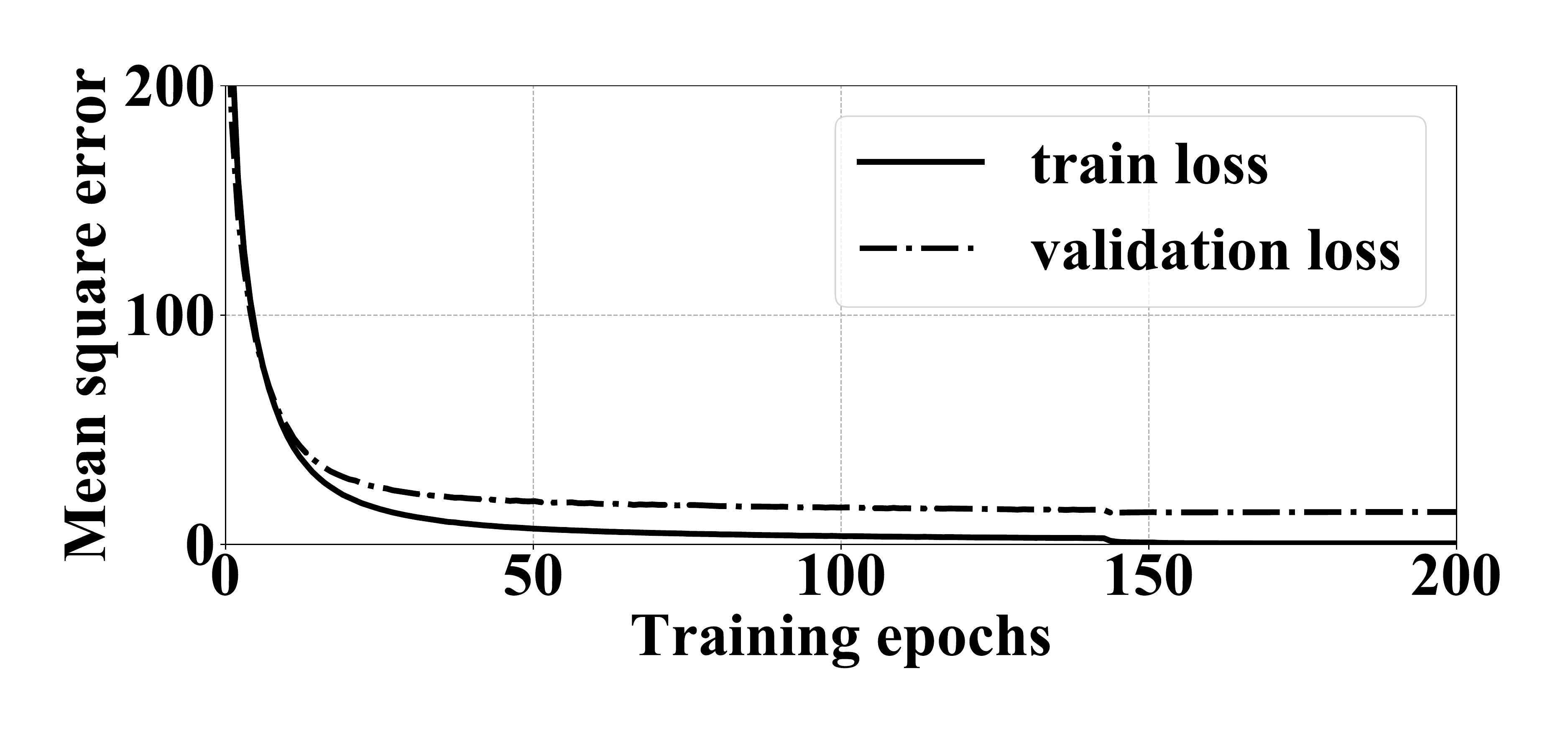}}
	\subfigure[$d=6$.]{\label{subfig:rnn}\includegraphics[width=0.24\textwidth]{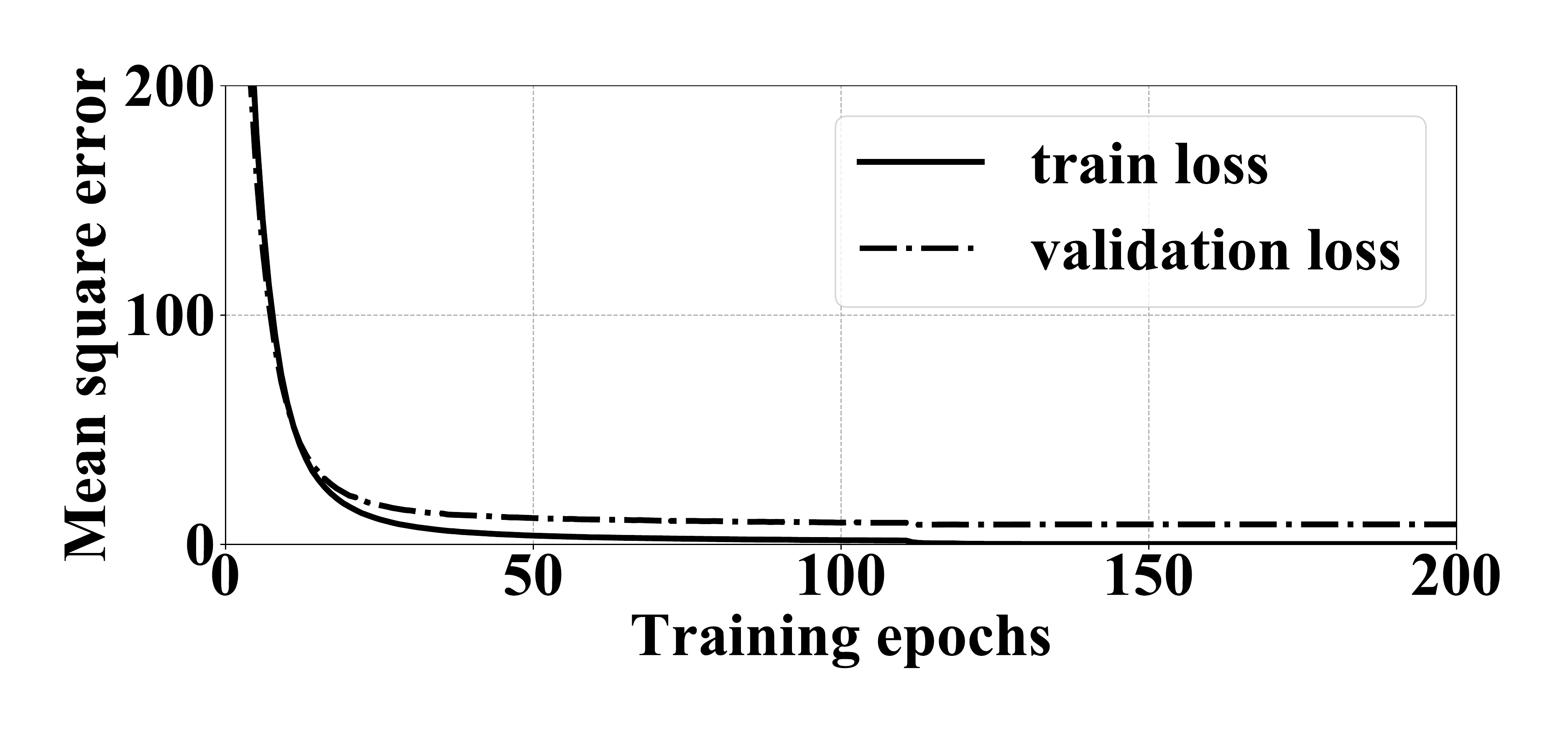}}
	\caption{Reweighted-RNN: Average mean square error vs. training epoches on the training and the validation sets with different network depths $d$ (a CS rate is $0.2$,~$h=2^{10}$).}
	\label{fig:mseDepth}
\end{figure} 
Figure~\ref{fig:mse} shows the learning curves of all methods under the default setting. It can be seen that reweighted-RNN achieves the lowest mean square error on both the training and validation sets. It can also be observed that the unfolding methods converge faster than the stacked RNNs, with the proposed reweighted-RNN being the fastest. More experimental results for the proposed reweighted-RNN are provided to illustrate the learning curves, which measure the average mean square error vs. the training epochs between the original frames and their reconstructed counterparts, with different CS rates [Fig. \ref{fig:mseRate}], different network depths $d$ [Fig. \ref{fig:mseDepth}], and different network widths $h$ [Fig.~\ref{fig:mseWidth}].	

Since we use different frameworks to implement the RNNs used in our benchmarks, we do not report and compare the computational time for training of the models. Specifically, we rely on the Tensorflow implementations from the authors of Independent-RNN, Fast-RNN and Spectral RNN, while the rest is written in Pytorch. Furthermore, even among Pytorch models, the vanilla RNN, LSTM, and GRU cells are written in CuDNN (default Pytorch implementations), so that they are significantly faster in training than the others. This does not mean that these networks have better runtime complexities, but rather more efficient implementations. However, an important comparison could be made between $\ell_1$-$\ell_1$-RNN~\cite{LeArXiv19} (as the baseline method) and Reweighted-RNN due to their similarities in implementations. At the default settings, it takes 3,521 seconds and 2,985 seconds to train Reweighted-RNN and $\ell_1$-$\ell_1$-RNN~\cite{LeArXiv19}, respectively.

	\subsection{Additional tasks}
We test our model on three popular tasks for RNNs, namely the sequential pixel MNIST classification, the adding task, and the copy task~\cite{LeArXiv15,ArjovskyICML16,ZhangICML18}.

\textbf{Sequential pixel MNIST and permuted pixel MNIST classification}. This task aims to classify MNIST images to a class label. MNIST images are formed by a 28$\times$28 gray-scale image with a label from 0 to 9. We use the reweighted-RNN along with a softmax for category classification. We set $d=5$ layers and $h = 256$ hidden units for the reweighted-RNN. We consider two scenarios: the first one where the pixels of each MNIST image are read in the order from left-to-right and bottom-to-top and the second one where the pixels of each MNIST image are randomly permuted. The classification accuracy results are shown in Fig.~\ref{pixelMNIST} (for pixel MNIST) and Fig. \ref{pixelPermutedMNIST} (for permuted pixel MNIST).

\begin{figure}[h!]
	\centering 
	\subfigure[Pixel MNIST.]{\label{pixelMNIST}\includegraphics[width=0.24\textwidth]{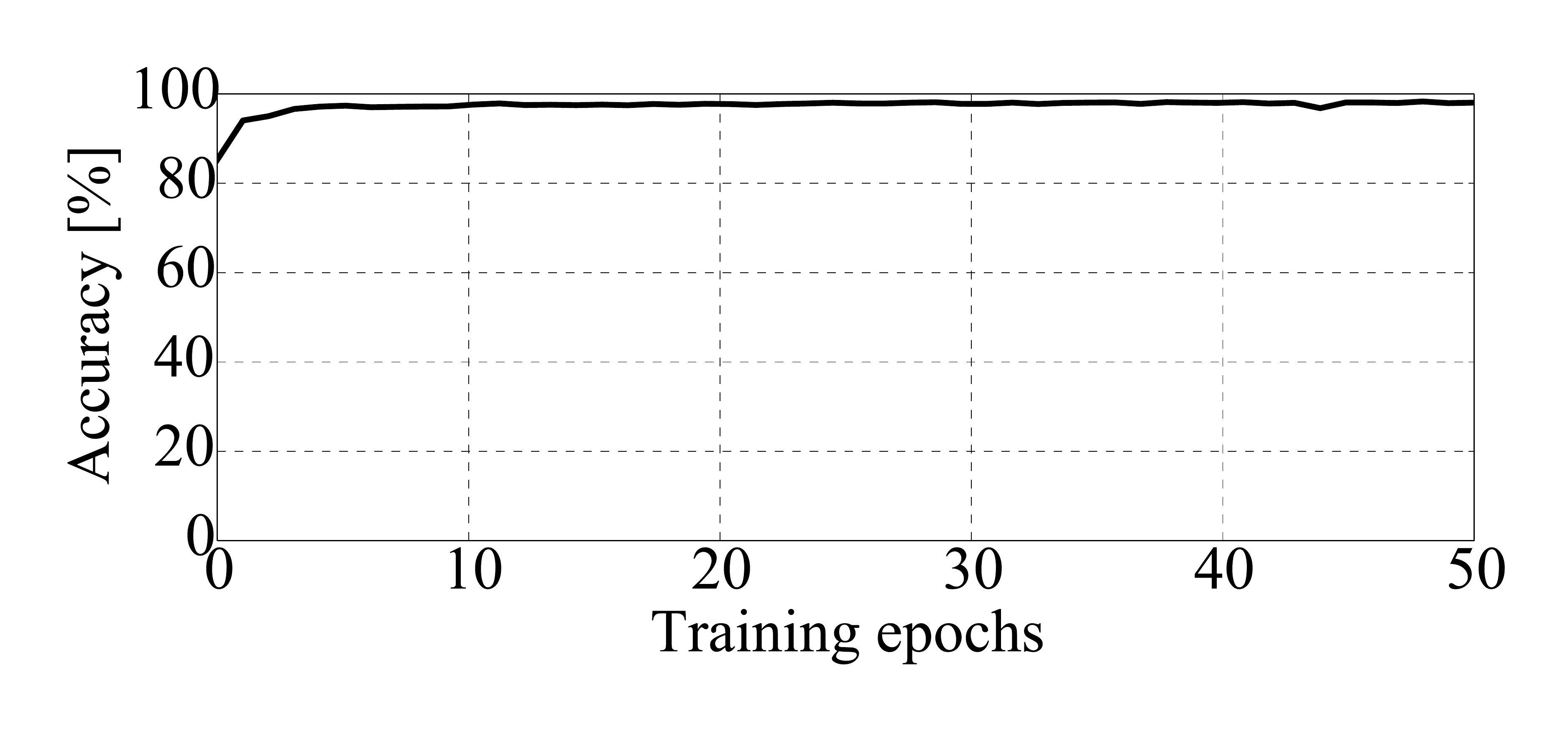}}
	\subfigure[Permuted pixel MNIST.]{\label{pixelPermutedMNIST}\includegraphics[width=0.24\textwidth]{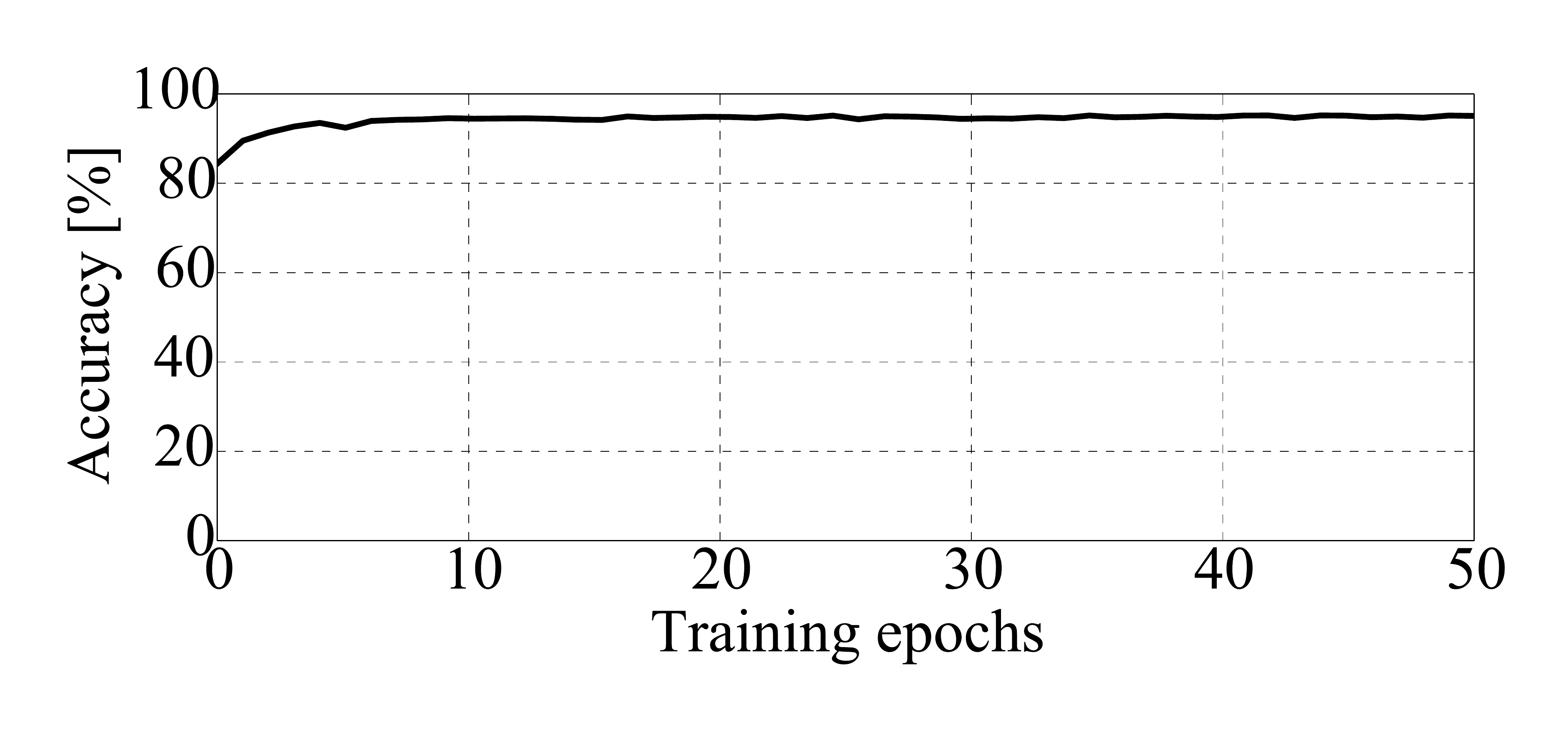}}			
	\caption{Reweighted-RNN on the (a) pixel-MNIST classification and (b) permuted pixel MNIST classification: Average classification accuracy vs. training epoches on the validation set.}\label{MNISTclassification}
\end{figure} 
\begin{figure}[h!]
	\centering 
	\subfigure[Adding task.]{\label{addingTask}\includegraphics[width=0.24\textwidth]{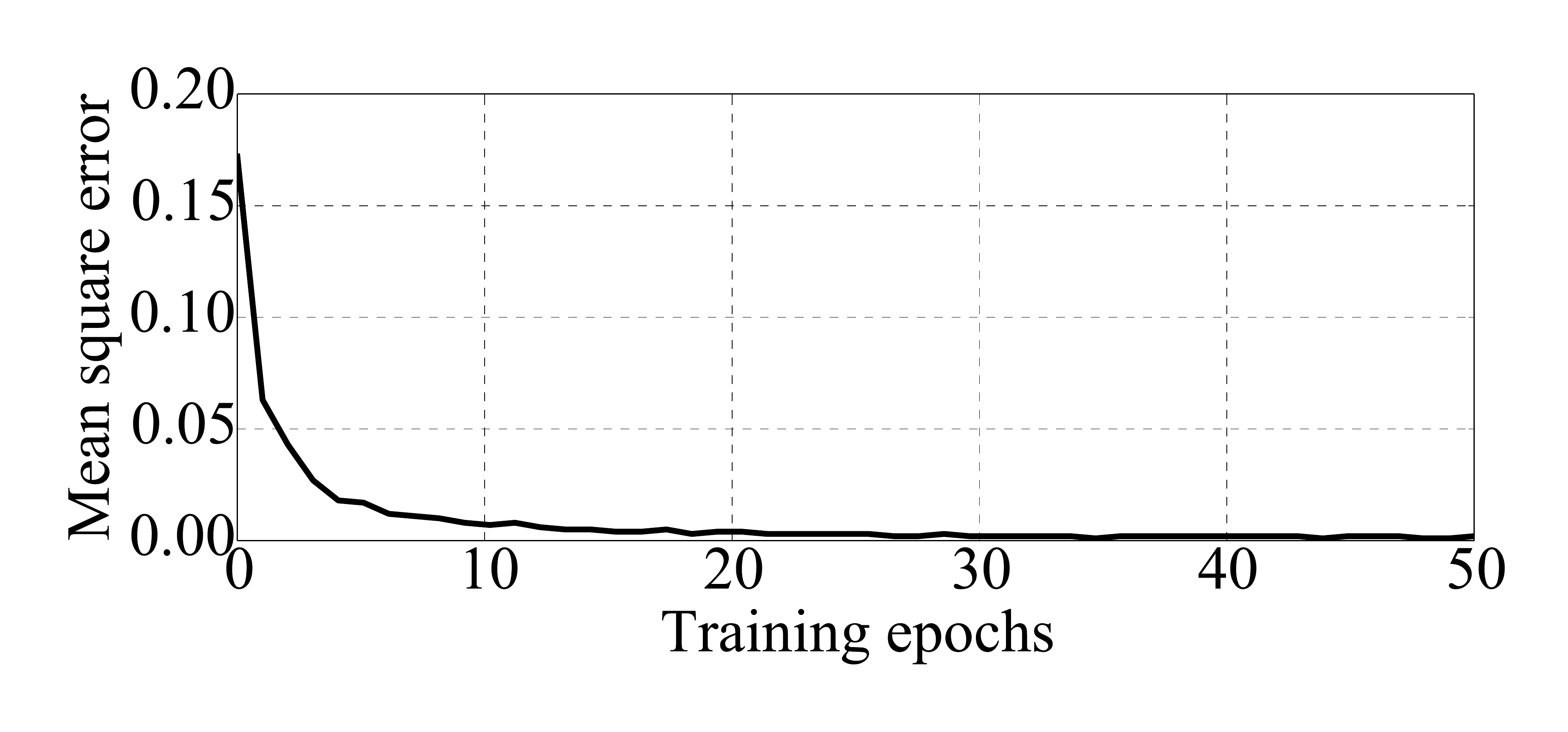}}
	\subfigure[Copy task.]{\label{copyMemory}\includegraphics[width=0.24\textwidth]{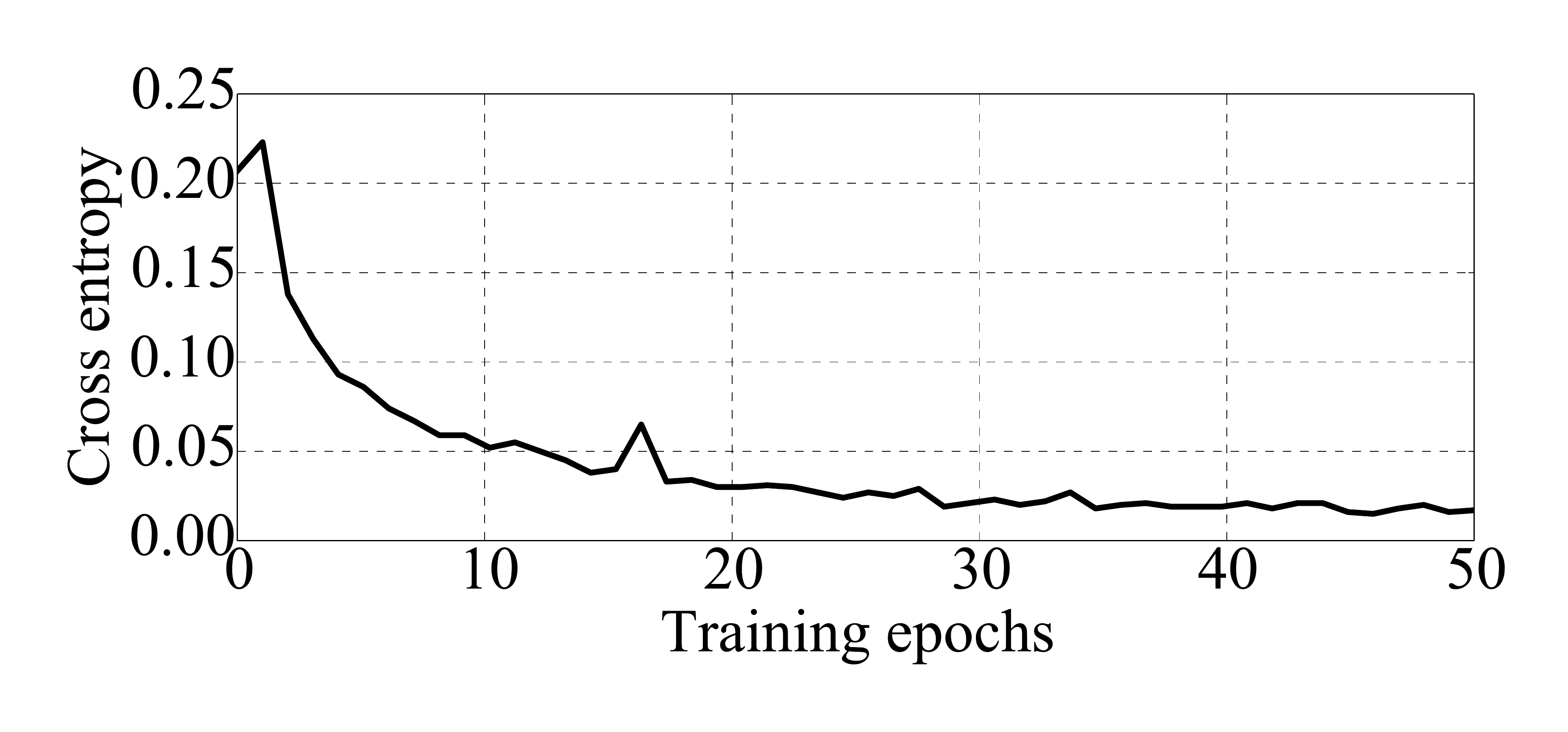}}			
	\caption{Reweighted-RNN on the (a) adding task with average mean square error and the (b) copy task with average cross entropy vs. training epoches on the validation set.}\label{adding and copy tasks}
\end{figure} 

\textbf{Adding Task}. The task inputs two sequences of length $T$. The first sequence consists of entries that are uniformly sampled from $[0,~1]$. The second sequence comprises two entries of 1 and the remaining entries of 0, in which the first entry of 1 is randomly located in the first half of the sequence and the second entry of 1 is randomly located in the second half. The output is the sum of the two entrie of the first sequence, where are located in the same posisions of the entries of 1 in the second sequence. We also use the reweighted-RNN with $d=5$ layers and $h = 256$ hidden units for the input sequences of length $T=300$. Fig.~\ref{addingTask} shows the mean square error versus training epoches on the validation set.

\textbf{Copy task}. We consider an input sequence $\bx \in \mathbb{A}^{T+20}$~\cite{ZhangICML18}, where $\mathbb{A}=\{a_0,\cdots,a_9\}$. $x_0,\cdots,x_9$ are uniformly sampled from $\{a_0,\cdots,a_7\}$, $x_{T+10}=a_9$, and the remaining $x_i$ are set to $a_8$. The purpose of this task is to copy $x_0,\cdots,x_9$ to the end of the output sequence $\by \in \mathbb{A}^{T+20}$ given a time lag $T$, i.e., $\{y_{T+10},\cdots,y_{T+19}\}\equiv\{x_0,\cdots,x_9\}$ and the remaining $y_i$ are equal to $a_8$. We set the reweighted-RNN $d=5$ layers and $h = 256$ hidden units for the input sequences of a time lag $T=100$. Fig.~\ref{copyMemory} shows the cross entropy versus training epoches on the validation set.
	
	\section{Conclusions}
	\vspace{-0pt}
	We designed a novel deep RNN by unfolding an algorithm that solves a reweighted $\ell_1$-$\ell_1$ minimization problem. The proposed reweighted-RNN model has high network expressivity due to \textit{per-unit learnable activation functions} and \textit{over-parameterized weights}. We also established the generalization error bound for the proposed model via Rademacher complexity. We showed that reweighted-RNN has good generalization properties and its error bound is tighter than existing ones concerning the number of time steps. Experimentation on the task of sequential video frame reconstruction suggests that our model (\textit{i}) outperforms various state-of-the-art RNNs in terms of accuracy and convergence speed. (\textit{ii}) is capable of stacking many hidden layers resulting in a better learning capability.


%

%
%

		\bibliographystyle{IEEEtran}
\bibliography{./IEEEfull,./IEEEabrv,./bibliography}		

\begin{thebibliography}{10}
\providecommand{\url}[1]{#1}
\csname url@samestyle\endcsname
\providecommand{\newblock}{\relax}
\providecommand{\bibinfo}[2]{#2}
\providecommand{\BIBentrySTDinterwordspacing}{\spaceskip=0pt\relax}
\providecommand{\BIBentryALTinterwordstretchfactor}{4}
\providecommand{\BIBentryALTinterwordspacing}{\spaceskip=\fontdimen2\font plus
\BIBentryALTinterwordstretchfactor\fontdimen3\font minus
  \fontdimen4\font\relax}
\providecommand{\BIBforeignlanguage}[2]{{%
\expandafter\ifx\csname l@#1\endcsname\relax
\typeout{** WARNING: IEEEtran.bst: No hyphenation pattern has been}%
\typeout{** loaded for the language `#1'. Using the pattern for}%
\typeout{** the default language instead.}%
\else
\language=\csname l@#1\endcsname
\fi
#2}}
\providecommand{\BIBdecl}{\relax}
\BIBdecl

\bibitem{DonohoTIT06}
D.~Donoho, ``Compressed sensing,'' \emph{IEEE Transactions on Information
  Theory}, vol.~52, no.~4, pp. 1289--1306, Apr. 2006.

\bibitem{daubechies2004iterative}
I.~Daubechies, M.~Defrise, and C.~D. Mol, ``An iterative thresholding algorithm
  for linear inverse problems with a sparsity constraint,''
  \emph{Communications on Pure and Applied Mathematics}, vol.~57, no.~11, pp.
  1413--1457, 2004.

\bibitem{WisdomICASSP17}
S.~Wisdom, T.~Powers, J.~Pitton, and L.~Atlas, ``Building recurrent networks by
  unfolding iterative thresholding for sequential sparse recovery,'' in
  \emph{2017 IEEE International Conference on Acoustics, Speech and Signal
  Processing (ICASSP)}, March 2017.

\bibitem{LeArXiv19}
H.~D. Le, H.~V. Luong, and N.~Deligiannis, ``Designing recurrent neural
  networks by unfolding an l1-l1 minimization algorithm,'' in \emph{Proceedings
  of IEEE International Conference on Image Processing}, 2019.

\bibitem{MotaTIT17}
J.~F.~C. Mota, N.~Deligiannis, and M.~R.~D. Rodrigues, ``Compressed sensing
  with prior information: Strategies, geometry, and bounds,'' \emph{IEEE
  Transactions on Information Theory}, vol.~63, no.~7, pp. 4472--4496, July
  2017.

\bibitem{MotaTSP17}
J.~F.~C. Mota, N.~Deligiannis, A.~C. Sankaranarayanan, V.~Cevher, and M.~R.~D.
  Rodrigues, ``Adaptive-rate reconstruction of time-varying signals with
  application in compressive foreground extraction,'' \emph{IEEE Transactions
  on Signal Processing}, vol.~64, no.~14, pp. 3651--3666, July 2016.

\bibitem{mousavi2015deep}
A.~Mousavi, A.~B. Patel, and R.~G. Baraniuk, ``A deep learning approach to
  structured signal recovery,'' in \emph{2015 53rd Annual Allerton Conference
  on Communication, Control, and Computing (Allerton)}.\hskip 1em plus 0.5em
  minus 0.4em\relax IEEE, 2015, pp. 1336--1343.

\bibitem{lucas2018using}
A.~Lucas, M.~Iliadis, R.~Molina, and A.~K. Katsaggelos, ``Using deep neural
  networks for inverse problems in imaging: beyond analytical methods,''
  \emph{IEEE Signal Processing Magazine}, vol.~35, no.~1, pp. 20--36, 2018.

\bibitem{GregorICML10}
K.~Gregor and Y.~LeCun, ``Learning fast approximations of sparse coding,'' in
  \emph{Proceedings of the 27th International Conference on International
  Conference on Machine Learning}, ser. ICML'10, 2010, pp. 399--406.

\bibitem{sun2016deep}
J.~Sun, H.~Li, and Z.~Xu, ``Deep {ADMM}-{N}et for compressive sensing {MRI},''
  in \emph{Advances in Neural Information Processing Systems}, 2016, pp.
  10--18.

\bibitem{borgerding2017amp}
M.~Borgerding, P.~Schniter, and S.~Rangan, ``{AMP}-inspired deep networks for
  sparse linear inverse problems,'' \emph{IEEE Transactions on Signal
  Processing}, vol.~65, no.~16, pp. 4293--4308, 2017.

\bibitem{xin2016maximal}
B.~Xin, Y.~Wang, W.~Gao, D.~Wipf, and B.~Wang, ``Maximal sparsity with deep
  networks?'' in \emph{Advances in Neural Information Processing Systems},
  2016, pp. 4340--4348.

\bibitem{SpechmannPAMI15}
P.~{Sprechmann}, A.~M. {Bronstein}, and G.~{Sapiro}, ``Learning efficient
  sparse and low rank models,'' \emph{IEEE Transactions on Pattern Analysis and
  Machine Intelligence}, vol.~37, no.~9, pp. 1821--1833, Sep. 2015.

\bibitem{ChenNIPS18}
X.~Chen, J.~Liu, Z.~Wang, and W.~Yin, ``Theoretical linear convergence of
  unfolded ista and its practical weights and thresholds,'' in \emph{Advances
  in Neural Information Processing Systems 31}, 2018.

\bibitem{LiuICLR19}
J.~Liu, X.~Chen, Z.~Wang, and W.~Yin, ``{ALISTA}: Analytic weights are as good
  as learned weights in {LISTA},'' in \emph{International Conference on
  Learning Representations}, 2019.

\bibitem{Candes08}
E.~J. Cand{\`e}s, M.~B. Wakin, and S.~P. Boyd, ``Enhancing sparsity by
  reweighted $\ell_1$ minimization,'' \emph{Journal of Fourier Analysis and
  Applications}, vol.~14, no.~5, pp. 877--905, 2008.

\bibitem{LuongTIP18}
H.~V. Luong, N.~Deligiannis, J.~Seiler, S.~Forchhammer, and A.~Kaup,
  ``Compressive online robust principal component analysis via $n$-$\ell_1$
  minimization,'' \emph{IEEE Transactions on Image Processing}, vol.~27, no.~9,
  pp. 4314--4329, Sept 2018.

\bibitem{HeCVPR16}
K.~He, X.~Zhang, S.~Ren, and J.~Sun, ``Deep residual learning for image
  recognition,'' in \emph{2016 IEEE Conference on Computer Vision and Pattern
  Recognition (CVPR)}, June 2016, pp. 770--778.

\bibitem{CortesPCML17}
C.~Cortes, X.~Gonzalvo, V.~Kuznetsov, M.~Mohri, and S.~Yang, ``{A}da{N}et:
  Adaptive structural learning of artificial neural networks,'' in
  \emph{Proceedings of the 34th International Conference on Machine Learning},
  Sydney, Australia, 06--11 Aug.

\bibitem{HuangCVPR17}
G.~Huang, Z.~Liu, L.~v.~d. Maaten, and K.~Q. Weinberger, ``Densely connected
  convolutional networks,'' in \emph{2017 IEEE Conference on Computer Vision
  and Pattern Recognition (CVPR)}, July 2017.

\bibitem{NeyshaburICLR19}
B.~Neyshabur, Z.~Li, S.~Bhojanapalli, Y.~LeCun, and N.~Srebro, ``The role of
  over-parametrization in generalization of neural networks,'' in \emph{Int.
  Conf. on Learning Representations, ICLR 2019}, 2019.

\bibitem{PascanuICLR14}
R.~Pascanu, C.~Gulcehre, K.~Cho, and Y.~Bengio, ``\BIBforeignlanguage{English
  (US)}{How to construct deep recurrent neural networks},'' in
  \emph{\BIBforeignlanguage{English (US)}{International Conference on Learning
  Representations (ICLR)}}, 2014.

\bibitem{LiCVPR18}
S.~{Li}, W.~{Li}, C.~{Cook}, C.~{Zhu}, and Y.~{Gao}, ``Independently recurrent
  neural network (indrnn): Building a longer and deeper rnn,'' in \emph{2018
  IEEE/CVF Conference on Computer Vision and Pattern Recognition}, June 2018,
  pp. 5457--5466.

\bibitem{LuoICCV18}
W.~Luo, W.~Liu, and S.~Gao, ``A revisit of sparse coding based anomaly
  detection in stacked rnn framework,'' in \emph{2017 IEEE International
  Conference on Computer Vision (ICCV)}, Oct 2017, pp. 341--349.

\bibitem{ZhangICML18}
J.~Zhang, Q.~Lei, and I.~S. Dhillon, ``Stabilizing gradients for deep neural
  networks via efficient {SVD} parameterization,'' in \emph{Proceedings of the
  35th International Conference on Machine Learning, {ICML} 2018}, 2018.

\bibitem{KusupatiNIPS18}
A.~Kusupati, M.~Singh, K.~Bhatia, A.~Kumar, P.~Jain, and M.~Varma, ``Fastgrnn:
  A fast, accurate, stable and tiny kilobyte sized gated recurrent neural
  network,'' in \emph{Advances in Neural Information Processing Systems 31},
  2018.

\bibitem{Beck09}
A.~Beck and M.~Teboulle, ``A fast iterative shrinkage-thresholding algorithm
  for linear inverse problems,'' \emph{SIAM Journal on Imaging Sciences}, vol.
  2(1), pp. 183--202, 2009.

\bibitem{ShwartzBook14}
S.~Shalev-Shwartz and S.~Ben-David, \emph{Understanding Machine Learning: From
  Theory to Algorithms}.\hskip 1em plus 0.5em minus 0.4em\relax New York, NY,
  USA: Cambridge University Press, 2014.

\bibitem{SrivastavaICML15}
N.~Srivastava, E.~Mansimov, and R.~Salakhudinov, ``Unsupervised learning of
  video representations using {LSTMs},'' in \emph{Proceedings of the 32nd
  International Conference on Machine Learning (ICML)}, 2015.

\bibitem{ELMAN90}
J.~L. Elman, ``Finding structure in time,'' \emph{Cognitive Science}, vol.~14,
  no.~2, pp. 179 -- 211, 1990.

\bibitem{hochreiter1997long}
S.~Hochreiter and J.~Schmidhuber, ``Long short-term memory,'' \emph{Neural
  Computation}, vol.~9, no.~8, pp. 1735--1780, 1997.

\bibitem{cho2014learning}
K.~Cho, B.~V. Merri{\"e}nboer, C.~Gulcehre, D.~Bahdanau, F.~Bougares,
  H.~Schwenk, and Y.~Bengio, ``Learning phrase representations using rnn
  encoder-decoder for statistical machine translation,'' \emph{arXiv preprint
  arXiv:1406.1078}, 2014.

\bibitem{pascanu2013construct}
R.~Pascanu, C.~Gulcehre, K.~Cho, and Y.~Bengio, ``How to construct deep
  recurrent neural networks,'' \emph{arXiv preprint arXiv:1312.6026}, 2013.

\bibitem{LeArXiv15}
Q.~V. Le, N.~Jaitly, and G.~E. Hinton, ``A simple way to initialize recurrent
  networks of rectified linear units,'' \emph{CoRR}, vol. abs/1504.00941, 2015.

\bibitem{ArjovskyICML16}
\BIBentryALTinterwordspacing
M.~Arjovsky, A.~Shah, and Y.~Bengio, ``Unitary evolution recurrent neural
  networks,'' in \emph{Proceedings of the 33rd International Conference on
  International Conference on Machine Learning - Volume 48}, ser.
  ICML'16.\hskip 1em plus 0.5em minus 0.4em\relax JMLR.org, 2016, pp.
  1120--1128. [Online]. Available:
  \url{http://dl.acm.org/citation.cfm?id=3045390.3045509}
\BIBentrySTDinterwordspacing

\bibitem{MohriBook18}
M.~Mohri, A.~Rostamizadeh, and A.~Talwalkar, \emph{Foundations of Machine
  Learning, Second edition}.\hskip 1em plus 0.5em minus 0.4em\relax Cambridge,
  Massachusetts, USA: MIT Press, 2018.

\end{thebibliography}
\ifCLASSOPTIONcaptionsoff
  \newpage
\fi

		\appendices
		\section{Proximal operator for the reweighted $\ell_1$-$\ell_1$ minimization problem}\label{solvingReweightedl1-l1Minimization}
		
		\begin{proposition}\label{propReweighted-l1-l1}
			The proximal operator $\varPhi_{\frac{\lambda_1}{c}\g,\frac{\lambda_2}{c}\g,\hbar}(u)$ in \eqref{reweighted-l1-proximalOperator} for the reweighted $\ell_1$-$\ell_1$ minimization problem~\eqref{reweighted-l1-l1minimization}, for which $g(v) = \lambda_1\g| v| + \lambda_2\g|v - \hbar|$, is given by
			\begin{align}
			&\varPhi_{\frac{\lambda_1}{c}\g,\frac{\lambda_2}{c}\g,\hbar\geq 0}(u)\nonumber
			\\
			&=\begin{cases}
			u - \frac{\lambda_1}{c}\g - \frac{\lambda_2}{c}\g, & \hbar + \frac{\lambda_1}{c}\g+ \frac{\lambda_2}{c}\g < u < \infty \\
			\hbar, & \hbar + \frac{\lambda_1}{c}\g - \frac{\lambda_2}{c}\g \leq u \leq \hbar + \frac{\lambda_1}{c}\g + \frac{\lambda_2}{c}\g \\
			u - \frac{\lambda_1}{c}\g + \frac{\lambda_2}{c}\g, & \frac{\lambda_1}{c}\g - \frac{\lambda_2}{c}\g <u < \hbar + \frac{\lambda_1}{c}\g
			- \frac{\lambda_2}{c}\g\\
			0, & -\frac{\lambda_1}{c}\g- \frac{\lambda_2}{c}\g\leq u \leq \frac{\lambda_1}{c}\g- \frac{\lambda_2}{c}\g\\
			u + \frac{\lambda_1}{c}\g + \frac{\lambda_2}{c}\g, & -\infty < u < -\frac{\lambda_1}{c}\g - \frac{\lambda_2}{c}\g\\
			\end{cases}\label{reweighted-l1-proximalOperatorElementCompute1}\\
			&\varPhi_{\frac{\lambda_1}{c}\g,\frac{\lambda_2}{c}\g,\hbar<0}(u)
			\nonumber
			\\
			&=\begin{cases}
			u - \frac{\lambda_1}{c}\g - \frac{\lambda_2}{c}\g, & \frac{\lambda_1}{c}\g+ \frac{\lambda_2}{c}\g < u < \infty \\
			0, & - \frac{\lambda_1}{c}\g +\frac{\lambda_2}{c}\g \leq u \leq \frac{\lambda_1}{c}\g + \frac{\lambda_2}{c}\g \\
			u + \frac{\lambda_1}{c}\g - \frac{\lambda_2}{c}\g , & \hbar - \frac{\lambda_1}{c}\g +\frac{\lambda_2}{c}\g < u < - \frac{\lambda_1}{c}\g
			+ \frac{\lambda_2}{c}\g\\
			\hbar, & \hbar-\frac{\lambda_1}{c}\g- \frac{\lambda_2}{c}\g\leq u \leq \hbar - \frac{\lambda_1}{c}\g + \frac{\lambda_2}{c}\g\\
			u- \frac{\lambda_1}{c}\g + \frac{\lambda_2}{c}\g, & -\infty < u <\hbar -\frac{\lambda_1}{c}\g - \frac{\lambda_2}{c}\g\label{reweighted-l1-proximalOperatorElementCompute2}\\
			\end{cases}
			\end{align}
		\end{proposition}
		\begin{proof}
			We compute the proximal operator $\varPhi_{\frac{\lambda_1}{c}\g,\frac{\lambda_2}{c}\g,\hbar}(u)$ \eqref{reweighted-l1-proximalOperatorElementCompute1} for $\hbar\geq 0$, it is similar for $\hbar<0$. From \eqref{reweighted-l1-proximalOperator}, $\varPhi_{\frac{\lambda_1}{c}\g,\frac{\lambda_2}{c}\g,\hbar}(u)$ is expressed by:
			\begin{align}\label{reweighted-l1-proximalOperatorCompute}
			&\varPhi_{\frac{\lambda_1}{c}\g,\frac{\lambda_2}{c}\g,\hbar}(u)
			\nonumber
			\\
			& = \argmin_{v \in\mathbb{R}}\Big\{\varphi(v):=\frac{\lambda_1}{c}\g| v| + \frac{\lambda_2}{c}\g|v - \hbar|+\frac{1}{2}|v-u|^{2}\Big\}.
			\end{align}
			
			We consider the $\partial \varphi(v)/\partial v$. When $v$ is located in one of the intervals, $v\in\{(-\infty,0),(0,\hbar),(\hbar,\infty)\}$, where $\partial \varphi(v)$ exists. Taking the derivative of $\varphi(v)$ in these intervals delivers
			\vspace{-0.8pt}
			\begin{equation}\label{reweighted-l1-proximalOperatorElementDerivative}
			\frac{\partial \varphi(v)}{\partial v}=\frac{\lambda_1}{c}\g\cdot \mathrm{sign}(v)+ \frac{\lambda_2}{c}\g\cdot \mathrm{sign}(v-\hbar) + (v-u),
			\end{equation}
			where $\mathrm{sign}(.)$ is a sign function. When setting $\partial \varphi(v)/\partial v= 0$ to minimize $\varphi(v)$, we derive:
			\begin{align}\label{partialZero}
			v=
			\begin{cases}
			u - \frac{\lambda_1}{c}\g - \frac{\lambda_2}{c}\g, & \hbar < v < \infty \\
			u - \frac{\lambda_1}{c}\g + \frac{\lambda_2}{c}\g, & 0< v < \hbar \\
			u + \frac{\lambda_1}{c}\g + \frac{\lambda_2}{c}\g, & -\infty < v <0\\
			\end{cases}
			\end{align}
			
			From \eqref{reweighted-l1-proximalOperatorCompute} and \eqref{partialZero}, we have
			\begin{align}\label{reweighted-l1-proximalX}
			&\varPhi_{\frac{\lambda_1}{c}\g,\frac{\lambda_2}{c}\g,\hbar}(u)
			\nonumber
			\\
			&=
			\begin{cases}
			u - \frac{\lambda_1}{c}\g - \frac{\lambda_2}{c}\g, & \hbar + \frac{\lambda_1}{c}\g+ \frac{\lambda_2}{c}\g < u < \infty \\
			u - \frac{\lambda_1}{c}\g + \frac{\lambda_2}{c}\g, & \frac{\lambda_1}{c}\g - \frac{\lambda_2}{c}\g < u < \hbar + \frac{\lambda_1}{c}\g
			- \frac{\lambda_2}{c}\g\\
			u + \frac{\lambda_1}{c}\g + \frac{\lambda_2}{c}\g, & -\infty < u < -\frac{\lambda_1}{c}\g - \frac{\lambda_2}{c}\g\\
			\end{cases}
			\end{align}
			
			In the remaining
			range value of $u$, namely, $-\frac{\lambda_1}{c}\g - \frac{\lambda_2}{c}\g\leq u \leq \frac{\lambda_1}{c}\g - \frac{\lambda_2}{c}\g$ and $\hbar + \frac{\lambda_1}{c}\g- \frac{\lambda_2}{c}\g \leq u \leq \hbar + \frac{\lambda_1}{c}\g+ \frac{\lambda_2}{c}\g$, we prove that the minimum of $\varphi(v)$ \eqref{reweighted-l1-proximalOperatorCompute} is obtained when $v=0$ and $v=\hbar$, respectively.
			
			Let us recall $\varphi(v)$ in \eqref{reweighted-l1-proximalOperatorCompute} as
			\begin{equation}\label{reweighted-l1-proximalOperatorElementLemma}
			\varphi(v)=\frac{\lambda_1}{c}\g| v| + \frac{\lambda_2}{c}\g|v - \hbar|+\frac{1}{2}|v-u|^{2}
			\end{equation}
			Applying the inequality $|a-b|\geq |a|-|b|$, where $a,b\in \mathbb{R}$, to \eqref{reweighted-l1-proximalOperatorElementLemma}, we obtain:
			\begin{align}\label{reweighted-l1-proximalOperatorElementLemmaInequality}
			\varphi(v)&\geq\frac{\lambda_1}{c}\g| v|+\frac{\lambda_2}{c}\g|v| - \frac{\lambda_2}{c}\g|\hbar|+\frac{1}{2}v^{2}-vu+\frac{1}{2}u^{2}\nonumber\\
			&\geq |v|\Big(\frac{\lambda_1}{c}\g+\frac{\lambda_2}{c}\g -|u| \Big)+\frac{1}{2}v^{2} -\frac{\lambda_2}{c}\g|\hbar|+\frac{1}{2}u^{2}
			\end{align}
			For $-\frac{\lambda_1}{c}\g - \frac{\lambda_2}{c}\g\leq u \leq \frac{\lambda_1}{c}\g - \frac{\lambda_2}{c}\g$, from \eqref{reweighted-l1-proximalOperatorElementLemmaInequality}, $\varphi(v)$ is minimal when $v=0$, due to $\frac{\lambda_1}{c}\g+\frac{\lambda_2}{c}\g -|u|\geq 0$.
			
			Similarly, for $\hbar + \frac{\lambda_1}{c}\g- \frac{\lambda_2}{c}\g \leq u \leq \hbar + \frac{\lambda_1}{c}\g+ \frac{\lambda_2}{c}\g$, i.e., $\frac{\lambda_1}{c}\g- \frac{\lambda_2}{c}\g \leq u-\hbar \leq \frac{\lambda_1}{c}\g+ \frac{\lambda_2}{c}\g$, we have
			\begin{align}\label{reweighted-l1-proximalOperatorElementLemmaInequality2}
			\varphi(v)\geq&\frac{\lambda_1}{c}\g| v-\hbar|-\frac{\lambda_1}{c}\g|\hbar| + \frac{\lambda_2}{c}\g|v-\hbar|+\frac{1}{2}(v-\hbar)^{2}
			\nonumber
			\\
			&-|v-\hbar||u-\hbar|+\frac{1}{2}(u-\hbar)^{2}\nonumber\\
			\geq &|v-\hbar|\Big(\frac{\lambda_1}{c}\g+\frac{\lambda_2}{c}\g -|u-\hbar| \Big)+\frac{1}{2}(v-\hbar)^{2} 
			\nonumber
			\\
			&-\frac{\lambda_1}{c}\g|\hbar|+\frac{1}{2}(u-\hbar)^{2}.
			\end{align}			
			From \eqref{reweighted-l1-proximalOperatorElementLemmaInequality2}, $\varphi(v)$ is minimal when $v=\hbar$, since $\frac{\lambda_1}{c}\g+\frac{\lambda_2}{c}\g -|u-\hbar|\geq 0$. As the results, we conclude the proof.
		\end{proof}

		
		
		

		\section{Deep unfolding RNNs}\label{deepUnfoldRNNs}
		We define the proposed reweighted-RNN, $\ell_1$-$\ell_1$-RNN, and Sista-RNN in more details as follows:
		
		\textbf{The proposed reweighted-RNN}. Let $\bh_t^{(l)}$ be the hidden states in layer $l$ evolving in time step $t$. We write reweighted-RNN recursively as $\bh_t^{(1)}=f_{\bcW,\bcU}^{(1)}(\bh_{t-1}^{(d)},\bx_t)=\varPhi(\mathbf{W}_{1}\bh_{t-1}^{(d)}+\mathbf{U}_1\bx_t)$ and $\bh_t^{(l)}=f_{\bcW,\bcU}^{(l)}(\bh_{t-1}^{(d)},\bx_t)=\varPhi\Big(\bW_l f_{\bcW,\bcU}^{(l-1)}(\bh_{t-1}^{(d)},\bx_t)+\mathbf{U}_l\bx_t\Big)$, where $\varPhi$ is an activation function. The hidden state is updated as depicted in \eqref{reweighted-l1-l1-RNN}.
		The real-valued family of functions, $\mathcal{F}_{d,t}:\mathbb{R}^h\times \mathbb{R}^n\mapsto \mathbb{R}$, for the functions $f_{\bcW,\bU}^{(d)}$ in layer $d$ is defined by: 
		\begin{align}\label{deepreweighted-RNNFamily}
		\mathcal{F}_{d,t}=\Big\{& (\bh_{t-1}^{(d)},\bx_t) \mapsto \varPhi(\bw_{d}^{\mathrm{T}} f_{\bcW,\bcU}^{(d-1)}(\bh_{t-1}^{(d)},\bx_t)+\bu_d^{\mathrm{T}}\bx_t):
		\nonumber
		\\
		&\|\bW_d\|_{1,\infty}\leq \alpha_d ,~\|\bU_d\|_{1,\infty}\leq \beta_d \Big\},
		\end{align}
		where $\alpha_{l},\beta_l$ are nonnegative hyper-parameters for layer $l$, where~ $1< l\leq d$. In layer $l=1$,	the real-valued family of functions, $\mathcal{F}_{1,t}:\mathbb{R}^h\times \mathbb{R}^n\mapsto \mathbb{R}$, for the functions $f_{\bcW,\bcU}^{(1)}$ is defined by: 
		\begin{align}\label{deepreweighted-RNNFamilyLayer1}
		\mathcal{F}_{1,t}=\Big\{ &(\bh_{t-1}^{(d)},\bx_t) \mapsto \varPhi(\bw_{1}^{\mathrm{T}}\bh_{t-1}^{(d)}+\bu_1^{\mathrm{T}}\bx_t):
		\nonumber
		\\
		& \|\bW_1\|_{1,\infty}\leq \alpha_1 ,~\|\bU\|_{1,\infty}\leq \beta_1 \Big\},
		\end{align}
		where $\alpha_{1},\beta_1$ are nonnegative hyper-parameters. We denote the input layer as $f_{\bcW,\bcU}^{(0)}=\bh_{t-1}^{(d)}$, in particular, at $t=1$, $\bh_0^{(l)}\equiv \bh_0$.
		
		\textbf{$\ell_1$-$\ell_1$-RNN}. The hidden state $\bh_t^{(l)}$ is updated as shown in \eqref{l1-l1-RNN}. The real-valued family of functions, $\mathcal{F}_{d,t}:\mathbb{R}^h\times \mathbb{R}^n\mapsto \mathbb{R}$, for the function $f_{\bcW,\bU}^{(d)}$ in layer $d$ is defined by: 
		\begin{align}\label{deepl1-l1RNNFamily}
		\mathcal{F}_{d,t}=\Big\{& (\bh_{t-1}^{(d)},\bx_t) \mapsto \varPhi(\bw_{2}^{\mathrm{T}} f_{\bcW,\bcU}^{(d-1)}(\bh_{t-1}^{(d)},\bx_t)+\bu_1^{\mathrm{T}}\bx_t): 
		\nonumber
		\\
		&\|\bW_2\|_{1,\infty}\leq \alpha_2 ,~\|\bU_1\|_{1,\infty}\leq \beta_1 \Big\},
		\end{align}
		where $\alpha_{2},\beta_1$ are nonnegative hyper-parameters for layer $l$, where~ $1< l\leq d$. In layer $l=1$,	the real-valued family of functions, $\mathcal{F}_{1,t}:\mathbb{R}^h\times \mathbb{R}^n\mapsto \mathbb{R}$, for the functions $f_{\bcW,\bcU}^{(1)}$ is defined by: 
		\begin{align}\label{deepl1-l1RNNFamilyLayer1}
		\mathcal{F}_{1,t}=\Big\{& (\bh_{t-1}^{(d)},\bx_t) \mapsto \varPhi(\bw_{1}^{\mathrm{T}}\bh_{t-1}^{(d)}+\bu_1^{\mathrm{T}}\bx_t):
		\nonumber
		\\
		& \|\bW_1\|_{1,\infty}\leq \alpha_1 ,~\|\bU\|_{1,\infty}\leq \beta_1 \Big\},
		\end{align}
		where $\alpha_{1},\beta_1$ are nonnegative hyper-parameters.
		
		\textbf{Sista-RNN}.	The hidden state $\bh_t^{(l)}$ is updated by:
		\begin{equation}\label{deepl1-l2RNN}
		\bh_t^{(l)}\hspace{-2pt}=\hspace{-2pt}\left\{
		\begin{array}{l}
		\phi\Big(\mathbf{W}_{1}\bh_{t-1}^{(d)}+\mathbf{U}_1\bx_t\Big),~~~~~~~~~~~~~~~~~~~~~~l=1,\\
		\phi\Big(\bW_2\bh_{t}^{(l-1)}+\mathbf{U}_1\bx_t+\bU_2\bh_{t-1}^{(d)}\Big),~l>1,\\
		\end{array}
		\right.
		\end{equation}
		
		The real-valued family of functions, $\mathcal{F}_{d,t}:\mathbb{R}^h\times \mathbb{R}^n\mapsto \mathbb{R}$, for the functions $f_{\bcW,\bU}^{(d)}$ in layer $d$ is defined by: 
		\begin{align}\label{deepl1-l2RNNFamily}
		\mathcal{F}_{d,t}&
		\nonumber
		\\
		=\Big\{& (\bh_{t-1}^{(d)},\bx_t) \mapsto \phi\Big(\bw_{2}^{\mathrm{T}} f_{\bcW,\bcU}^{(d-1)}(\bh_{t-1}^{(d)},\bx_t)+\bu_1^{\mathrm{T}}\bx_t+\bu_2^{\mathrm{T}}\bh_{t-1}^{(d)}\Big) \nonumber\\
		&~:\|\bW_2\|_{1,\infty}\leq \alpha_2,\|\bU\|_{1,\infty}\leq \beta_1,~\|\bU\|_{2,\infty}\leq \beta_2 \Big\},
		\end{align}
		where $\alpha_{2}, \beta_1,\beta_2$ are nonnegative hyper-parameters. In layer $l=1$,	
		\begin{align}\label{deepl1-l2RNNFamilyLayer1}
		\mathcal{F}_{1,t}=\Big\{& (\bh_{t-1}^{(d)},\bx_t) \mapsto \phi\Big(\bw_{1}^{\mathrm{T}} \bh_{t-1}^{(d)}+\bu_1^{\mathrm{T}}\bx_t\Big): 
		\nonumber
		\\
		&\|\bW_1\|_{1,\infty}\leq \alpha_1 ,~\|\bU\|_{1,\infty}\leq \beta_1 \Big\},
		\end{align}
		where $\alpha_{1}, \beta_1$ are nonnegative hyper-parameters.

	\section{Supports for Rademacher complexity calculus}
	The contraction lemma in \cite{ShwartzBook14} shows the Rademacher complexity of the composition of a class of functions with $\rho$-Lipschitz functions.
	\begin{lemma}\label{ContractionLemma}\cite[Lemma 26.9---Contraction lemma]{ShwartzBook14}\\
		Let $\mathcal{F}$ be a set of functions, $\mathcal{F}=\{f:\mathcal{X}\mapsto \mathbb{R}\}$, and $\varPhi_1,...,\varPhi_m$, $\rho$-Lipschitz functions, namely, $|\varPhi_i(\alpha)-\varPhi_i(\beta)|\leq \rho|\alpha - \beta|$ for all $\alpha,\beta \in \mathbb{R}$ for some $\rho>0$. For any sample set $S$ of $m$ points $\bx_1,...,\bx_m \in \mathcal{X}$, let $(\varPhi\boldsymbol{\circ} f)(\bx_i)=\varPhi(f(\bx_i))$. Then,
		\begin{align}\label{contraction}
		\frac{1}{m} \underset{\bep\in \{\pm 1\}^m}{\mathbb{E}}& \Bigg[\sup_{f \in \mathcal{F}}\sum_{i=1}^{m}\epsilon_i (\varPhi\boldsymbol{\circ} f)(\bx_i)\Bigg]\nonumber\\
		&\leq
		\frac{\rho}{m} \underset{\bep\in \{\pm 1\}^m}{\mathbb{E}} \Bigg[\sup_{f \in \mathcal{F}}\sum_{i=1}^{m}\epsilon_i f(\bx_i)\Bigg],
		\end{align}
		alternatively, $\mathfrak{R}_S(\boldsymbol{\varPhi}\boldsymbol{\circ} \mathcal{F}) \leq \rho\mathfrak{R}_S(\mathcal{F})$, where $\boldsymbol{\varPhi}$ denotes $\varPhi_1(\bx_1),...,\varPhi_m(\bx_m)$ for $S$.
	\end{lemma}

	\begin{proposition}\cite[Proposition A.1---H\"{o}lder's inequality]{MohriBook18}\label{HolderIneq}\\
		Let $p,q \geq 1$ be conjugate: $\frac{1}{p}+\frac{1}{q}=1$. Then, for all $\bx,\by \in \mathbb{R}^n$,
		\begin{equation}\label{HolderIneqEQ}
		\|\bx\cdot\by\|_1\leq \lVert\bx\rVert_p\lVert\by\rVert_q,
		\end{equation}
		with the equality when $|\mathrm{y}_i|=|\mathrm{x}_i|^{p-1}$ for all $i\in[1,n]$. 	 
		
	\end{proposition}
	
	\textbf{Supported inequalities}:
	\begin{itemize}
		\item [(i)]	If A, B are sets of positive real numbers, then:
		\begin{equation}\label{supIneq}
		\sup(AB) = \sup(A)\cdot\sup(B).
		\end{equation}
		\item[(ii)] 		Given $x \in \mathbb{R}$, we have:
		\begin{equation}\label{expIneq}
		\frac{\exp(x)+\exp(-x)}{2}\leq \exp(x^2/2).
		\end{equation}		
		\item [(iii)]		Let $X$ and $Y$ be random variables, the Cauchy–Bunyakovsky–Schwarz inequality gives:
		\begin{equation}\label{expectIneq}
		(\mathbb{E}[XY])^2 \leq \mathbb{E}[X^2]\cdot\mathbb{E}[Y^2].
		\end{equation}	
		\item 	[(iv)]	If $\psi$ is a convex function, the Jensen's inequality gives:
		\begin{equation}\label{expectConvexIneq}
		\psi(\mathbb{E}[X]) \leq \mathbb{E}[\psi(X)]. 
		\end{equation}			
	\end{itemize}

\onecolumn

\section{Proof of Theorem \ref{GE-reweighted-RNN-Theorem}}\label{deepreweighted-RNNTheoremProofTime}
\begin{proof}
	We consider the real-valued family of functions $\mathcal{F}_{d,T}:\mathbb{R}^h\times \mathbb{R}^n\mapsto \mathbb{R}$ for the functions $f_{\bcW,\bU}^{(d)}$ to update $\bh_T^{(d)}$  in layer $d$, time step $T$, defined as
	\begin{equation}\label{deepreweighted-RNNFamilyTime}
	\mathcal{F}_{d,T}=\Big\{ (\bh_{T-1}^{(d)},\bx_T)\mapsto \varPhi(\bw_{d}^{\mathrm{T}} f_{\bcW,\bcU}^{(d-1)}(\bh_{T-1}^{(d)},\bx_T)+\bu_d^{\mathrm{T}}\bx_T): \|\bW_d\|_{1,\infty}\leq \alpha_d ,~\|\bU_d\|_{1,\infty}\leq \beta_d \Big\},
	\end{equation}
	where $\bw_{d},\bu_d$ are the corresponding rows from $\bW_d,\bU_d $, respectively, and $\alpha_{l},\beta_l$, with $1<l\leq d$, are nonnegative hyper-parameters. For the first layer and the first time step, i.e., $l=1$,	$t=1$, the real-valued family of functions, $\mathcal{F}_{1,1}:\mathbb{R}^h\times \mathbb{R}^n\mapsto\mathbb{R}$, for the functions $f_{\bcW,\bcU}^{(1)}$ is defined by: 
	\begin{equation}\label{deepreweighted-RNNFamilyLayer1Time}
	\mathcal{F}_{1,1}=\Big\{ (\bh_{0},\bx_1) \mapsto \varPhi(\bw_{1}^{\mathrm{T}}\bh_{0}+\bu_1^{\mathrm{T}}\bx_1): \|\bW_1\|_{1,\infty}\leq \alpha_1 ,~\|\bU\|_{1,\infty}\leq \beta_1 \Big\},
	\end{equation}	
	where $\alpha_{1},\beta_1$ are nonnegative hyper-parameters. We denote the input layer as $f_{\bcW,\bcU}^{(0)}=\bh_{0}$ at the first time step. From the definition of Rademacher complexity in \eqref{empiricalRademacher} and the family of functions in \eqref{deepreweighted-RNNFamilyTime} and \eqref{deepreweighted-RNNFamilyLayer1Time}, we obtain:
	\begin{subequations}
		\begin{align}
		m\mathfrak{R}_S(\mathcal{F}_{d,T})&\leq
		\underset{\bep\in \{\pm 1\}^m}{\mathbb{E}} \Bigg[\underset{\lVert \bu_{d} \rVert_1\leq \beta_{d} }{\underset{{\lVert \bw_{d} \rVert_1\leq \alpha_{d}}}{\underset{\bcW,\bcU}{\sup}}}\sum_{i=1}^{m}\epsilon_i \varPhi\Big(\bw_{d}^{\mathrm{T}} f_{\bcW,\bcU}^{(d-1)}(\bh_{T-1,i},\bx_{T,i})+\bu_d^{\mathrm{T}}\bx_{T,i}\Big)\Bigg]
		\nonumber\\
		&\leq	\frac{1}{\lambda}\log\exp\Bigg(\underset{\bep\in \{\pm 1\}^m}{\mathbb{E}} \Bigg[\underset{\lVert \bu_{d} \rVert_1\leq \beta_{d} }{\underset{{\lVert \bw_{d} \rVert_1\leq \alpha_{d}}}{\underset{\bcW,\bcU}{\sup}}}\lambda\sum_{i=1}^{m}\epsilon_i \Big(\bw_{d}^{\mathrm{T}} f_{\bcW,\bcU}^{(d-1)}(\bh_{T-1,i},\bx_{T,i})+\bu_d^{\mathrm{T}}\bx_{T,i}\Big)\Bigg]\Bigg) \nonumber
		\\ 
		&\leq
		\frac{1}{\lambda}\log \underset{\bep\in \{\pm 1\}^m}{\mathbb{E}}\Bigg[\underset{\lVert \bu_{d} \rVert_1\leq \beta_{d} }{\underset{{\lVert \bw_{d} \rVert_1\leq \alpha_{d}}}{\underset{\bcW,\bcU}{\sup}}}\exp\Bigg(\lambda \sum\limits_{i=1}^{m}\epsilon_i \Big(\bw_{d}^{\mathrm{T}} f_{\bcW,\bcU}^{(d-1)}(\bh_{T-1,i},\bx_{T,i})\Big)+\lambda \sum\limits_{i=1}^{m}\epsilon_i  \bu_d^{\mathrm{T}}\bx_{T,i}\Bigg)\Bigg]\label{RCRNNProofEQsubLayerd-1Time1}
		\\ 
		&\leq
		\frac{1}{\lambda}\log \underset{\bep\in \{\pm 1\}^m}{\mathbb{E}}\Bigg[\underset{{\lVert \bw_{d} \rVert_1\leq \alpha_{d}}}{\underset{\bcW,\bcU}{\sup}}\exp\Bigg(\lambda \sum\limits_{i=1}^{m}\epsilon_i \Big(\bw_{d}^{\mathrm{T}} f_{\bcW,\bcU}^{(d-1)}(\bh_{T-1,i},\bx_{T,i})\Big)\Bigg) \underset{\|\bu_d\|_1\leq\beta_d}{\sup}\exp\Bigg(\lambda \sum\limits_{i=1}^{m}\epsilon_i \bu_d^{\mathrm{T}}\bx_{T,i}\Bigg)\Bigg], \label{RCRNNProofEQsubLayerd-1Time2}
		\end{align}
	\end{subequations}	
	where $\lambda>0$ is an arbitrary parameter, Eq. \eqref{RCRNNProofEQsubLayerd-1Time1} follows Lemma \ref{ContractionLemma} for 1-Lipschitz $\varPhi$ a long with Inequality \eqref{expectConvexIneq}, and \eqref{RCRNNProofEQsubLayerd-1Time2} holds by Inequality \eqref{supIneq}.
	
	For layer $1\leq l\leq d$ and time step $t$, let us denote: 
	\begin{align}
	\varDelta_{\bh_{t-1},\bx_t}^{(l)}&=\underset{{\lVert \bw_{l} \rVert_1\leq \alpha_{l}}}{\underset{\bcW,\bcU}{\sup}}\exp\Bigg(\lambda\varLambda_l \sum\limits_{i=1}^{m}\epsilon_i \Big(\bw_{l}^{\mathrm{T}} f_{\bcW,\bcU}^{(l-1)}(\bh_{t-1,i},\bx_{t,i})\Big)\Bigg),\label{DeltaHTime}\\
	\varDelta_{\bx_t}^{(l)}&=\underset{\|\bu_l\|_1\leq\beta_l}{\sup}\exp\Bigg(\lambda \varLambda_l\sum\limits_{i=1}^{m}\epsilon_i \Big(\bu_l^{\mathrm{T}}\bx_{t,i}\Big)\Bigg),\label{DeltaXTime}
	\end{align}
	where $\varLambda_l$ is defined as follows: $\varLambda_d=1$, $\varLambda_l=\prod\limits_{k=l+1}^{d}\alpha_k$ with $~1\leq l\leq d-1$, and $\varLambda_0=\prod\limits_{k=1}^{d}\alpha_k$.
	
	Following the H\"{o}lder's inequality in \eqref{HolderIneqEQ} in case of $p=1$ and $q=\infty$ applied to $\bw_{l}^{\mathrm{T}}$ and $f_{\bcW,\bcU}^{(l-1)}(\bh_{t-1,i},\bx_{t,i})$ in \eqref{DeltaHTime}, respectively, we get:
	\begin{align}\label{RCRNNProofEQsubLayerd-HTime}
	\varDelta_{\bh_{t-1},\bx_t}^{(d)}
	&\leq 
	\underset{\lVert \bU_{d-1} \rVert_{1,\infty}\leq \beta_{d-1} }{\underset{{\lVert \bW_{d-1} \rVert_{1,\infty}\leq \alpha_{d-1}}}{\underset{\bcW,\bcU}{\sup}}}\exp\Bigg(\lambda \alpha_d\Bigg\|\sum_{i=1}^{m}\epsilon_i \varPhi\Big(\bW_{d-1} f_{\bcW,\bcU}^{(d-2)}(\bh_{t-1,i},\bx_{t,i})+\mathbf{U}_{d-1}\bx_{t,i}\Big)\Bigg\|_\infty\Bigg)\nonumber
	\\
	&\leq 
	\underset{\lVert \bu_{d-1,k} \rVert_1\leq \beta_{d-1} }{\underset{{\lVert \bw_{d-1,k} \rVert_1\leq \alpha_{d-1}}}{\underset{\bcW,\bcU}{\sup}}}\exp\Bigg(\lambda \alpha_d\underset{~k\in \{1,\cdots,h\}}{\max}\Bigg|\sum_{i=1}^{m}\epsilon_i \varPhi\Big(\bw_{d-1,k}^{\mathrm{T}} f_{\bcW,\bcU}^{(d-2)}(\bh_{t-1,i},\bx_{t,i})+\bu_{d-1,k}^{\mathrm{T}}\bx_{t,i}\Big)\Bigg|\Bigg)\nonumber
	\\
	&\leq 
	\underset{\lVert \bu_{d-1,k} \rVert_1\leq \beta_{d-1} }{\underset{{\lVert \bw_{d-1,k} \rVert_1\leq \alpha_{d-1}}}{\underset{\bcW,\bcU}{\sup}}}\exp\Bigg(\lambda\alpha_d\Bigg|\sum_{i=1}^{m}\epsilon_i \varPhi\Big(\bw_{d-1,k}^{\mathrm{T}} f_{\bcW,\bcU}^{(d-2)}(\bh_{t-1,i},\bx_{t,i})+\bu_{d-1,k}^{\mathrm{T}}\bx_{t,i}\Big)\Bigg|\Bigg).
	\end{align}

	Similarly, from \eqref{DeltaXTime}, we obtain:
	\begin{align}\label{RCRNNProofEQsubLayerd-XTime}
	\varDelta_{\bx_t}^{(d)}
	&\leq
	\underset{\|\bu_d\|_1\leq\beta_{d}}{\sup}\exp\Bigg(\lambda\sum_{i=1}^{m}\epsilon_i \bu_d^{\mathrm{T}}\bx_{t,i}\Bigg)
	\leq
	\exp\Bigg(\lambda\beta_d\Big\|\sum_{i=1}^{m}\epsilon_i \bx_{t,i}\Big\|_\infty\Bigg)\leq
	\exp\Bigg(\lambda\beta_d\Big|\sum_{i=1}^{m}\epsilon_i \mathrm{x}_{\tau,i,\kappa}\Big|\Bigg),
	\end{align}
	where $\{\tau,\kappa\}=\mathop{\mathrm{argmax}}\limits_{t\in\{1,\dots,T\},j\in\{1,\dots,n\}}\Big|\sum\limits_{i=1}^{m}\epsilon_i \mathrm{x}_{t,i,j}\Big|$.
	
	From \eqref{RCRNNProofEQsubLayerd-1Time2}, \eqref{RCRNNProofEQsubLayerd-HTime}, and \eqref{RCRNNProofEQsubLayerd-XTime}, we get:
	\begin{subequations}
		\begin{align}
	&	m\mathfrak{R}_S(\mathcal{F}_{d,T})
		\nonumber\\
		&\leq
		\frac{1}{\lambda}\log \Bigg(\underset{\bep\in \{\pm 1\}^m}{\mathbb{E}}\Bigg[\underset{\lVert \bu_{d-1,k} \rVert_1\leq \beta_{d-1}}{\underset{{\lVert \bw_{d-1,k} \rVert_1\leq \alpha_{d-1}}}{\underset{\bcW,\bcU}{\sup}}}\exp\Bigg(\lambda\alpha_d\Bigg|\sum_{i=1}^{m}\epsilon_i \varPhi\Big(\bw_{d-1,k}^{\mathrm{T}} f_{\bcW,\bcU}^{(d-2)}(\bh_{T-1,i},\bx_{T,i})+\bu_{d-1,k}^{\mathrm{T}}\bx_{T,i}\Big)\Bigg|
		+\lambda\beta_{d}\Big|\sum_{i=1}^{m}\epsilon_i \mathrm{x}_{\tau,i,\kappa}\Big|\Bigg)\Bigg]\Bigg)\nonumber
		\\
		&\leq
		\frac{1}{\lambda}\log \Bigg(\underset{\bep\in \{\pm 1\}^m}{\mathbb{E}}\Bigg[\underset{\lVert \bu_{d-1,k} \rVert_1\leq \beta_{d-1}}{\underset{{\lVert \bw_{d-1,k} \rVert_1\leq \alpha_{d-1}}}{\underset{\bcW,\bcU}{\sup}}}\Bigg(\exp\Bigg(\lambda\alpha_d\sum_{i=1}^{m}\epsilon_i \varPhi\Big(\bw_{d-1,k}^{\mathrm{T}} f_{\bcW,\bcU}^{(d-2)}(\bh_{T-1,i},\bx_{T,i})+\bu_{d-1,k}^{\mathrm{T}}\bx_{T,i}\Big)
		+\lambda\beta_{d}\sum_{i=1}^{m}\epsilon_i \mathrm{x}_{\tau,i,\kappa}\Bigg)\nonumber
		\\
		&~~~~~+\exp\Bigg(\lambda\alpha_d\sum_{i=1}^{m}\epsilon_i \varPhi\Big(\bw_{d-1,k}^{\mathrm{T}} f_{\bcW,\bcU}^{(d-2)}(\bh_{T-1,i},\bx_{T,i})+\bu_{d-1,k}^{\mathrm{T}}\bx_{T,i}\Big)
		-\lambda\beta_{d}\sum_{i=1}^{m}\epsilon_i \mathrm{x}_{\tau,i,\kappa}\Bigg)\nonumber
		\\
		&~~~~~+
		\exp\Bigg(-\lambda\alpha_d\sum_{i=1}^{m}\epsilon_i \varPhi\Big(\bw_{d-1,k}^{\mathrm{T}} f_{\bcW,\bcU}^{(d-2)}(\bh_{T-1,i},\bx_{T,i})+\bu_{d-1,k}^{\mathrm{T}}\bx_{T,i}\Big)
		+\lambda\beta_{d}\sum_{i=1}^{m}\epsilon_i \mathrm{x}_{\tau,i,\kappa}\Bigg)\nonumber
		\\
		&~~~~~+
		\exp\Bigg(-\lambda\alpha_d\sum_{i=1}^{m}\epsilon_i \varPhi\Big(\bw_{d-1,k}^{\mathrm{T}} f_{\bcW,\bcU}^{(d-2)}(\bh_{T-1,i},\bx_{T,i})+\bu_{d-1,k}^{\mathrm{T}}\bx_{T,i}\Big)
		-\lambda\beta_{d}\sum_{i=1}^{m}\epsilon_i \mathrm{x}_{\tau,i,\kappa}\Bigg)\Bigg)\Bigg]\Bigg)\nonumber
		\\
		&\leq
		\frac{1}{\lambda}\log \Bigg(4\underset{\bep\in \{\pm 1\}^m}{\mathbb{E}}\Bigg[\varDelta_{\bh_{T-1},\bx_T}^{(d-1)}\varDelta_{\bx_T}^{(d-1)}\exp\Big(\beta_d\lambda\sum_{i=1}^{m}\epsilon_i \mathrm{x}_{\tau,i,\kappa}\Big)\Bigg]\Bigg)
		\label{RCRNNProofEQsubLayerd-3Time1}
		\\
		&\leq
		\frac{1}{\lambda}\log \Bigg(4^{d-1}\underset{\bep\in \{\pm 1\}^m}{\mathbb{E}}\Bigg[\varDelta_{\bh_{T-1},\bx_T}^{(1)}\varDelta_{\bx_T}^{(1)}\exp\Bigg(\lambda\Big(\sum\limits_{l=2}^{d}\beta_l\varLambda_l\Big)\sum_{i=1}^{m}\epsilon_i \mathrm{x}_{\tau,i,\kappa}\Bigg)\Bigg]\Bigg)
		\label{RCRNNProofEQsubLayerd-3Time2}
		\\
		&\leq
		\frac{1}{\lambda}\log \Bigg(4^{d-1}\underset{\bep\in \{\pm 1\}^m}{\mathbb{E}}\Bigg[\exp\Bigg(\lambda\Big(\sum\limits_{l=2}^{d}\beta_l\varLambda_l\Big)\sum_{i=1}^{m}\epsilon_i \mathrm{x}_{\tau,i,\kappa}\Bigg)\underset{{\lVert \bw_{1} \rVert_1\leq \alpha_{1}}}{\sup}\exp\Bigg(\lambda\varLambda_{1} \sum\limits_{i=1}^{m}\epsilon_i \Big(\bw_{1}^{\mathrm{T}} \bh_{T-1,i}\Big)\Bigg)
		\nonumber
		\\
		&~~~~~\cdot\underset{\|\bu_1\|_1\leq\beta_1}{\sup}\exp\Bigg(\lambda \varLambda_{1}\sum\limits_{i=1}^{m}\epsilon_i \Big(\bu_1^{\mathrm{T}}\bx_{T,i}\Big)\Bigg)\Bigg]\Bigg)
		\label{RCRNNProofEQsubLayerd-3Time3}
		\\
		&\leq
		\frac{1}{\lambda}\log \Bigg(4^{d-1}\underset{\bep\in \{\pm 1\}^m}{\mathbb{E}}\Bigg[\exp\Bigg(\lambda\Big(\sum\limits_{l=2}^{d}\beta_l\varLambda_l\Big)\sum_{i=1}^{m}\epsilon_i \mathrm{x}_{\tau,i,\kappa}\Bigg)\underset{\lVert \bu_{d} \rVert_1\leq \beta_{d} }{\underset{{\lVert \bw_{d} \rVert_1\leq \alpha_{d}}}{\underset{\bcW,\bcU}{\sup}}}\exp\Bigg(\lambda\varLambda_{0} \Big\|\sum\limits_{i=1}^{m}\epsilon_i \bh_{T-1,i}\Big\|_\infty\Bigg)\nonumber\\
		&~~~~~\cdot\exp\Bigg(\lambda\beta_1 \varLambda_{1}\Big\|\sum\limits_{i=1}^{m}\epsilon_i \bx_{T,i}\Big\|_\infty\Big)\Bigg)\Bigg]\Bigg)
		\label{RCRNNProofEQsubLayerd-3Time4}
		\\
		&\leq
		\frac{1}{\lambda}\log \Bigg(4^{d}\underset{\bep\in \{\pm 1\}^m}{\mathbb{E}}\Bigg[\exp\Bigg(\lambda\Big(\sum\limits_{l=1}^{d}\beta_l\varLambda_l\Big)\sum_{i=1}^{m}\epsilon_i \mathrm{x}_{\tau,i,\kappa}\Bigg)
		\nonumber\\
		&~~~~~\cdot\underset{\lVert \bu_{d} \rVert_1\leq \beta_{d} }{\underset{{\lVert \bw_{d} \rVert_1\leq \alpha_{d}}}{\underset{\bcW,\bcU}{\sup}}} \exp\Bigg(\lambda\varLambda_{0}\sum_{i=1}^{m}\epsilon_i \varPhi\Big(\bw_{d}^{\mathrm{T}} f_{\bcW,\bcU}^{(d-1)}(\bh_{T-2,i},\bx_{T-1,i})+\bu_d^{\mathrm{T}}\bx_{T-1,i}\Big)\Bigg) \Bigg]\Bigg),	\label{RCRNNProofEQsubLayerd-3Time5}
		\end{align}
	\end{subequations}
	where \eqref{RCRNNProofEQsubLayerd-3Time1} holds by Inequality \eqref{supIneq} and \eqref{RCRNNProofEQsubLayerd-3Time2} follows by repeating the process from layer $d-1$ to layer 1 for time step $T$. Furthermore, \eqref{RCRNNProofEQsubLayerd-3Time3} is returned as the beginning of the process for time step $T-1$ and \eqref{RCRNNProofEQsubLayerd-3Time4} follows Inequality \eqref{HolderIneqEQ}. 
	
	Proceeding by repeating the above procedure in \eqref{RCRNNProofEQsubLayerd-3Time5} from time step $T-1$ to time step $1$, we get:	
	\begin{align}\label{RCRNNProofEQsubLayerd-4-0Time}
	m\mathfrak{R}_S(\mathcal{F}_{d,T})&\leq
	\frac{1}{\lambda}\log \Bigg(4^{dT}\underset{\bep\in \{\pm 1\}^m}{\mathbb{E}}\Bigg[\exp\Bigg(\lambda\Big(\sum\limits_{l=1}^{d}\beta_l\varLambda_l\Big)\Big(\frac{\Lambda_0^T-1}{\Lambda_0-1}\Big)\sum_{i=1}^{m}\epsilon_i \mathrm{x}_{\tau,i,\kappa}\Bigg)
	\exp\Bigg(\lambda\varLambda_{0}^T \Big\|\sum\limits_{i=1}^{m}\epsilon_i \bh_{0}\Big\|_\infty\Bigg]\Bigg).
	\end{align}

	Let us denote $\mu=\mathop{\mathrm{argmax}}\limits_{j\in\{1,\dots,h\}}\Big|\sum\limits_{i=1}^{m}\epsilon_i \mathrm{h}_{0,j}\Big|$, from \eqref{RCRNNProofEQsubLayerd-4-0Time}, we have:
	\begin{subequations}
		\begin{align}
	&	m\mathfrak{R}_S(\mathcal{F}_{d,T})
		\leq
		\frac{1}{\lambda}\log \Bigg(4^{dT}\underset{\bep\in \{\pm 1\}^m}{\mathbb{E}}\Bigg[\exp\Bigg(\lambda\Big(\sum\limits_{l=1}^{d}\beta_l\varLambda_l\Big)\Big(\frac{\Lambda_0^T-1}{\Lambda_0-1}\Big)\sum_{i=1}^{m}\epsilon_i \mathrm{x}_{\tau,i,\kappa}\Bigg)
		\exp\Bigg(\lambda\varLambda_{0}^T\sum\limits_{i=1}^{m}\epsilon_i \mathrm{h}_{0,\mu}\Bigg) \Bigg]\Bigg)
		\nonumber\\		
		&\leq
		\frac{2dT\log 2}{\lambda}+\frac{1}{2\lambda}\log \Bigg(\underset{\bep\in \{\pm 1\}^m}{\mathbb{E}}\Bigg[\exp\Bigg(\lambda\Big(\sum\limits_{l=1}^{d}\beta_l\varLambda_l\Big)\Big(\frac{\Lambda_0^T-1}{\Lambda_0-1}\Big)\sum_{i=1}^{m}\epsilon_i \mathrm{x}_{\tau,i,\kappa}\Bigg)
		\exp\Bigg(\lambda\varLambda_{0}^T\sum\limits_{i=1}^{m}\epsilon_i \mathrm{h}_{0,\mu}\Bigg) \Bigg]\Bigg)^2
		\nonumber
		\\
		&\leq
		\frac{2dT\log 2}{\lambda}+\frac{1}{2\lambda}\log \underset{\bep\in \{\pm 1\}^m}{\mathbb{E}}\Bigg[\exp\Bigg(2\lambda\Big(\sum\limits_{l=1}^{d}\beta_l\varLambda_l\Big)\Big(\frac{\Lambda_0^T-1}{\Lambda_0-1}\Big)\sum_{i=1}^{m}\epsilon_i \mathrm{x}_{\tau,i,\kappa}\Bigg)\Bigg]
		+\frac{1}{2\lambda}\log \underset{\bep\in \{\pm 1\}^m}{\mathbb{E}}\Bigg[\exp\Bigg(2\lambda\varLambda_{0}^T \sum\limits_{i=1}^{m}\epsilon_i \mathrm{h}_{0,\mu}\Bigg)\Bigg]
		\label{RCRNNProofEQsubLayerd-4Time1}
		\\
		&\leq
		\frac{2dT\log 2}{\lambda}+\frac{1}{2\lambda}\log \sum\limits_{j=1}^{n}\underset{\bep\in \{\pm 1\}^m}{\mathbb{E}}\Bigg[\exp\Bigg(2\lambda\Big(\sum\limits_{l=1}^{d}\beta_l\varLambda_l\Big)\Big(\frac{\Lambda_0^T-1}{\Lambda_0-1}\Big)\sum_{i=1}^{m}\epsilon_i \mathrm{x}_{\tau,i,j}\Bigg)\Bigg]
		\nonumber
		\\
		&~~~~~
		+\frac{1}{2\lambda}\log \sum\limits_{j=1}^{h}\underset{\bep\in \{\pm 1\}^m}{\mathbb{E}}\Bigg[\exp\Bigg(2\lambda\varLambda_{0}^T \sum\limits_{i=1}^{m}\epsilon_i \mathrm{h}_{0,j}\Bigg)\Bigg]
		\label{RCRNNProofEQsubLayerd-4Time2}
		\\
		&\leq
		\frac{2dT\log 2}{\lambda}+\frac{1}{2\lambda}\log \sum\limits_{j=1}^{n}\prod\limits_{i=1}^{m}\underset{\bep\in \{\pm 1\}^m}{\mathbb{E}}\Bigg[\exp\Bigg(2\lambda\Big(\sum\limits_{l=1}^{d}\beta_l\varLambda_l\Big)\Big(\frac{\Lambda_0^T-1}{\Lambda_0-1}\Big)\epsilon_i \mathrm{x}_{\tau,i,j}\Bigg)\Bigg]
		\nonumber
		\\
		&~~~~~
		+\frac{1}{2\lambda}\log \sum\limits_{j=1}^{h}\prod\limits_{i=1}^{m}\underset{\bep\in \{\pm 1\}^m}{\mathbb{E}}\Bigg[\exp\Bigg(2\lambda\varLambda_{0}^T\epsilon_i \mathrm{h}_{0,j}\Bigg)\Bigg]
		\nonumber
		\\
		&\leq
		\frac{2dT\log 2}{\lambda}+\frac{1}{2\lambda}\log \sum\limits_{j=1}^{n}\prod\limits_{i=1}^{m}\Bigg[\frac{1}{2}\exp\Bigg(2\lambda\Big(\sum\limits_{l=1}^{d}\beta_l\varLambda_l\Big) \Big(\frac{\Lambda_0^T-1}{\Lambda_0-1}\Big)\mathrm{x}_{\tau,i,j}\Bigg)
		+\frac{1}{2}\exp\Bigg(-2\lambda\Big(\sum\limits_{l=1}^{d}\beta_l\varLambda_l\Big) \Big(\frac{\Lambda_0^T-1}{\Lambda_0-1}\Big)\mathrm{x}_{\tau,i,j}\Bigg)\Bigg]
		\nonumber
		\\
		&~~~~~
		+\frac{1}{2\lambda}\log \sum\limits_{j=1}^{h}\prod\limits_{i=1}^{m}\Bigg[\frac{1}{2}\exp\Bigg(2\lambda\varLambda_{0}^T \mathrm{h}_{0,j}\Bigg)+\frac{1}{2}\exp\Bigg(-2\lambda\varLambda_{0}^T \mathrm{h}_{0,j}\Bigg)\Bigg]
		\nonumber
		\\
		&\leq
		\frac{2dT\log 2}{\lambda}+\frac{1}{2\lambda}\log \sum\limits_{j=1}^{n}\Bigg[\exp\Bigg(2\lambda^2\Big(\sum\limits_{l=1}^{d}\beta_l\varLambda_l\Big)^2 \Big(\frac{\Lambda_0^T-1}{\Lambda_0-1}\Big)^2\sum\limits_{i=1}^{m} x^2_{\tau,i,j}\Bigg)\Bigg]
		+\frac{1}{2\lambda}\log \sum\limits_{j=1}^{h}\Bigg[\exp\Bigg(2\lambda^2\varLambda_{0}^{2T} \sum\limits_{i=1}^{m}h^2_{0,j}\Bigg)\Bigg]
		\label{RCRNNProofEQsubLayerd-4Time3}
		\\
		&\leq
		\frac{2dT\log 2}{\lambda}+\frac{\log n}{2\lambda} + \lambda\Big(\sum\limits_{l=1}^{d}\beta_l\varLambda_l\Big)^2 \Big(\frac{\Lambda_0^T-1}{\Lambda_0-1}\Big)^2 mB_{\bx}^2
		+\frac{\log h}{2\lambda}+\lambda\varLambda_{0}^{2T} m\|\bh_0\|_{\infty}^2
		\nonumber
		\\
		&\leq
		\frac{2dT\log 2+\log \sqrt{n} + \log \sqrt{h}}{\lambda}+ \lambda\Bigg(\Big(\sum\limits_{l=1}^{d}\beta_l\varLambda_l\Big)^2 \Big(\frac{\Lambda_0^T-1}{\Lambda_0-1}\Big)^2 mB_{\bx}^2
		+\varLambda_{0}^{2T} m\|\bh_0\|_{\infty}^2\Bigg),\label{RCRNNProofEQsubLayerd-4Time4}
		\end{align}
	\end{subequations}	
	where \eqref{RCRNNProofEQsubLayerd-4Time1} follows Inequality \eqref{expectIneq}, \eqref{RCRNNProofEQsubLayerd-4Time2} holds by replacing with $\sum_{j=1}^{n}$ and $\sum_{j=1}^{h}$, respectively. In addition, \eqref{RCRNNProofEQsubLayerd-4Time3} follows \eqref{expIneq} and \eqref{RCRNNProofEQsubLayerd-4Time4} is received by the following definition: At time step $t$, we define $\bX_t\in \mathbb{R}^{n\times m}$, a matrix composed of $m$ columns from the $m$ input vectors $\{\bx_{t,i}\}_{i=1}^m$; we also define $\|\bX_t\|_{2,\infty}=\sqrt{\max\limits_{k\in \{1,\dots,n\}}\sum_{i=1}^m\mathrm{x}_{t,i,k}^2}\leq \sqrt{m}B_{\bx}$, representing the maximum of the $\ell_2$-norms of the rows of matrix $\bX_t$, and $\|\bh_0\|_{\infty}=\max\limits_{j}|\mathrm{h}_{0,j}|$.
	
	Choosing $\lambda=\sqrt{\frac{2dT\log 2+\log \sqrt{n} + \log \sqrt{h}}{\Big(\sum\limits_{l=1}^{d}\beta_l\varLambda_l\Big)^2 \Big(\frac{\Lambda_0^T-1}{\Lambda_0-1}\Big)^2 mB_{\bx}^2
			+\varLambda_{0}^{2T} m\|\bh_0\|_{\infty}^2}}$, we achieve the upper bound:
	\begin{align}\label{RCRNNProofEQsubLayerd-FinalTime}
	\mathfrak{R}_S(\mathcal{F}_{d,T})&\leq
	\sqrt{\frac{2(4dT\log 2+\log n + \log h)}{m}\Bigg(\Big(\sum\limits_{l=1}^{d}\beta_l\varLambda_l\Big)^2 \Big(\frac{\Lambda_0^T-1}{\Lambda_0-1}\Big)^2 B_{\bx}^2+\varLambda_{0}^{2T} \|\bh_0\|_{\infty}^2}\Bigg).
	\end{align}
	
	It can be noted that $\mathfrak{R}_S(\mathcal{F}_{d,T})$ in \eqref{RCRNNProofEQsubLayerd-FinalTime} is derived for the real-valued functions $\mathcal{F}_{d,T}$. For the vector-valued functions $\mathcal{F}_{d,T}:\mathbb{R}^h\times \mathbb{R}^n\mapsto \mathbb{R}^h$ [in Theorem~\ref{GE-reweighted-RNN-Theorem}], we apply the contraction lemma [Lemma~\ref{ContractionLemma}] to a Lipschitz loss to obtain the complexity of such vector-valued functions by means of the complexity of the real-valued functions. Specifically, in Theorem~\ref{GE-reweighted-RNN-Theorem}, under the assumption of the 1-Lipschitz loss function and from Theorem~\ref{GETheorem}, Lemma~\ref{ContractionLemma}, we complete the proof.
	
\end{proof}

\end{document}